\documentclass{article}

\usepackage[final]{neurips_2025}

\usepackage[utf8]{inputenc} 
\usepackage[T1]{fontenc}    
\usepackage{hyperref}       
\usepackage{url}            
\usepackage{booktabs}       
\usepackage{amsfonts}       
\usepackage{nicefrac}       
\usepackage{microtype}      
\usepackage{xcolor}         

\usepackage{color}
\usepackage{amsmath}
\usepackage{amssymb}
\usepackage{amsthm}
\usepackage{bm}
\usepackage{graphicx}
\usepackage{wrapfig}
\usepackage{subfigure}
\usepackage{bbm}

\usepackage{arydshln}

\usepackage{multirow}
\usepackage{booktabs}
\usepackage{colortbl}

\newtheorem{theorem}{Theorem}

\newtheorem{assumption}{Assumption}

\usepackage{algorithm}
\usepackage{algorithmic}

\title{Multi-Task Vehicle Routing Solver via Mixture of Specialized Experts under State-Decomposable MDP}

\author{
Yuxin Pan$^{1}$\thanks{This work is done when Yuxin Pan works as an intern in Tencent AI Lab.} \quad Zhiguang Cao$^{2}$ \quad Chengyang Gu$^{3}$ \quad Liu Liu$^{4}$\thanks{Corresponding Authors.} \\
\textbf{Peilin Zhao}$^{4,6}$ \quad \textbf{Yize Chen}$^{5}$$^{\dag}$ \quad \textbf{Fangzhen Lin}$^{1}$$^{\dag}$ \\
$^{1}$The Hong Kong University of Science and Technology \\
$^{2}$Singapore Management University \\
$^{3}$The Hong Kong University of Science and Technology (Guangzhou) \\
$^{4}$Tencent AI Lab \\
$^{5}$University of Alberta \\
$^{6}$Shanghai Jiao Tong University \\
\texttt{yuxin.pan@connect.ust.hk} \\
}

\begin{document}

\maketitle

\begin{abstract}
    Existing neural methods for multi-task vehicle routing problems (VRPs) typically learn unified solvers to handle multiple constraints simultaneously. However, they often underutilize the compositional structure of VRP variants, each derivable from a common set of basis VRP variants. This critical oversight causes unified solvers to miss out the potential benefits of basis solvers, each specialized for a basis VRP variant. To overcome this limitation, we propose a framework that enables unified solvers to perceive the shared-component nature across VRP variants by proactively reusing basis solvers, while mitigating the exponential growth of trained neural solvers.
    Specifically, we introduce a State-Decomposable MDP (SDMDP) that reformulates VRPs by expressing the state space as the Cartesian product of basis state spaces associated with basis VRP variants. More crucially, this formulation inherently yields the optimal basis policy for each basis VRP variant. Furthermore, a Latent Space-based SDMDP extension is developed by incorporating both the optimal basis policies and a learnable mixture function to enable the policy reuse in the latent space. Under mild assumptions, this extension provably recovers the optimal unified policy of SDMDP through the mixture function that computes the state embedding as a mapping from the basis state embeddings generated by optimal basis policies. For practical implementation, we introduce the Mixture-of-Specialized-Experts Solver (MoSES), which realizes basis policies through specialized Low-Rank Adaptation (LoRA) experts, and implements the mixture function via an adaptive gating mechanism. Extensive experiments conducted across VRP variants showcase the superiority of MoSES over prior methods. The source code is available at \href{https://github.com/panyxy/moses_vrp}{https://github.com/panyxy/moses\_vrp}.
\end{abstract}

\section{Introduction} 

Vehicle routing problems (VRPs) are a canonical class of combinatorial optimization problems (COPs) with broad applications spanning transportation~\cite{garaix2010vehicle}, logistics~\cite{cattaruzza2017vehicle}, and manufacturing~\cite{zhang2023review}. While exact methods are computationally prohibitive~\cite{laporte1983branch}, and heuristic methods rely on substantial domain-specific knowledge~\cite{helsgaun2017extension,vidal2012hybrid}, recent studies craft learning-based neural solvers for empirically sound performance with lower makespan and marginal domain expertise~\cite{vinyals2015pointer,nazari2018reinforcement,kool2018attention,kwon2020pomo}. Prevailing approaches primarily target one specific VRP, with tremendous efforts dedicated to out-of-distribution (OOD) generalization~\cite{bi2022learning,zhou2023towards,hou2023generalize,drakulic2024bq,luo2024neural,ye2024glop,zheng2024udc,pan2025hlgp4cvrp,luo2025boosting}. However, practical applications involve multiple VRP variants with diverse node attributes and solution constraints, making the development of specialized neural solver costly due to the need for retraining from scratch.

Recent advances in multi-task neural solvers are conscious of inherent partial similarities present among these variants,
enabling efficient knowledge transfer via shared embeddings. These methods commonly resort to specialized adaptations of a pretrained backbone model~\cite{lin2024cross,correa2025tunensearch}.
Although these methods perform favorably without requiring training from scratch, they struggle to scale due to the combinatorial explosion of VRP variants. This is because arbitrary combinations of orthogonal attributes can introduce new variants, inevitably leading to costly repetitive fine-tuning and adapter proliferation.
As an alternative, unified neural solvers instead are capable of handling numerous VRP variants simultaneously.  
These methods typically unify VRP variants via attribute composition~\cite{liu2024multi}, and propose various architectural innovations~\cite{zhou2024mvmoe,berto2024routefinder,li2024cada}. Although these methods treat VRP variants as combinations of attributes, they fail to fully exploit this compositional structure inherent in VRP variants. This structure implies that attributes of each VRP variant actually derive from a shared set of basis VRP variants, each characterized by a unique attribute. We argue that there exist basis solvers, each tailored to a specific basis VRP variant, may enjoy experience valuable for unified solvers.

To elucidate the possible valuable insights that basis neural solvers could offer for unified solvers, we exemplify the widely-used encoder-decoder based neural solver as a case study. 
The encoder is tasked with yielding static node embeddings for the subsequent step-wise solution decoding. Thus, each basis solver's encoder produces embeddings capturing unique attribute of its corresponding basis VRP variant. Since attributes of each VRP variant stem from shared basis variants, the unified solver's encoder possibly perceive both individual attribute information and inter-attribute connections in its embeddings. 
The decoder gradually builds solutions using static node embeddings and the dynamic context while adhering to constraints. Likewise, two critical properties emerge within each VRP variant: its dynamic context can be broken down into conditionally independent components, each corresponding to a distinct basis variant; and its constraints form a superset of those associated with corresponding basis variants. The internal embeddings of the unified solver's decoder thus may partially coincide with those of the decoders of basis solvers. Therefore, we argue that \emph{the unified solver could benefit from valuable insights of basis solvers}. Moreover, the key insights of a neural solver predominantly reside in its continuous embeddings. This motivates us to reformulate VRPs with the aim of reusing basis solvers in the latent space.

In this paper, to seamlessly bridge between VRP variants and basis VRP variants, we propose a VRPs reformulation through the novel State-Decomposable MDP (SDMDP) framework. To be specific, this framework expresses the state space as the Cartesian product of basis state spaces, each associated with a basis VRP variant. As a result, any state admits a decomposition into conditionally independent basis states. Indeed, as disclosed above, SDMDP is fundamentally grounded in the observation that both static attributes and dynamic contexts of VRP variants originate from their corresponding basis variants. More importantly, SDMDP not only targets towards an optimal unified policy, but also yields optimal basis policies inherently, each tailored to a specific basis variant in cases where a VRP variant comprises only a single attribute. To fully reuse optimal basis policies, we operate under the primary assumptions that any unified policy can generate basis state embeddings and a mixture function exists to map these resulting basis state embeddings to the corresponding state embedding. Built upon this, we further develop a Latent Space-based SDMDP (LS-SDMDP) extension by incorporating the optimal basis policies and the mixture function. In this extension, each basis state inherent in a state is fed to its corresponding optimal basis policy for the basis state embedding. These embeddings are then transformed into the state embedding through the mixture function, which subsequently informs the action selection. Under mild assumptions, LS-SDMDP provably recovers the optimal unified policy for SDMDP. For practical implementation, we introduce the Mixture-of-Specialized-Experts Solver (MoSES), which realizes basis policies through specially designed Low-Rank Adaption (LoRA)~\cite{hu2022lora} experts, each fine-tuned for a specific basis variant from a frozen pretrained backbone model, and implements the mixture function via an adaptive gating mechanism. In addition, we design multiple adaptive gating mechanisms for comparative analysis, and implement MoSES with different backbone networks to show its plug-and-play versatility. Extensive experiments across VRP variants validate the superiority of MoSES against prior methods.

\section{Related Works}

This Section reviews recent advances in task-specific neural VRP solvers, multi-task learning approaches for VRPs, and mixture-of-specialized-experts (MoSE) methods. Please refer to Appendix~\ref{appendix:related_work} for more detailed reviews.

\textbf{Task-Specific Neural VRP solvers.} Neural solvers, individually developed for each specific VRP, fall into three main categories: constructive methods end-to-end infer solutions via autoregressive mechanisms~\cite{vinyals2015pointer,kool2018attention,kwon2020pomo,kim2022sym,manchanda2022generalization,qiu2022dimes,zhou2023towards,bi2022learning,gao2023towards,grinsztajn2023winner,jiang2024ensemble,drakulic2024bq,luo2024neural,luo2025boosting}, iterative methods refine solutions via local search operators until convergence~\cite{lu2019learning,chen2019learning,hottung2020neural,ma2021learning,xin2021neurolkh,ma2023neuopt}, and divide-and-conquer methods decompose problem instances into smaller solvable sub-instances~\cite{fu2021generalize,kim2021learning,cheng2023select,li2021learning,zong2022rbg,hou2023generalize,ye2024glop,zheng2024udc,pan2025hlgp4cvrp}. However, these methods typically require specialized network architectures and retraining from scratch to handle numerous VRP variants, bringing about excessively high costs.

\textbf{Multi-Task Learning for VRPs.}
To cope with practical scenarios involving multiple VRP variants, the efficient transfer learning is leveraged to obtain specialized neural solvers by considering inherent similarities among VRP variants. Current methods predominantly adopt either full-parameter fine-tuning~\cite{correa2025tunensearch} or problem-specific adapter fine-tuning~\cite{lin2024cross,drakulic2024goal}, both built on a pretrained backbone model. However, these methods struggle with the combinatorial explosion of VRP variants. \emph{Notably, although the LoRA adapter is used for problem-specific adaptation in~\cite{lin2024cross}, its potential as part of a unified solver to handle the exponential growth of variants remains unexplored}. In contrast, unified solvers are designed to handle multiple VRP variants simultaneously. These methods commonly unify VRP variants via attribute compositions~\cite{liu2024multi} and employ various architectural innovations such as mixture of experts (MoE)~\cite{zhou2024mvmoe}, modified Transformer~\cite{berto2024routefinder}, large language model (LLM) based encoder~\cite{jiang2024unco}, dual attention model~\cite{li2024cada}, mixture of depths (MoD)~\cite{goh2025shield}, specialized decoders~\cite{wangefficient,li2025toward}, diffusion model~\cite{lei2025boosting}, or mixed-curvature based encoder~\cite{liu2025mixed}. However, the potential benefits of incorporating explicitly specialized basis solvers remain unexplored. \emph{Notably, while MoE is employed in~\cite{zhou2024mvmoe}, it learns an implicit and less interpretable specialization rather than incorporating off-the-shelf basis solvers as experts}.

\textbf{Mixture of Specialized Experts.} Prevailing MoSE methods can be broadly categorized into two paradigms: merging entire models and module composition. Approaches based on merging entire models seek to combine independently trained models to efficiently achieve the performance comparable to model ensembling or multi-task learning~\cite{wortsman2022modelsoups,matena2022mergingmodels,yadav2023tiesmerging,tam2024merging,daheim2024model,ainsworth2023git,stoica2024zipit,rame2023ratatouille,jin2023dataless,yang2024AdaMerging}. 
However, these approaches exhibit inferior OOD generalization, compared to layer-wise aggregation of expert models. Our implementation aligns more closely with the module composition paradigm which supports the finer-grained aggregation. These methods primarily include: selective adapter averaging~\cite{chronopoulou2023adaptersoup,ponti2023combining}, module fusion via arithmetic operations~\cite{chronopoulou2024language,zhang2023composing,ilharco2023editing}, adapter routing based on task similarity~\cite{lv2023parameter,gou2023mixture,wu2023pituning}, and adaptive mixture of LoRA experts (MoLE)~\cite{huang2024lorahub,ye2022eliciting,wu2024mixture,liu2024adamole,caccia2023multihead,dou2024loramoe,pfeiffer2021adapterfusion,muqeeth2024learning,ostapenko2024towards,zadouri2024pushing,gao2024higherlayersneedlora}. However, the potential of MoLE methods for unified VRP solvers remains underexplored, and existing composition methods possibly prove inadequate for multi-task VRP scenarios.

\section{Preliminaries}
\textbf{VRP Variants.} We adopt the vehicle routing environment from~\cite{berto2024routefinder}. A capacitated VRP (CVRP) instance of size $N$ is defined on a graph $\mathcal{G} = \{ \mathcal{V}, \mathcal{E} \}$ with nodes $\mathcal{V}=\{v_{0} \} \bigcup \{ v_{i} \}_{i=1}^{N}$ (depot and customers) and edges $\mathcal{E} = \{e(v_{i}, v_{j})|0 \leq i \neq j \leq N \}$. Each node $v_{i}$ $(i\geq0)$ has coordinates $(x_{i}, y_{i})$, with each customer $v_{j}$ $(j\geq1)$ having demand $q_{j}^{\mathrm{LH}}>0$. Each vehicle has a capacity $Q$. A feasible solution (i.e., tour) $\tau$ consists of subtours, each beginning and ending at the depot while visiting a customer subset, with each customer visited exactly once and total demand of each subtour not exceeding $Q$. The cost function $c(\cdot)$ is the total Euclidean length of the tour. The objective is to find the optimal tour $\tau^{\ast}$ with the minimal cost. CVRP, as a \emph{Basis VRP Variant}, features the capacity constraint \emph{(C)}, serving as the foundation for deriving the remaining \emph{Basis VRP Variants} by adding one of the following constraints. 1) \emph{Open Route (O):} A binary variable $o$ indicates whether the vehicle needs to return to depot $(o=0)$ or not $(o=1)$; 2) \emph{Backhaul (B):} Each customer $v_{j}$ is either a linehaul with $q_{j}^{\mathrm{LH}} > 0$ or a backhaul with $q_{j}^{\mathrm{BH}} < 0$, where linehauls require deliveries and backhauls require pickups. VRPs with backhaul allow traversing both types in a mixed manner, but linehauls must precede backhauls in each subtour~\cite{ropke2006backhauls}. In VRPs without backhaul, only linehaul customers are present; 3) \emph{Duration Limit (L):} Each subtour's cost cannot exceed a threshold $l^{\mathrm{dur}}$; 4) \emph{Time Window (TW):} Each node $v_{i}$ has a time window $[w_{i}^{\mathrm{beg}}, w_{i}^{\mathrm{end}}]$ and a service duration $w_{i}^{\mathrm{dur}}$, requiring service to begin within this window. Vehicles arriving before $w_{i}^{\mathrm{beg}}$ must wait until the window opens, and all vehicles must return to the depot by $w_{0}^{\mathrm{end}}$. Thus, 16 VRP variants emerge from adding arbitrary combinations of the remaining four constraints to CVRP. Please refer to Appendix~\ref{appendix:vrp_variants}~\ref{appendix:problem_instance} for details of the VRP variants.

\textbf{Learning to Solve VRPs.} We adopt the widely-used attention-based neural network~\cite{kool2018attention,berto2024routefinder,li2024cada} to parameterize the VRP policy $\pi_{\theta}$, which generates feasible solutions autoregressively through masked decoding. The encoder generates static node embeddings, which, with the dynamic context of the constructed partial tour $\tau^{(<t)}$, are fed to the
decoder to output the probabilities of valid nodes for the next node $\tau^{(t)}$. The policy factorizes as $\pi_{\theta}(\tau|\mathcal{G}) = \prod_{t=1}^{T}\pi_{\theta}(\tau^{(t)}|\tau^{(<t)}, \mathcal{G})$, where $T$ is the solution horizon. REINFORCE~\cite{williams1992simple} algorithm with reward $-c(\tau)$ is used to optimize the policy.

\textbf{Mixture of LoRA Experts.} LoRA~\cite{hu2022lora} is a parameter-efficient fine-tuning method that adapts pretrained frozen LLMs through low-rank matrix factorization. For a linear layer with weights $W_{0} \in \mathbb{R}^{d_{1} \times d_{2}}$, it introduces trainable matrices $A \in \mathbb{R}^{r \times d_{2}}$ and $B \in \mathbb{R}^{d_{1} \times r}$ (where $r < \min(d_{1},d_{2})$ denotes the LoRA rank), modifying the forward pass as $h^{\mathrm{out}}= W_{0}h^{\mathrm{in}} + \beta BAh^{\mathrm{in}}$, where $h^{\mathrm{in}} \in \mathbb{R}^{d_{2}}$ and $h^{\mathrm{out}} \in \mathbb{R}^{d_{1}}$ are the input and the output, and $\beta\in(0,1]$. To enhance the cross-task generalization, $K$ task-specific LoRA experts $\{ B_{k}A_{k} \}_{k=1}^{K}$ are integrated into the LLM~\cite{liu2024adamole}. A trainable gating function $G(h^{\mathrm{in}}) = \mathrm{softmax}(W^{G}h^{\mathrm{in}})$ with weights $W^{G} \in \mathbb{R}^{K \times  d_{2}}$ computes coefficients $\{\alpha_{k}\}_{k=1}^{K}$ for LoRA experts, yielding the forward pass as $h^{\mathrm{out}}=W_{0}h^{\mathrm{in}} + \sum_{k=1}^{K}\alpha_{k} B_{k}A_{k}h^{\mathrm{in}}$.

\section{Methodology}

In this Section, we first present the State-Decomposable MDP framework to reformulate VRP variants. To efficiently reuse readily available basis neural solvers, we extend it to the Latent Space-based SDMDP to enable the unified neural solver. Finally, we propose the Mixture-of-Specialized-Experts Solver which implements basis neural solvers using specialized LoRA experts.

\subsection{State-Decomposable MDP}
The State-Decomposable Markov Decision Process (SDMDP) framework is described by a 7-tuple $( \mathcal{S}, \mathcal{A}, \mathcal{P}, \mathcal{R}, \mu, \gamma, \bar{\mathcal{S}})$, which extends the standard MDP by introducing a \emph{Full State Space} $\bar{\mathcal{S}}$. This full state space $\bar{\mathcal{S}}$ can be partitioned into $T+1$ disjoint sets over a finite horizon $T$: $\bar{\mathcal{S}} = \bigcup_{t=0}^{T} \bar{\mathcal{S}}_{t}$. For each time step $0 \leq t \leq T$, the full state space $\bar{\mathcal{S}}_{t}$ can be further represented as the Cartesian product of $n+1$ basis state spaces: 
$\bar{\mathcal{S}}_{t} = \prod_{i=0}^{n}\mathcal{S}_{t}^{(i)}$. 
This allows any full state $\bar{s}_{t} \in \bar{\mathcal{S}}_{t}$ to be broken down into $n+1$ conditionally independent basis states: $\bar{s}_{t} = \{s_{t}^{(i)}\}_{i=0}^{n}$,
where $s_{t}^{(i)} \in \mathcal{S}_{t}^{(i)}$, for $i = 0,\ldots,n$. Please note that the full state $\bar{s}_{t}$ is neither observed nor influenced by the agent.

Likewise, the \emph{State Space} $\mathcal{S}$ can be partitioned as: $\mathcal{S} = \bigcup_{t=0}^{T} \mathcal{S}_{t}$. During each episode, prior to the policy rollout, the initial state space $\mathcal{S}_{0}$ is built by randomly sampling $m+1$ basis state spaces, where $0 \leq m \leq n$. This results in 
$\mathcal{S}_{0} = \prod_{i=0}^{m}\mathcal{S}_{0}^{(b_{i})}$,
where $\forall 0 \leq i \neq j \leq m,\;  0 \leq b_{i} \neq b_{j} \leq n$. \emph{Especially, we stipulate that $\mathcal{S}_{0}^{(b_{0})} = \mathcal{S}_{0}^{(0)}$.} Please note that the value of $m$ may vary across different episodes. The initial state distribution $\mu$ is thus defined on $\mathcal{S}_{0}$. The state space $\mathcal{S}_{t}$ ($0 \leq t \leq T$) is defined as the joint space of $m+1$ basis state spaces evolving from the initial sampled basis state spaces, such that 
$\mathcal{S}_{t} = \prod_{i=0}^{m}\mathcal{S}_{t}^{(b_{i})}$.
The state $s_{t} \in \mathcal{S}_{t}$ can be decomposed into $m+1$ conditionally independent basis states: 
$s_{t} = \{ s_{t}^{(b_{i})} \}_{i=0}^{m}$,
where $s_{t}^{(b_{i})} \in \mathcal{S}_{t}^{(b_{i})}$, for $i = 0,\ldots,m$.

The \emph{Action Space} $\mathcal{A}$ is conditioned on the state $s_{t}$, denoted as $\mathcal{A}_{t}=\mathcal{A}(s_{t})$, indicating that the basis states in the state $s_{t}$ jointly define the feasible action space. The \emph{Transition Probability Function} $\mathcal{P}$ returns the probability distribution of the next state $s_{t+1} \in \mathcal{S}_{t+1}$ given the current state-action pair $(s_{t}, a_{t}) \in \mathcal{S}_{t}\times\mathcal{A}_{t}$. Due to the conditional independence of the basis states, the transition probability function $\mathcal{P}$ factorizes as: $\mathcal{P}(s_{t+1}|s_{t}, a_{t}) = \prod_{i=0}^{m} \mathcal{P}(s_{t+1}^{(b_{i})}|s_{t}^{(b_{i})}, a_{t})$. The \emph{Reward Function} $\mathcal{R}$ maps the state-action pair $(s_{t}, a_{t})$ to a scalar reward $R_{t}$. $\gamma\in(0,1]$ is the discount factor. Given the state $s_{t}$, the policy $\pi$ yields the probability distribution over $\mathcal{A}_{t}$.

We abuse $\tau$ to denote the trajectory $(s_{0}, a_{0}, \ldots, s_{T})$. The objective of the SDMDP framework is to identify an \emph{Optimal Unified Policy} $\pi^{\ast} = \arg\max_{\pi} \mathbb{E}_{s\sim\mu} \mathbb{E}_{\tau \sim (\pi, \mathcal{P})}[\sum_{t=0}^{T-1} \gamma^{t} R_{t}|s_{0}=s] = \arg\max_{\pi} \mathbb{E}_{s\sim\mu} [V^{\pi, \mathcal{P}}(s)]$, where $V^{\pi, \mathcal{P}}$ is the value function. This framework defines the \emph{0-th Basis Task} using the initial state distribution over $\mathcal{S}_{0}^{(0)}$, while the \emph{$i$-th Basis Task} $(1 \leq i \leq n)$ extends this with initial state distribution over $\mathcal{S}_{0}^{(0)} \times \mathcal{S}_{0}^{(i)}$. The $i$-th \emph{Optimal Basis Policy} $\pi^{(i)\ast}$ $(0 \leq i \leq n)$ is the policy that maximizes the expected cumulative rewards for the $i$-th basis Task.

\textbf{\emph{Remark.}} To reformulate VRP variants within the SDMDP framework, we begin by aligning basis states with basis VRP variants. At each step $t$, the state $s_{t}$ decomposes into 5 basis states: 0) \emph{CVRP:} $s_{t}^{(0)}$ consists of node coordinates $\{ x_{i}, y_{i} \}_{i=0}^{N}$, linehaul demands $\{ q_{i}^{\mathrm{LH}} \}_{i=1}^{N}$, and the remaining linehaul capacity $Q_{t}^{\mathrm{LH}}$; 1) \emph{Open Route:} $s_{t}^{(1)}$ introduces the binary variable $o$; 2) \emph{Backhaul:} $s_{t}^{(2)}$ integrates backhaul demands $\{ q_{i}^{\mathrm{BH}} \}_{i=1}^{N}$ and the remaining backhaul capacity $Q_{t}^{\mathrm{BH}}$; 3) \emph{Duration Limit:} $s_{t}^{(3)}$ encompasses the duration limit $l^{\mathrm{dur}}$ and the current traveled length $l_{t}^{\mathrm{cur}}$ along the present subtour; 4) \emph{Time Window:} $s_{t}^{(4)}$ includes time windows $\{ w_{i}^{\mathrm{beg}}, w_{i}^{\mathrm{end}} \}_{i=0}^{N}$, service durations $\{ w_{i}^{\mathrm{dur}} \}_{i=0}^{N}$, and the current time $w_{t}^{\mathrm{cur}}$. Each VRP variant can be formed by composing $(m+1)$ ($0 \leq m  \leq 4$) basis state spaces, using SDMDP's initial sampling mechanism prior to the policy rollout. During rollout, each current basis state evolves conditionally independently from its corresponding previous basis state given the action. The action space is defined by a masking mechanism that filters out nodes according to visitation status and VRP constraints. Please refer to Appendix~\ref{appendix:vrp_constraints}~\ref{appendix:vrp_example} for details on VRP constraints and formulation. There exist 5 basis tasks, each associated with a basis VRP variant.

\begin{figure*}[t]
\centering
\subfigure[SDMDP framework]{\label{fig:sdmdp}\includegraphics[width=0.46\textwidth]{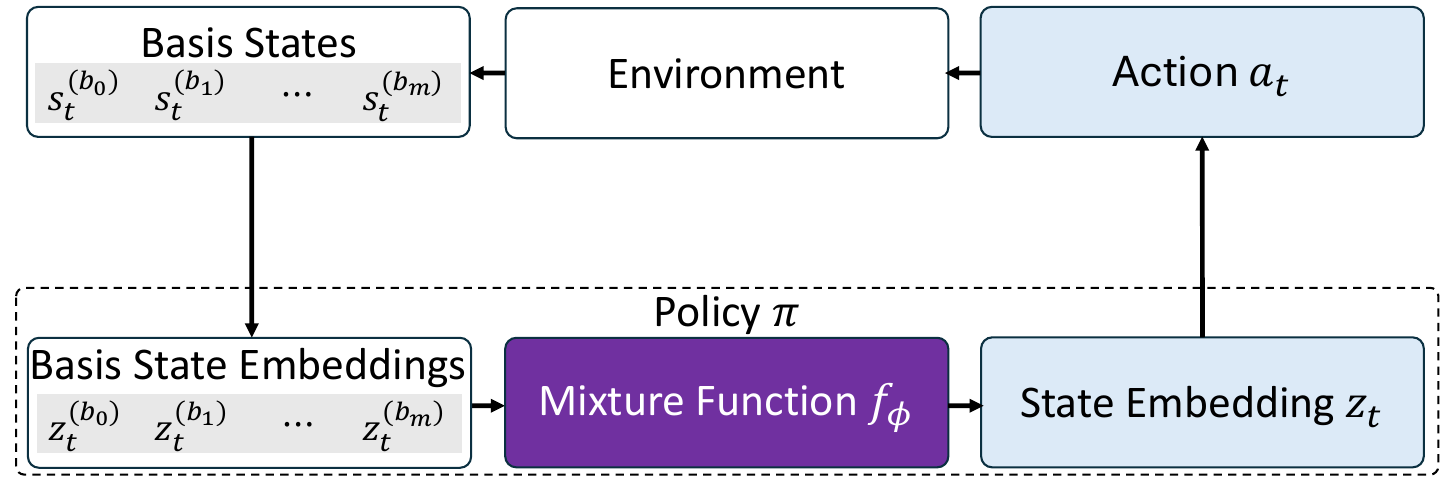}}
\hspace{0.05\textwidth}
\subfigure[LS-SDMDP framework]{\label{fig:ls_sdmdp}\includegraphics[width=0.46\textwidth]{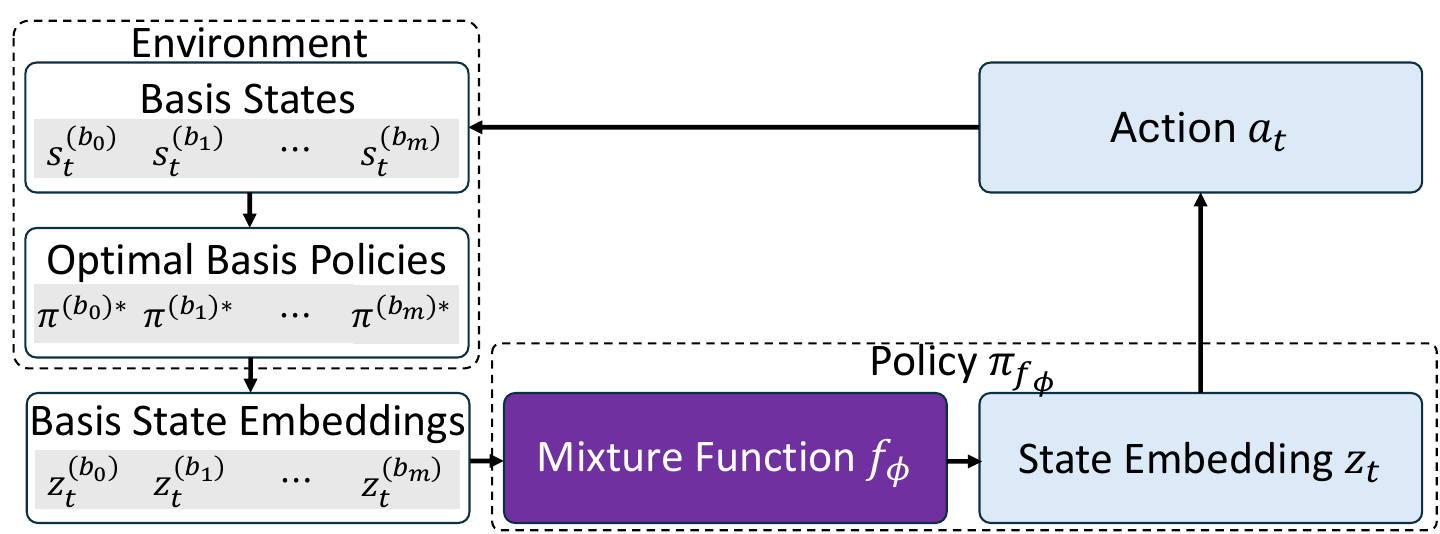}}
\caption{\textbf{Left:} SDMDP employs the policy that incorporates a mixture function. \textbf{Right:} LS-SDMDP integrates both the optimal basis policies and the mixture function. \vspace{-10pt}}
\label{fig:defence}
\end{figure*}

\begin{theorem}
\label{thm:1}
The optimal unified policy $\pi^{\ast}$ and the $i$-th optimal basis policy $\pi^{(i)\ast}$ coincide in their value functions for each state $s_{t}$ associated with the $i$-th basis task: $V^{\pi^{\ast}, \mathcal{P}}(s_{t}) = V^{\pi^{(i)\ast},\mathcal{P}}(s_{t})$. Furthermore, if both the optimal unified policy and the optimal basis policy are unique, then for each state-action pair $(s_{t}, a_{t})$ corresponding to the $i$-th basis task, it holds that $\pi^{\ast}(a_{t}|s_{t}) = \pi^{(i)\ast}(a_{t}|s_{t})$. 
\end{theorem}

\begin{proof}
    Please refer to Appendix~\ref{appendix:proof:1} for a detailed proof of Theorem~\ref{thm:1}.
\end{proof}

\begin{assumption}
\label{assup:1}
In the SDMDP framework, any state $s$ is composed of $m+1$ conditionally independent basis states, denoted as $s = \{s^{(b_{i})} \}_{i=0}^{m}$. 
Accordingly, we assume that any policy $\pi$ is capable of extracting the basis state embedding $z^{(b_{i})} \in \mathbb{R}^{d}$ for each $s^{(b_{i})}$, where $i=0, \ldots, m$. Under this assumption, we further posit that there exists a deterministic bijective mixture function $f_{\phi}:\mathcal{S}\times\prod_{i=0}^{m}\mathbb{R}^{d}\rightarrow\mathbb{R}^{d}$, parameterized by $\phi$, which maps the basis state embeddings 
$\{z^{(b_{i})} \}_{i=0}^{m}$ 
to the state embedding $z \in \mathbb{R}^{d}$ for the given state $s$, represented as 
$z=f_{\phi}(z^{(b_{0})}, \ldots, z^{(b_{m})};s)$. 
Thus, the policy defined over the action space can be rewritten as
\begin{equation}
\small
\label{equ:pi_decomp}
    \pi(a|s) = \sum_{z}\pi(a|z)\pi(z|s)=\sum_{z^{(b_{0})}, \dots, z^{(b_{m})}} \pi(a|f_{\phi}(z^{(b_{0})}, \dots, z^{(b_{m})}; s))\prod_{i=0}^{m}\pi(z^{(b_{i})}|s^{(b_{i})})
\end{equation}
\end{assumption}
where $\prod_{i=0}^{m}\pi(z^{(b_{i})}|s^{(b_{i})})=\pi(z^{(b_{0})}, \dots, z^{(b_{m})}|s)$. The second equivalence in Equation~\ref{equ:pi_decomp} holds because $f_{\phi}$ is assumed to be deterministic.

\begin{assumption}
\label{assup:2}
For any state $s$, and for any two policies $\pi$ and $\pi^{\prime}$, we assume that if $\forall a\in \mathcal{A}(s), \pi(a|s) = \pi^{\prime}(a|s)$, then $\forall z\in\mathbb{R}^{d}, \pi(z|s) = \pi^{\prime}(z|s)$, and conversely.
\end{assumption}

\textbf{\emph{Remark.}} Theorem~\ref{thm:1} discloses the partial consensus between the optimal unified policy and each optimal basis policy, laying the foundation of the policy reuse. Motivated by the observation that neural networks primarily encode knowledge in their embeddings, with distinct features captured at different layers, we develop a latent space approach for reusing optimal basis policies. Assumptions~\ref{assup:1}~\ref{assup:2} serve as the prerequisites for Theorem~\ref{thm:2}. Figure~\ref{fig:sdmdp} illustrates the SDMDP framework.

\subsection{Latent Space-based SDMDP}

The Latent Space-based SDMDP (LS-SDMDP) incorporates all elements of SDMDP, and extends SDMDP by introducing $(\pi^{(0)\ast}, \ldots, \pi^{(n)\ast}, f_{\phi})$. During each episode, at each time step $0 \leq t\leq T$, a state $s_{t} \in \mathcal{S}_{t}$ is sampled from either the initial distribution $\mu(s_{0})$ or the transition probability function $\mathcal{P}(s_{t+1}|s_{t}, a_{t})$, both of which are well defined within the SDMDP framework. Following that, each component $s_{t}^{(b_{i})}$ of $s_{t}$ ($0 \leq i \leq m$) is fed to the corresponding optimal basis policy $\pi^{(b_{i})\ast}$ for the embedding $z_{t}^{(b_{i})} \in \mathbb{R}^{d}$. Collectively, these embeddings form a tuple, denoted as $(z_{t}^{(b_{0})}, \ldots, z_{t}^{(b_{m})})$. To infer the state embedding $z_{t}$ for the state $s_{t}$, $f_{\phi}$ takes as input this tuple along with the state, formally expressed as $z_{t} = f_{\phi}(z_{t}^{(b_{0})}, \ldots, z_{t}^{(b_{m})}; s_{t})$. The policy, exclusively designed for LS-SDMDP, directly observes the embedding vector $z_{t}$ and outputs an action distribution, denoted as $\pi_{z}$. For brevity, the integration of $f_{\phi}$ into the policy $\pi_{z}$ is abbreviated as $\pi_{f_{\phi}}$, such that $\pi_{z}(\cdot|z_{t}) = \pi_{f_{\phi}}(\cdot|z_{t}^{(b_{0})}, \ldots, z_{t}^{(b_{m})};s)$. Accordingly, the underlying initial distribution over $z_{0}$ and the transition probability function for $z_{t+1}$ given $s_{t}$ and $a_{t}$ are written as follows 
\begin{equation}
\small
\label{equ:dist_decomp}
\mu_{z}(z_{0}) =\mu(s_{0})\prod_{i=0}^{m}\pi^{(b_{i})\ast}(z_{0}^{(b_{i})}|s_{0}^{(b_{i})}); \quad
\mathcal{P}_{z}(z_{t+1}|s_{t}, a_{t}) = \mathcal{P}(s_{t+1}|s_{t}, a_{t}) \prod_{i=0}^{m}\pi^{(b_{i})\ast}(z_{t+1}^{(b_{i})}|s_{t+1}^{(b_{i})}).
\end{equation}
Equation~\ref{equ:dist_decomp} holds due to the deterministic and bijective nature of $f_{\phi}$. The reward function remains unchanged. The objective is to discover an optimal unified policy $\pi_{f_{\phi}}^{\ast}$ which can maximize the expected discounted cumulative rewards.

\begin{theorem}
\label{thm:2}
Let $J(\pi, \mathcal{P}, \mu)$ and $J(\pi_{f_{\phi}}, \mathcal{P}_{z}, \mu_{z})$ denote the objective functions (expected returns) in SDMDP and LS-SDMDP, respectively. By Theorem~\ref{thm:1} and Assumptions~\ref{assup:1}~\ref{assup:2}, it follows that the values of the objective functions are equal at their respective optimal policies, formally written as $J(\pi^{\ast}, \mathcal{P}, \mu) = J(\pi_{f_{\phi}}^{\ast}, \mathcal{P}_{z}, \mu_{z})$. Moreover, the value functions of SDMDP and LS-SDMDP at their respective optimal policies satisfy the following relationship $V^{\pi^{\ast}, \mathcal{P}}(s) = \mathbb{E}_{z^{(b_{0})} \sim \pi^{(b_{0})\ast}}\cdots \mathbb{E}_{z^{(b_{m})} \sim \pi^{(b_{m})\ast}}V^{\pi_{f_{\phi}}^{\ast}, \mathcal{P}_{z}}(z^{(b_{0})},\ldots,z^{(b_{m})};s)$.
\end{theorem}

\begin{proof}
    Please refer to Appendix~\ref{appendix:proof:2} for a detailed proof of Theorem~\ref{thm:2}.
\end{proof}

\textbf{\emph{Remark.}} Theorem~\ref{thm:2} indicates that the optimal unified policy $\pi^{\ast}$ of SDMDP can be recovered given both the optimal unified policy $\pi_{f_{\phi}}^{\ast}$ of LS-SDMDP and optimal basis policies $(\pi^{(0)\ast}, \ldots, \pi^{(n)\ast})$. Figure~\ref{fig:ls_sdmdp} depicts the LS-SDMDP framework.

\subsection{Mixture-of-Specialized-Experts Solver}
For the practical implementation, we need to realize both the optimal basis policies and the mixture function to enable the policy reuse in the latent space. If we consider the embedding $z^{(b_{i})}$ ($0 \leq i \leq m$) generated by the optimal basis policy $\pi^{(b_{i})\ast}$ to encapsulate all layer-wise embeddings, the mixture function would effectively perform entire-model merging. However, as disclosed in ~\cite{ostapenko2024towards, muqeeth2024learning}, such entire-model merging leads to inferior OOD generalization. We therefore introduce the Mixture-of-Specialized-Experts Solver (MoSES), which implements a layer-wise and token-wise aggregation approach through three stages: 1) pretraining a shared backbone model; 2) fine-tuning specialized experts, and 3) dynamically aggregating experts.

\textbf{Pretraining a shared backbone model.} Since all the other basis VRP variants are derived from CVRP by adding different constraints, we pretrain the shared backbone model exclusively on CVRP instances and freeze its parameters throughout subsequent stages. Please note the backbone model indeed acts as the 0-th optimal basis policy $\pi^{(0)\ast}$.

\textbf{Fine-tuning specialized experts.} To obtain the optimal basis policy $\pi^{(i)\ast}$ ($i>0$), we employ the specialized LoRA expert that performs parameter-efficient fine-tuning on the frozen backbone model. Problem instances of any basis VRP variant (excluding CVRP) are OOD inputs to the frozen backbone model, leading to task-misaligned features in embeddings generated from the backbone model. We thus propose the Gated-LoRA method that uses a dynamic gating mechanism with the trainable weights $W^{g} \in \mathbb{R}^{d_{2}}$ to suppress these irrelevant features from the backbone. Specifically, for the $i$-th optimal basis policy $\pi^{(i)\ast}$ ($i>0$), the forward pass through a frozen backbone linear layer with weights $W_{0}$ augmented with trainable LoRA matrices $B_{i}$ and $A_{i}$ is computed as:
\begin{equation}
    h^{\mathrm{out}} = \mathrm{sigmoid}(\langle W^{g},h^{\mathrm{in}} \rangle)W_{0}h^{\mathrm{in}} + B_{i}A_{i}h^{\mathrm{in}}.
\end{equation}

\textbf{Dynamically aggregating experts.} In this stage, we freeze both the backbone weights $W_{0}$ and the LoRA weights $\{ B_{i}A_{i} \}_{i=1}^{4}$ in the unified policy. The adaptive gating function $G(h^{\mathrm{in}}) = \mathrm{act}(W^{G}h^{\mathrm{in}})$ with weights $W^{G}\in\mathbb{R}^{5 \times d_{2}}$ computes coefficients $\{ \alpha_{i} \}_{0}^{4}$, where $\mathrm{act}(\cdot)$ is the activation function. Let $\tilde{h}^{\mathrm{out}}$ denote the output of the unknown optimal unified solver. The aggregation goal is to determine coefficients $\{ \alpha_{i} \}_{0}^{4}$ satisfying the linear system $\sum_{i=0}^{4}\alpha_{i}W_{i}h^{\mathrm{in}} = \tilde{h}^{\mathrm{out}}$ (where $W_{i} = B_{i}A_{i}$, for $i \geq 1$). However, this linear system may not be solvable. We thus introduce the trainable LoRA weights $\hat{B}\hat{A}$ to model the residual. The forward pass of a linear layer in the unified policy is written as:
\begin{equation}
    h^{\mathrm{out}} = \alpha_{0}W_{0}h^{\mathrm{in}} + \sum_{i=1}^{4}\alpha_{i}B_{i}A_{i}h^{\mathrm{in}} + \hat{B}\hat{A}h^{\mathrm{in}}; \quad \{\alpha_{i}\}_{i=0}^{4} = G(h^{\mathrm{in}})=\mathrm{act}(W^{G}h^{\mathrm{in}}).
\end{equation}
We propose three activation functions for the gating mechanism: 1) $\mathrm{softmax}(\cdot)$ enforces the convex combination; 2) $\mathrm{norm\_softplus}(\cdot)$ 
$L_{1}$-normalizes $\mathrm{softplus}(W^{G}h^{\mathrm{in}})$ to prevent gradient vanishing; 3) $\mathrm{sigmoid}(\cdot)$ expands the coefficient space by relaxing the unit-sum constraint. Additionally, we introduce three routing strategies: 1) \emph{Dense Routing} activates all optimal basis policies ; 2) \emph{Variant-Aware Routing-\uppercase\expandafter{\romannumeral1}} selects Top-$K$ optimal basis policies, where $K$ is the number of basis variants in the current VRP variant. 3) \emph{Variant-Aware Routing-\uppercase\expandafter{\romannumeral2}} selects optimal basis policies corresponding to basis variants present in the current VRP variant;

\section{Experiments}
In this Section, we empirically validate the superiority of MoSES through evaluations on 16 VRP variants with five constraints, supplemented by hyperparameter and ablation studies. All experiments are conducted on NVIDIA Tesla V100-32GB GPUs on and Intel(R) Xeon(R) Platinum 8255C CPU @ 2.50GHz. Please refer to Appendix~\ref{appendix:experiment} for more empirical results, hyperparameter and ablation studies, along with the analysis.

\begin{table*}[!t]
    \caption{Performance comparison of multi-task VRP solvers. The lower, the better ($\downarrow$).}
    \label{tab:main}
    \begin{center}
    \renewcommand\arraystretch{1.05}
    \resizebox{0.98\textwidth}{!}{ 
    \begin{tabular}{ll|cccccc|ll|cccccc}
      \toprule
      \multicolumn{2}{c|}{\multirow{2}{*}{Solver}} & \multicolumn{3}{c}{\textbf{$N=50$}} & \multicolumn{3}{c|}{$N=100$} & \multicolumn{2}{c|}{\multirow{2}{*}{Solver}} &
      \multicolumn{3}{c}{\textbf{$N=50$}} & \multicolumn{3}{c}{$N=100$} \\
      
      \cmidrule(lr){3-5} \cmidrule(lr){6-8} \cmidrule(lr){11-13} \cmidrule(lr){14-16} 
      
       & & Cost & Gap & Time & Cost & Gap & Time & & & Cost & Gap & Time & Cost & Gap & Time \\
      \midrule
      
\multirow{10}*{\rotatebox{90}{CVRP}} & HGS-PyVRP & 10.372 & * & 10.4m & 15.628 & * & 20.8m  & 
\multirow{10}*{\rotatebox{90}{VRPTW}} & HGS-PyVRP & 16.031 & * & 10.4m & 25.423 & * & 20.8m  \\

& OR-Tools & 10.572 & 1.907\% & 10.4m & 16.280 & 4.178\% & 20.8m & 
& OR-Tools & 16.089 & 0.347\% & 10.4m & 25.814 & 1.506\% & 20.8m  \\

& MTPOMO  & 10.518  & 1.408\% & 2s    & 15.933 & 1.986\% & 8s   &  
& MTPOMO  & 16.409  & 2.358\% & 2s    & 26.410 & 3.863\% & 9s   \\

& MVMoE   & 10.501  & 1.241\% & 3s    & 15.888 & 1.692\% & 11s  &  
& MVMoE   & 16.405  & 2.333\% & 3s    & 26.391 & 3.793\% & 11s  \\

& RF-POMO & 10.508  & 1.315\% & 2s    & 15.908 & 1.830\% & 8s   &  
& RF-POMO & 16.366  & 2.089\% & 2s    & 26.335 & 3.570\% & 9s   \\
 
& RF-MoE  & 10.499  & 1.228\% & 3s    & 15.877 & 1.624\% & 11s  &  
& RF-MoE  & 16.390  & 2.239\% & 3s    & 26.319 & 3.506\% & 11s  \\

\cdashline{2-8} \cdashline{10-16}

& RF-TE   & 10.504  & 1.276\% & 2s    & 15.857 & 1.507\% & 8s   &  
& RF-TE   & 16.363  & 2.069\% & 2s    & 26.234 & 3.177\% & 8s   \\

& MoSES(RF) & \textbf{10.465} & \textbf{0.900\%} & 6s   & \textbf{15.808} & \textbf{1.190\%} & 21s &
& MoSES(RF) & \textbf{16.264} & \textbf{1.445\%} & 6s   & \textbf{26.143} & \textbf{2.822\%} & 21s \\

\cdashline{2-8} \cdashline{10-16}

& CaDA    & 10.483 & 1.072\% & 3s     & \textbf{15.831} & \textbf{1.336\%} & 10s   &  
& CaDA    & 16.297 & 1.652\% & 2s     & 26.128 & 2.753\% & 10s   \\

& MoSES(CADA) & \textbf{10.462} & \textbf{0.873\%} & 7s & 15.833 & 1.354\% & 24s &
& MoSES(CaDA) & \textbf{16.262} & \textbf{1.435\%} & 7s & \textbf{26.032} & \textbf{2.383\%} & 25s \\ 

\midrule

\multirow{10}*{\rotatebox{90}{OVRP}} & HGS-PyVRP & 6.507 & * & 10.4m & 9.725 & * & 20.8m  & 
\multirow{10}*{\rotatebox{90}{VRPL}} & HGS-PyVRP & 10.587 & * & 10.4m & 15.766 & * & 20.8m \\

& OR-Tools & 6.553 & 0.686\%  & 10.4m & 9.995 & 2.732\% & 20.8m  & 
& OR-Tools & 10.570 & 2.343\% & 10.4m & 16.466 & 5.302\% & 20.8m     \\

& MTPOMO  & 6.718   & 3.211\%  & 2s    & 10.210  & 4.959\%  & 8s   &  
& MTPOMO  & 10.775  & 1.732\%  & 2s    & 16.151  & 2.445\%  & 8s   \\

& MVMoE   & 6.702  & 2.969\%  & 3s   & 10.176  & 4.615\%   & 11s  &  
& MVMoE   & 10.751 & 1.508\%  & 3s   & 16.099  & 2.117\%   & 11s  \\

& RF-POMO & 6.698  & 2.906\%  & 2s   & 10.181  & 4.671\%   & 8s   &  
& RF-POMO & 10.751 & 1.525\%  & 2s   & 16.106  & 2.166\%   & 8s   \\

& RF-MoE  & 6.697   & 2.879\%  & 3s   & 10.139  & 4.238\%   & 11s  &  
& RF-MoE  & 10.737  & 1.388\%  & 3s   & 16.070  & 1.937\%   & 11s  \\

\cdashline{2-8} \cdashline{10-16}

& RF-TE   & 6.684  & 2.693\%   & 2s   & 10.121 & 4.060\%  & 8s   & 
& RF-TE   & 10.748 & 1.499\%   & 2s   & 16.051 & 1.829\%  & 8s   \\

& MoSES(RF) & \textbf{6.632} & \textbf{1.892\%} & 5s & \textbf{10.064} & \textbf{3.469\%} & 20s &
& MoSES(RF) & \textbf{10.704} & \textbf{1.089\%} & 5s & \textbf{16.005} & \textbf{1.532\%} & 20s \\

\cdashline{2-8} \cdashline{10-16}

& CaDA    & 6.662  & 2.350\% & 2s   & 10.093  & 3.763\%  & 10s   &  
& CaDA    & 10.722 & 1.252\% & 2s   & 16.062  & 1.662\% &  10s   \\

& MoSES(CaDA) & \textbf{6.629} & \textbf{1.857\%} & 7s & \textbf{10.084} & \textbf{3.679\%} & 24s &
& MoSES(CaDA) & \textbf{10.704} & \textbf{1.083\%} & 7s & \textbf{16.024} & \textbf{1.659\%} & 24s \\

\midrule

\multirow{10}*{\rotatebox{90}{VRPB}} & HGS-PyVRP & 9.687 & * & 10.4m & 14.377 & * & 20.8m   & 
\multirow{10}*{\rotatebox{90}{OVRPTW}} & HGS-PyVRP & 10.510 & * & 10.4m & 16.926 & * & 20.8m   \\

& OR-Tools & 9.802 & 1.159\% & 10.4m & 14.933 & 3.853\% & 20.8m & 
& OR-Tools & 10.519 & 0.078\% & 10.4m & 17.027 & 0.583\% & 20.8m  \\

& MTPOMO  & 10.033   & 3.564\%   & 2s   & 15.082   & 4.917\%   & 8s   &  & MTPOMO  & 10.667   & 1.472\%   & 2s   & 17.421   & 2.896\%   & 9s   \\

& MVMoE   & 10.005   & 3.268\%   & 3s   & 15.022   & 4.506\%   & 10s  &  & MVMoE   & 10.669   & 1.495\%   & 3s   & 17.416   & 2.874\%   & 12s  \\

& RF-POMO & 9.996   & 3.173\%  & 2s   & 15.016  & 4.465\%   & 8s   &  
& RF-POMO & 10.657  & 1.376\%  & 2s   & 17.392  & 2.725\%   & 9s   \\

& RF-MoE  & 9.980   & 3.014\%  & 3s   & 14.973   & 4.165\%   & 10s  &  
& RF-MoE  & 10.673  & 1.533\%  & 3s   & 17.387   & 2.698\%   & 12s  \\

\cdashline{2-8} \cdashline{10-16}

& RF-TE   & 9.978   & 2.996\%   & 2s   & 14.942   & 3.950\%  & 8s   &  
& RF-TE   & 10.652  & 1.328\%   & 2s   & 17.326   & 2.341\%  & 9s   \\

& MoSES(RF) & \textbf{9.915}  & \textbf{2.342\%} & 5s & \textbf{14.884} & \textbf{3.546\%} & 19s &
& MoSES(RF) & \textbf{10.613} & \textbf{0.959\%} & 6s & \textbf{17.284} & \textbf{2.101\%} & 21s \\

\cdashline{2-8} \cdashline{10-16}

& CaDA    & 9.945  & 2.654\%  & 2s   & 14.905  & 3.684\%  & 10s   &  
& CaDA    & 10.621 & 1.037\%  & 3s   & 17.253  & 1.906\%  & 11s   \\

& MoSES(CaDA) & \textbf{9.904} & \textbf{2.225\%} & 7s & \textbf{14.901} & \textbf{3.668\%} & 23s &
& MoSES(CaDA) & \textbf{10.611} & \textbf{0.946\%} & 8s & \textbf{17.217} & \textbf{1.702\%} & 26s \\

\midrule

\multirow{10}*{\rotatebox{90}{VRPBL}} & HGS-PyVRP & 10.186 & * & 10.4m & 14.779 & * & 20.8m & 
\multirow{10}*{\rotatebox{90}{VRPBLTW}} &  HGS-PyVRP  & 18.361 & * & 10.4m  & 29.026 & *  & 20.8m     \\

& OR-Tools & 10.331 & 1.390\% & 10.4m & 15.426 & 4.338\% & 20.8m  & 
& OR-Tools & 18.422 & 0.332\% & 10.4m & 29.830 & 2.770\% & 20.8m   \\

& MTPOMO  & 10.672  & 4.699\%  & 2s   & 15.712  & 6.253\%  & 8s   &  
& MTPOMO  & 18.990  & 2.130\%  & 3s   & 30.896  & 3.616\%  & 9s   \\

& MVMoE   & 10.637  & 4.349\%   & 3s   & 15.640  & 5.763\%    & 11s  &  
& MVMoE   & 18.986  & 2.106\%   & 3s   & 30.893  & 3.612\%    & 12s  \\
 
& RF-POMO & 10.592   & 3.937\%   & 2s   & 15.628  & 5.696\%   & 8s   &  
& RF-POMO & 18.937   & 1.853\%   & 2s   & 30.794  & 3.278\%   & 9s   \\

& RF-MoE  & 10.575  & 3.767\%    & 3s   & 15.541  & 5.121\%   & 10s  &  
& RF-MoE  & 18.956  & 1.956\%    & 3s   & 30.807  & 3.321\%   & 12s  \\

\cdashline{2-8} \cdashline{10-16}

& RF-TE   & 10.578  & 3.798\%   & 2s   & 15.528  & 5.038\%   & 8s   &  
& RF-TE   & 18.941  & 1.877\%   & 2s   & 30.688  & 2.923\%   & 9s   \\

& MoSES(RF) & \textbf{10.518} & \textbf{3.185\%} & 6s & \textbf{15.469} & \textbf{4.638\%} & 20s &
& MoSES(RF) & \textbf{18.846} & \textbf{1.370\%} & 6s & \textbf{30.627} & \textbf{2.712\%} & 22s \\

\cdashline{2-8} \cdashline{10-16}

& CaDA   & 10.535 & 3.379\% & 2s  & 15.481 &  4.713\% & 10s   &  
& CaDA   & 18.877 & 1.531\% & 2s  & 30.586 &  2.579\% & 11s   \\

& MoSES(CaDA) & \textbf{10.517} & \textbf{3.193\%} & 7s & \textbf{15.478} & \textbf{4.705\%} & 24s &
& MoSES(CaDA) & \textbf{18.858} & \textbf{1.425\%} & 8s & \textbf{30.510} & \textbf{2.329\%} & 26s \\

\midrule

\multirow{10}*{\rotatebox{90}{VRPBTW}} 
& HGS-PyVRP & 18.292 & * & 10.4m & 29.467 & * & 20.8m & 
\multirow{10}*{\rotatebox{90}{VRPLTW}} & HGS-PyVRP & 16.356 & * & 10.4m & 25.757 & * & 20.8m \\

& OR-Tools & 18.366 & 0.383\% & 10.4m & 29.945 & 1.597\% & 20.8m & 
& OR-Tools & 16.441 & 0.499\% & 10.4m & 26.259 & 1.899\% & 20.8m    \\

& MTPOMO  & 18.639  & 1.876\%  & 2s   & 30.435  & 3.278\%  & 9s   &  
& MTPOMO  & 16.823  & 2.818\%  & 2s   & 26.891  & 4.364\%  & 9s   \\

& MVMoE   & 18.640  & 1.884\%  & 3s   & 30.438  & 3.287\%   & 12s  &  
& MVMoE   & 16.811  & 2.751\%  & 3s   & 26.866  & 4.271\%   & 12s  \\

& RF-POMO & 18.601  & 1.669\%  & 2s   & 30.343  & 2.967\%  & 9s   &  
& RF-POMO & 16.750  & 2.383\%  & 2s   & 26.784  & 3.951\%  & 9s   \\

& RF-MoE  & 18.617  & 1.760\%  & 3s   & 30.339 & 2.947\%  & 12s  &  
& RF-MoE  & 16.776  & 2.547\%  & 3s   & 26.775 & 3.918\%  & 12s  \\

\cdashline{2-8} \cdashline{10-16}

& RF-TE   & 18.600  & 1.675\%  & 2s   & 30.240 & 2.618\%   & 9s   &  
& RF-TE   & 16.763  & 2.460\%  & 2s   & 26.691 & 3.587\%   & 9s   \\

& MoSES(RF) & \textbf{18.499} & \textbf{1.121\%} & 6s & \textbf{30.148} & \textbf{2.303\%} & 21s &
& MoSES(RF) & \textbf{16.657} & \textbf{1.811\%} & 6s & \textbf{26.620} & \textbf{3.320\%} & 21s \\

\cdashline{2-8} \cdashline{10-16}

& CaDA   &  18.534 &  1.302\% & 3s   &  30.131 &  2.242\% & 11s   &  
& CaDA   &  16.694 &  2.038\% & 2s   &  26.592 &  3.204\% & 11s   \\

& MoSES(CaDA) & \textbf{18.495} & \textbf{1.095\%} & 8s & \textbf{30.050} & \textbf{1.969\%} & 25s &
& MoSES(CaDA) & \textbf{16.667} & \textbf{1.864\%} & 8s & \textbf{26.493} & \textbf{2.824\%} & 25s \\

\midrule

\multirow{10}*{\rotatebox{90}{OVRPB}} & HGS-PyVRP & 6.898 & * & 10.4m & 10.335 & * & 20.8m  & 
\multirow{10}*{\rotatebox{90}{OVRPBL}} & HGS-PyVRP & 6.899 & * & 10.4m & 10.335 & * & 20.8m \\

& OR-Tools & 6.928 & 0.412\% & 10.4m & 10.577 & 2.315\% & 20.8m & 
& OR-Tools & 6.927 & 0.386\% & 10.4m & 10.582 & 2.363\% & 20.8m  \\

& MTPOMO  & 7.108  & 3.004\%  & 2s   & 10.878  & 5.224\%  & 8s   &  
& MTPOMO  & 7.112  & 3.056\%  & 2s   & 10.883  & 5.272\%  & 8s   \\

& MVMoE   & 7.090   & 2.743\%   & 3s   & 10.840  & 4.859\%  & 11s  &  
& MVMoE   & 7.098   & 2.850\%   & 3s   & 10.847  & 4.928\%  & 11s  \\

& RF-POMO & 7.086  & 2.689\%   & 2s   & 10.836   & 4.823\%  & 8s   &  
& RF-POMO & 7.087  & 2.695\%   & 2s   & 10.837   & 4.835\%  & 8s   \\

& RF-MoE  & 7.080   & 2.613\%   & 3s   & 10.806   & 4.526\%   & 11s  &  
& RF-MoE  & 7.083   & 2.635\%   & 3s   & 10.807   & 4.540\%   & 11s  \\

\cdashline{2-8} \cdashline{10-16}

& RF-TE   & 7.071  & 2.477\%    & 2s   & 10.772   & 4.212\%   & 8s   &  
& RF-TE   & 7.075  & 2.515\%    & 2s   & 10.779   & 4.268\%   & 8s   \\

& MoSES(RF) & \textbf{7.037} & \textbf{1.979\%} & 6s & \textbf{10.733} & \textbf{3.829\%} & 20s &
& MoSES(RF) & \textbf{7.040} & \textbf{2.014\%} & 6s & \textbf{10.736} & \textbf{3.862\%} & 20s \\

\cdashline{2-8} \cdashline{10-16}
 
& CaDA  &  7.040  &  2.034\% & 2s   &  \textbf{10.724} &  \textbf{3.738\%} & 10s   &  
& CaDA  &  7.042  &  2.045\% & 2s   &  \textbf{10.723} &  \textbf{3.732\%} & 10s   \\

& MoSES(CaDA) & \textbf{7.034} & \textbf{1.942\%} & 7s & 10.726 & 3.765\% & 24s &
& MoSES(CaDA) & \textbf{7.036} & \textbf{1.964\%} & 7s & 10.724 & 3.743\% & 24s \\

\midrule
      
\multirow{10}*{\rotatebox{90}{OVRPBLTW}} & HGS-PyVRP & 11.668 & * & 10.4m & 19.156 & * & 20.8m 
& \multirow{10}*{\rotatebox{90}{OVRPBTW}} & HGS-PyVRP & 11.669 & * & 10.4m & 19.156 & * & 20.8m  \\

& OR-Tools & 11.681 & 0.106\% & 10.4m & 19.305 & 0.767\% & 20.8m & 
& OR-Tools & 11.682 & 0.109\% & 10.4m & 19.303 & 0.757\% & 20.8m  \\

& MTPOMO  & 11.817  & 1.259\%   & 3s   & 19.637 & 2.494\%   & 9s   &  
& MTPOMO  & 11.814  & 1.231\%   & 3s   & 19.635 & 2.484\%   & 9s   \\

& MVMoE   & 11.823  & 1.303\%   & 4s   & 19.641  & 2.516\%  & 12s  &  
& MVMoE   & 11.819  & 1.272\%   & 4s   & 19.639  & 2.505\%  & 13s  \\

& RF-POMO & 11.805  & 1.155\%   & 3s   & 19.608  & 2.344\%  & 10s  &  
& RF-POMO & 11.804  & 1.148\%   & 3s   & 19.608  & 2.343\%  & 10s  \\

& RF-MoE  & 11.823  & 1.307\%  & 4s   & 19.607  & 2.334\%  & 12s  &  
& RF-MoE  & 11.823  & 1.300\%  & 4s   & 19.606  & 2.327\%  & 12s  \\

\cdashline{2-8} \cdashline{10-16}

& RF-TE   & 11.804   & 1.147\%  & 2s   & 19.551  & 2.045\%   & 9s   &  
& RF-TE   & 11.805   & 1.151\%  & 2s   & 19.551  & 2.046\%   & 9s   \\

& MoSES(RF) & \textbf{11.762} & \textbf{0.791\%} & 6s & \textbf{19.508} & \textbf{1.821\%} & 22s &
& MoSES(RF) & \textbf{11.761} & \textbf{0.783\%} & 6s & \textbf{19.509} & \textbf{1.829\%} & 22s \\

\cdashline{2-8} \cdashline{10-16}

& CaDA   &  11.771 &  0.865\% & 2s   &  19.471 &  1.626\% & 11s  &  
& CaDA   &  11.770 &  0.854\% & 2s   &  19.472 &  1.630\% & 11s   \\

& MoSES(CaDA) & \textbf{11.761} & \textbf{0.781\%} & 8s & \textbf{19.440} & \textbf{1.470\%} & 26s &
& MoSES(CaDA) & \textbf{11.760} & \textbf{0.773\%} & 8s & \textbf{19.441} & \textbf{1.475\%} & 26s \\

\midrule
\multirow{10}*{\rotatebox{90}{OVRPL}} & HGS-PyVRP & 6.507 & * & 10.4m & 9.724 & * & 20.8m 
& \multirow{10}*{\rotatebox{90}{OVRPLTW}} & HGS-PyVRP & 10.510 & * & 10.4m & 16.926 & * & 20.8m \\

& OR-Tools & 6.552 & 0.668\% & 10.4m & 10.001 & 2.791\% & 20.8m  & 
& OR-Tools       & 10.497 & 0.114\% & 10.4m     & 17.023 & 0.728\% & 20.8m     \\

& MTPOMO  & 6.719   & 3.229\%   & 2s   & 10.214  & 5.000\%    & 8s   &  
& MTPOMO  & 10.670  & 1.503\%   & 2s   & 17.420  & 2.892\%    & 9s   \\

& MVMoE   & 6.707   & 3.029\%   & 3s   & 10.184  & 4.697\%   & 11s  &  
& MVMoE   & 10.671  & 1.511\%   & 3s   & 17.418  & 2.881\%   & 12s  \\

& RF-POMO & 6.701   & 2.951\%   & 2s   & 10.180   & 4.662\%   & 8s   &  
& RF-POMO & 10.657  & 1.372\%   & 3s   & 17.392   & 2.727\%   & 9s   \\

& RF-MoE  & 6.696    & 2.870\%   & 3s   & 10.141   & 4.253\%   & 11s  &  
& RF-MoE  & 10.673   & 1.532\%   & 3s   & 17.385   & 2.690\%   & 12s  \\

\cdashline{2-8} \cdashline{10-16}

& RF-TE   & 6.685   & 2.713\%  & 2s   & 10.121 & 4.054\%  & 8s   &  
& RF-TE   & 10.652  & 1.330\%  & 2s   & 17.327 & 2.348\%  & 9s   \\

& MoSES(RF) & \textbf{6.634}  & \textbf{1.917\%} & 6s & \textbf{10.063} & \textbf{3.463\%} & 20s &
& MoSES(RF) & \textbf{10.613} & \textbf{0.962\%} & 6s & \textbf{17.281} & \textbf{2.081\%} & 22s \\

\cdashline{2-8} \cdashline{10-16}

& CaDA  &  6.661  &  2.335\% & 2s   &  10.093 &  3.766\% & 11s   &  
& CaDA  &  10.622 &  1.045\% & 3s   &  17.255 &  1.914\% & 11s   \\

& MoSES(CaDA) & \textbf{6.629} & \textbf{1.846\%} & 7s & \textbf{10.081} & \textbf{3.652\%} & 24s & 
& MoSES(CaDA) & \textbf{10.611} & \textbf{0.940\%} & 8s & \textbf{17.219} & \textbf{1.714\%} & 26s \\

\bottomrule
\end{tabular}}
\end{center}
\end{table*}

\textbf{Baselines.} \emph{Traditional Solvers:} We benchmark against the open-source PyVRP~\cite{Wouda2024PyVRP} solver built on HGS-CVRP~\cite{VIDAL2022Hybrid}, and the widely-used Google OR-Tools~\cite{perron2023ortools}. Both solvers process each instance on a single CPU core with time limits of 10s for instances with 50 nodes; and 20s for instances with 100 nodes, while we parallelize their execution across 16 CPU cores for efficiency. \emph{Neural Solvers:} We compare our method with state-of-the-art unified neural solvers for multi-task VRPs. MTPOMO~\cite{liu2024multi} extends POMO~\cite{kwon2020pomo} and unifies VRP variants. MVMoE~\cite{zhou2024mvmoe} introduces mixture-of-experts. RouteFinder~\cite{berto2024routefinder} utilizes the mixed batch training to produce three models: RF-POMO (MTPOMO-based), RF-MoE (MVMoE-based) and RF-TE (modiled Transformer-based). CaDA~\cite{li2024cada} enhances model capacity through dual attention mechanism. Since CaDA lacks publicly available code, we implement it according to the original paper specifications. Our reproduction of CaDA achieves comparable performance to the reported results.

\textbf{MoSES Architecture.} To demonstrate the plug-and-play versatility of MoSES, we implement it upon both RF-TE and CaDA, denoted as MoSES(RF) and MoSES(CaDA) respectively. MoSES(RF) uses $\mathrm{norm\_softplus}(\cdot)$ activation for its adaptive gating function $G(\cdot)$, while MoSES(CaDA) utilizes $\mathrm{sigmoid}(\cdot)$ activation. Both adopt the dense routing strategy, and set LoRA rank as 32 for both frozen modules $\{ B_{i}A_{i} \}_{i=1}^{4}$ and the trainable module $\hat{B}\hat{A}$.

\textbf{Training and Evaluation.} We consider two
problem scales $N=\{50, 100\}$, and adopt the same training settings with prior works~\cite{berto2024routefinder, li2024cada} for baseline models. We optimize MoSES using REINFORCE~\cite{williams1992simple}. The backbone model is first pretrained on randomly generated CVRP instances. Then, the LoRA adapter fine-tuning produces 4 specialized experts, each optimized for a distinct VRP variant: OVRP, VRPB, VRPL, and VRPTW. Finally, the unified solver is trained on all 16 VRP variants. All phases share the training hyperparameters. Each model undergoes 300 training epochs, each containing 100,000 VRP instances generated on the fly. Adam Optimizer is used with a learning rate of $3 \times 10^{-4}$, weight decay of $1 \times 10^{-6}$, and batch size of 256. We decay the learning rate by a factor of 10 at epochs 270 and 295. During evaluation, each neural solver employs greedy multi-start rollouts with $8\times$ augmentations, selecting the best one from the generated solutions per instance. We report average costs and optimality gaps over 1K test instances, and the total time taken for testing. Gaps are calculated w.r.t. the results of the best heuristic solver (i.e., $\ast$ in Table~\ref{tab:main}).

\begin{figure*} [t]
\centering
\subfigure[RF-based]
{\label{fig:curve-1}\includegraphics[width=0.22\textwidth]{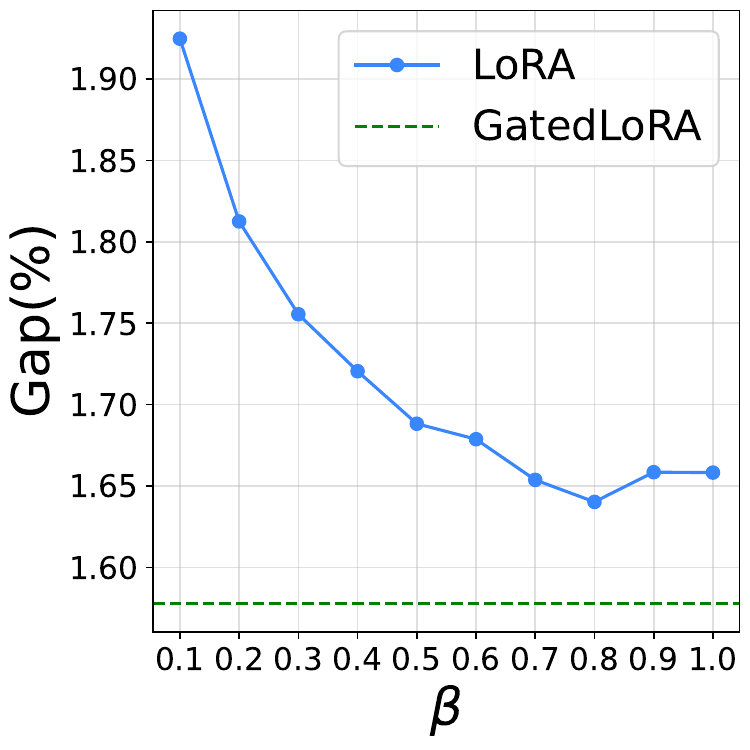}}
\subfigure[CaDA-based]
{\label{fig:curve-2}\includegraphics[width=0.22\textwidth]{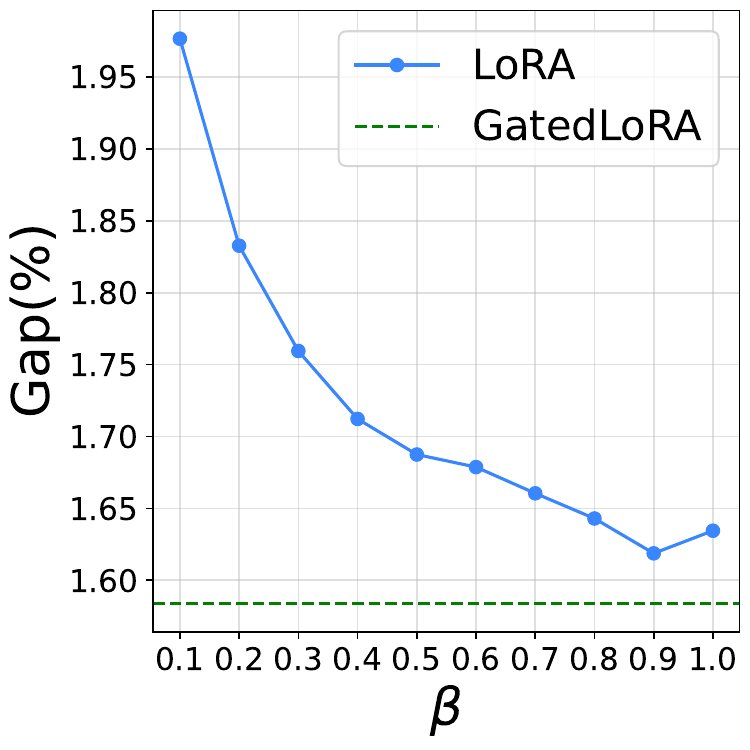}}
\subfigure[$N=50$]
{\label{fig:curve-3}\includegraphics[width=0.22\textwidth]{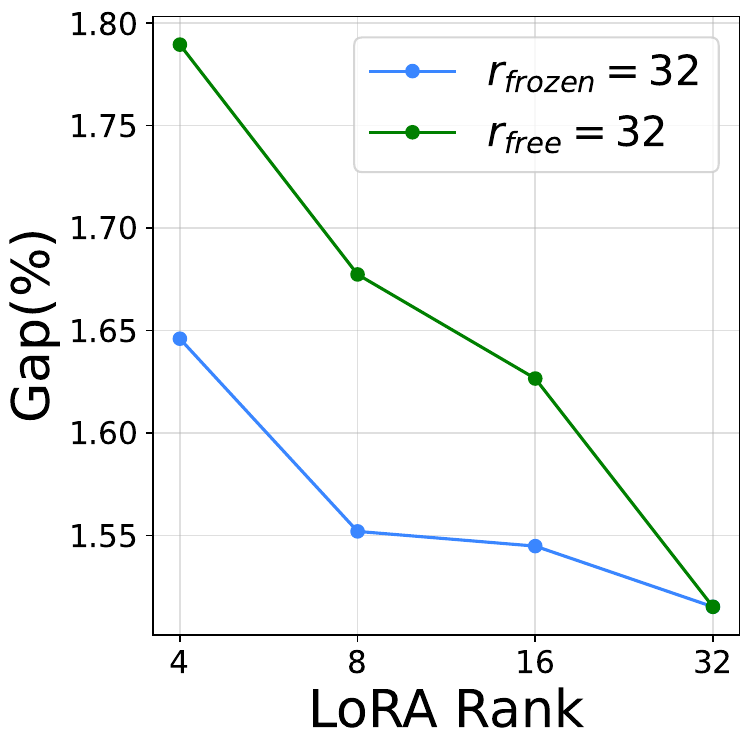}}
\subfigure[$N=100$]
{\label{fig:curve-4}\includegraphics[width=0.22\textwidth]{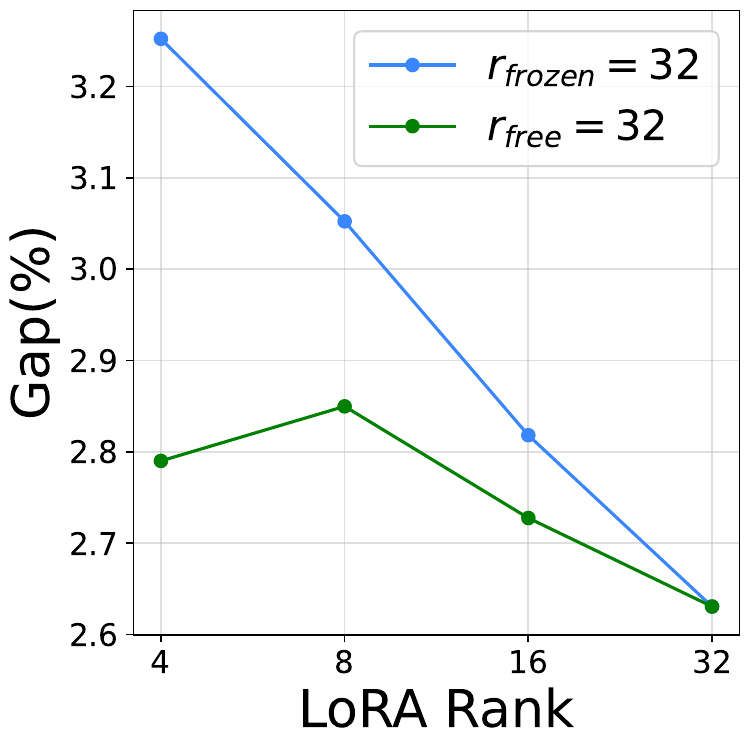}}
\caption{\ref{fig:curve-1}~\ref{fig:curve-2} compare Gated-LoRA against standard LoRA with varying $\beta\in(0,1]$. \ref{fig:curve-3}~\ref{fig:curve-4} investigate the effect of LoRA ranks on the model performance. \vspace{-10pt}}
\label{fig:curves}
\end{figure*}

\subsection{Empirical Results}

We conduct comprehensive benchmarking across all 16 VRP variants, with complete results presented in Table~\ref{tab:main}. MoSES(RF) consistently outperforms over its baseline RF-TE across all 16 VRP variants, achieving both lower solution costs and reduced optimality gaps for problem scales of $N=50$ and $N=100$. In terms of average optimality gap over 16 VRP variants, RF-TE achieves 2.063\% and 3.125\% for $N=50$ and $N=100$, respectively, while MoSES(RF) demonstrates superior performance with gaps of 1.535\% and 2.782\%, indicating relative improvements of 25.6\% and 11.0\%. MoSES(CaDA) demonstrates consistent improvements over its baseline CaDA across all VRP variants at $N=50$. For $N=100$, it shows superiority on 13 out of 16 tasks while maintaining comparable performance on CVRP, OVRPB, and OVRPBL tasks, with marginal decreases of $\leq0.002$ in costs and $\leq0.027\%$ in optimality gaps. MoSES(CaDA) reduces average optimality gaps over 16 VRP tasks from 1.715\% to 1.515\% (11.7\% relative improvement) at $N=50$ and from 2.766\% to 2.631\% (4.9\% relative improvement) at $N=100$, compared to its baseline CaDA. From the perspective of average optimality gap, MoSES(CaDA) is more preferred than MoSES(RF). In terms of the average total time overhead over 16 tasks, MoSES(RF) requires 5.8s and 20.8s for $N=50$ and $N=100$, respectively, compared to 2.0s and 8.4s consumed by RF-TE. MoSES(CaDA) requires 7.4s compared to CaDA’s 2.3s for $N=50$, and 24.8s compared to CaDA’s 10.5s for $N=100$. Since this time overhead represents the total time for 1K instances averaged over 16 tasks, the amortized time per instance remains at the microsecond level for both MoSES(RF) and MoSES(CaDA). Thus, the slightly more computational time of our method, resulting from the layer-wise token-wise routing mechanisms and dynamic aggregation operations, is acceptable given the favorable empirical improvements over prior methods, and may not be a major concern in practical applications.

\subsection{Hyperparameter Studies}

\setlength\intextsep{0pt}
\begin{wrapfigure}{r}{0pt}
\includegraphics[width=0.3\textwidth]{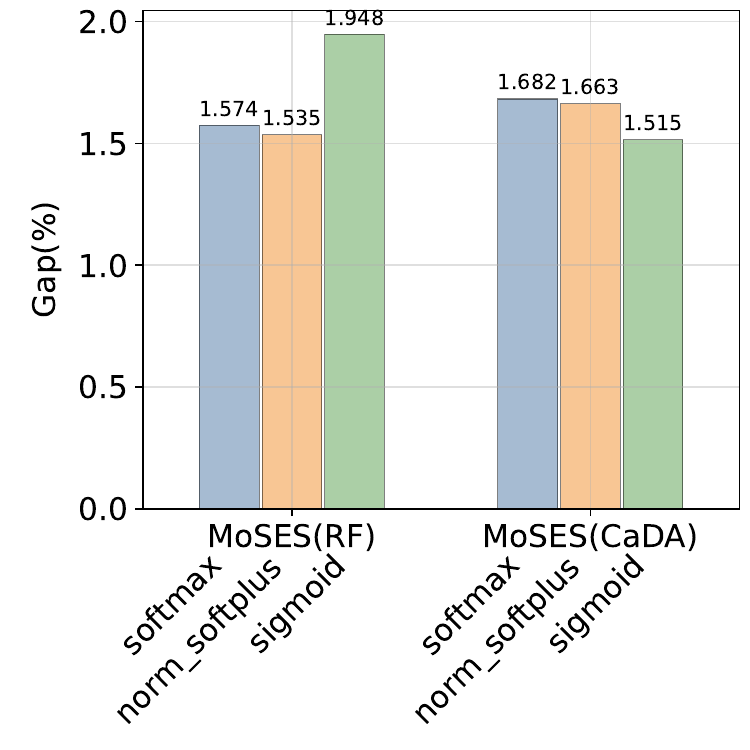}
\caption{Activations.}
\label{fig:act_bar}
\end{wrapfigure}

We also evaluate the impact of LoRA ranks on model performance, using MoSES(CaDA) as a case study across both $N=50$ and $N=100$ problem scales. We allow the ranks for frozen modules $\{ B_{i}A_{i} \}_{i=1}^{4}$ and trainable module $\hat{B}\hat{A}$ to differ, denoted as $r_{\mathrm{frozen}}$ and $r_{\mathrm{free}}$ respectively. Figures~\ref{fig:curve-3}~\ref{fig:curve-4} present the impact of LoRA ranks on the average optimality gap over 16 VRPs, where blue curves fix $r_{\mathrm{frozen}}=32$ while varying $r_{\mathrm{free}}$ from 4 to 32, and green curves fix $r_{\mathrm{free}}=32$ while varying $r_{\mathrm{frozen}}$ from 4 to 32. Figure~\ref{fig:curve-3} reveals that reducing the LoRA rank of trainable module $\hat{B}\hat{A}$ with fixed $r_{\mathrm{frozen}}=32$ incurs smaller performance degradation than reducing the LoRA rank of frozen modules $\{ B_{i}A_{i} \}_{i=1}^{4}$ with fixed $r_{\mathrm{free}}=32$. It suggests that the frozen LoRA experts contribute more significantly to MoSES(CaDA) than the trainable LoRA expert at $N=50$. Figure~\ref{fig:curve-4} demonstrates an inverse relationship at the scale of $N=100$ that the trainable LoRA expert contributes more significantly to MoSES(CaDA) than the frozen LoRA experts.

Before delving into the in-depth analysis of Figure~\ref{fig:curve-3}~\ref{fig:curve-4}, we first clarify the respective roles of the gating function and the trainable LoRA expert in MoSES. The gating function is primarily designed to extract individual insights from basis VRP solvers by identifying which solvers offer the most relevant experience for a given problem instance. In contrast, the trainable LoRA expert is intended to capture a holistic understanding of the given VRP variant, which naturally includes correlations among the basis VRP variants. From Figure~\ref{fig:curve-3}, we observe that for smaller problem instances ($N=50$), reducing the LoRA rank of the frozen LoRA experts leads to a more significant performance drop than reducing the LoRA rank of the trainable LoRA expert. This suggests that for simpler problems, the basis solvers already contain sufficient useful insights, and the unified solver relies less on the holistic understanding provided by the trainable LoRA expert. Conversely, Figure~\ref{fig:curve-4} shows that for larger problem instances ($N=100$), reducing the LoRA rank of the trainable LoRA expert results in a greater performance drop than reducing that of the frozen experts. This indicates that for more complex problems, the unified solver requires a deeper and more holistic understanding of the VRP variant, which the trainable LoRA expert is better suited to provide.

\setlength\intextsep{0pt}
\begin{wrapfigure}{r}{0pt}
\includegraphics[width=0.3\textwidth]{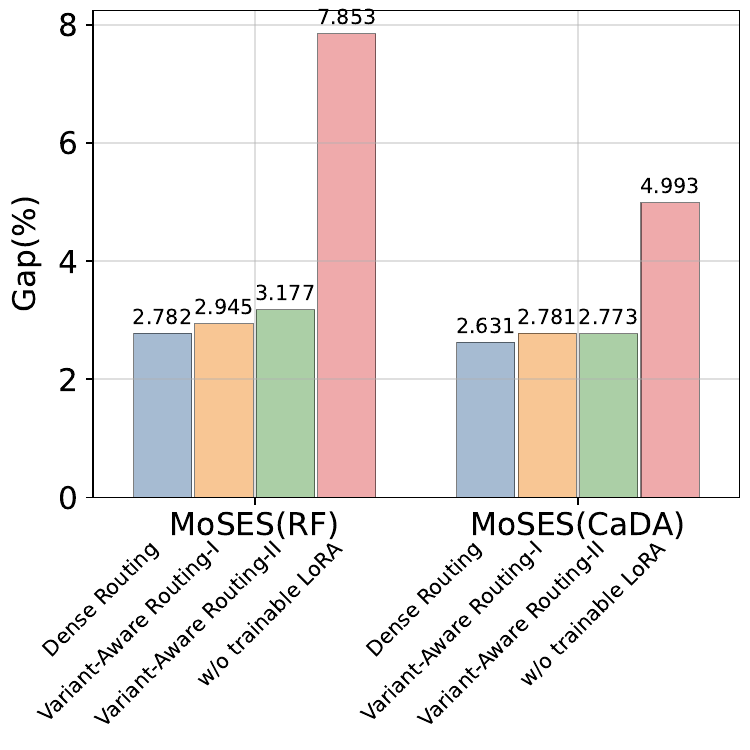}
\caption{Routing Methods.}
\label{fig:route_bar}
\end{wrapfigure}

\subsection{Ablation Studies}
We evaluate our proposed Gated-LoRA against standard LoRA that varies $\beta$ from 0.1 to 1.0, both used in the phase of fine-tuning specialized experts. Figures~\ref{fig:curve-1}~\ref{fig:curve-2} demonstrate Gated-LoRA's consistent superiority by showing lower average optimality gaps over four basis variants (OVRP, VRPB, VRPL, VRPTW) at $N=50$ for both RF-based and CaDA-based backbones. Figure~\ref{fig:act_bar} uncovers that MoSES(RF) prefers $\mathrm{norm\_softplus}(\cdot)$ as the activation function in the gating mechanism, while MoSES(CaDA) prefers $\mathrm{sigmoid}(\cdot)$. CaDA-based basis solvers may exhibit stronger generalization capability than RF-based ones, owing to the two parallel Transformer blocks in each layer. Consequently, MoSES(CaDA), which uses $\mathrm{sigmoid}(\cdot)$ as its activation function, tends to prioritize relevant solvers while assigning moderately higher scores to less relevant ones. This behavior suggests that MoSES(CaDA) recognizes that even task-irrelevant solvers can contribute useful insights for solving the current VRP variant. Figure~\ref{fig:route_bar} shows that dense routing method achieves best in both MoSES(RF) and MoSES(CaDA), and significant performance degradation occurs when the trainable LoRA module $\hat{B}\hat{A}$ is ablated. The dense routing strategy in the gating mechanism does not hinder interpretability, as evidenced by Figure~\ref{appendix:fig:gates_visual} of the Appendix~\ref{appendix:experiment}. Figure~\ref{appendix:fig:gates_visual} presents a behavioral analysis of the gating functions used in our models across 16 VRP variants. In this figure, we plot the scores assigned by the gating function to each basis solver. It is clear that for a given VRP variant, the gating mechanism tends to assign higher scores to basis solvers associated with the corresponding underlying basis VRPs. The reason the gating function does not completely zero out the remaining tasks is that many VRP variants are inherently correlated and solving instances of one task may benefit from insights learned from others. Since our method reuses only the basis VRP solvers, the number of solvers increases linearly, which is more manageable compared to the exponential growth in neural solver.

\section{Conclusions and Limitations}

In this paper, our objective is to design a unified neural solver which is capable of handling multiple VRP variants simultaneously. To achieve this, we propose the State-Decomposable MDP (SDMDP) to reformulate multi-task VRPs, grounded in observation that each VRP variant derives from a shared set of basis VRP variants. Then, the LS-SDMDP extension is developed to reuse basis neural solvers, each specialized for a basis VRP, in the latent space. We finally implement mixture-of-LoRA-experts as the unified solver. While our method demonstrates empirical superiority over prior approaches, it incurs mild computational overhead compared to other neural solvers due to its finer-grained layer-wise and token-wise aggregation and adaptive routing mechanisms. This limitation suggests designing more efficient aggregation techniques which preserve decent performance as future research directions. In addition, extending this framework to broader and general decision-making settings is also appealing.

\begin{ack}
This research is supported by a generous research grant from Xiaoi Robot Technology Limited, and the Singapore Ministry of Education (MOE) Academic Research Fund (AcRF) Tier 1 grant. We appreciate the anonymous reviewers, (S)ACs, and PCs of NeurIPS 2025 for their insightful comments to further improve our paper and their service to the community.

\end{ack}

\bibliographystyle{plain}
\bibliography{bibfile}

\begin{thebibliography}{10}

\bibitem{ainsworth2023git}
Samuel Ainsworth, Jonathan Hayase, and Siddhartha Srinivasa.
\newblock Git re-basin: Merging models modulo permutation symmetries.
\newblock In {\em The Eleventh International Conference on Learning Representations}, 2023.

\bibitem{bdeir2022attention}
Ahmad Bdeir, Jonas~K Falkner, and Lars Schmidt-Thieme.
\newblock Attention, filling in the gaps for generalization in routing problems.
\newblock In {\em Joint European Conference on Machine Learning and Knowledge Discovery in Databases}, pages 505--520. Springer, 2022.

\bibitem{berto2024routefinder}
Federico Berto, Chuanbo Hua, Nayeli~Gast Zepeda, Andr{\'e} Hottung, Niels Wouda, Leon Lan, Kevin Tierney, and Jinkyoo Park.
\newblock Routefinder: Towards foundation models for vehicle routing problems.
\newblock In {\em ICML 2024 Workshop on Foundation Models in the Wild}, 2024.

\bibitem{bi2022learning}
Jieyi Bi, Yining Ma, Jiahai Wang, Zhiguang Cao, Jinbiao Chen, Yuan Sun, and Yeow~Meng Chee.
\newblock Learning generalizable models for vehicle routing problems via knowledge distillation.
\newblock {\em Advances in Neural Information Processing Systems}, 35:31226--31238, 2022.

\bibitem{caccia2023multihead}
Lucas Caccia, Edoardo Ponti, Zhan Su, Matheus Pereira, Nicolas~Le Roux, and Alessandro Sordoni.
\newblock Multi-head adapter routing for cross-task generalization.
\newblock In {\em Thirty-seventh Conference on Neural Information Processing Systems}, 2023.

\bibitem{cattaruzza2017vehicle}
Diego Cattaruzza, Nabil Absi, Dominique Feillet, and Jes{\'u}s Gonz{\'a}lez-Feliu.
\newblock Vehicle routing problems for city logistics.
\newblock {\em EURO Journal on Transportation and Logistics}, 6(1):51--79, 2017.

\bibitem{chen2019learning}
Xinyun Chen and Yuandong Tian.
\newblock Learning to perform local rewriting for combinatorial optimization.
\newblock {\em Advances in neural information processing systems}, 32, 2019.

\bibitem{cheng2023select}
Hanni Cheng, Haosi Zheng, Ya~Cong, Weihao Jiang, and Shiliang Pu.
\newblock Select and optimize: Learning to solve large-scale tsp instances.
\newblock In {\em International Conference on Artificial Intelligence and Statistics}, pages 1219--1231. PMLR, 2023.

\bibitem{chronopoulou2023adaptersoup}
Alexandra Chronopoulou, Matthew~E. Peters, Alexander Fraser, and Jesse Dodge.
\newblock Adaptersoup: Weight averaging to improve generalization of pretrained language models, 2023.

\bibitem{chronopoulou2024language}
Alexandra Chronopoulou, Jonas Pfeiffer, Joshua Maynez, Xinyi Wang, Sebastian Ruder, and Priyanka Agrawal.
\newblock Language and task arithmetic with parameter-efficient layers for zero-shot summarization.
\newblock In Jonne S{\"a}lev{\"a} and Abraham Owodunni, editors, {\em Proceedings of the Fourth Workshop on Multilingual Representation Learning (MRL 2024)}, pages 114--126, Miami, Florida, USA, November 2024. Association for Computational Linguistics.

\bibitem{correa2025tunensearch}
Arthur Corr{\^e}a, Crist{\'o}v{\~a}o Silva, Liming Xu, Alexandra Brintrup, and Samuel Moniz.
\newblock Tunensearch: a hybrid transfer learning and local search approach for solving vehicle routing problems.
\newblock {\em arXiv preprint arXiv:2503.12662}, 2025.

\bibitem{daheim2024model}
Nico Daheim, Thomas M{\"o}llenhoff, Edoardo Ponti, Iryna Gurevych, and Mohammad~Emtiyaz Khan.
\newblock Model merging by uncertainty-based gradient matching.
\newblock In {\em The Twelfth International Conference on Learning Representations}, 2024.

\bibitem{dou2024loramoe}
Shihan Dou, Enyu Zhou, Yan Liu, Songyang Gao, Wei Shen, Limao Xiong, Yuhao Zhou, Xiao Wang, Zhiheng Xi, Xiaoran Fan, Shiliang Pu, Jiang Zhu, Rui Zheng, Tao Gui, Qi~Zhang, and Xuanjing Huang.
\newblock {L}o{RAM}o{E}: Alleviating world knowledge forgetting in large language models via {M}o{E}-style plugin.
\newblock In Lun-Wei Ku, Andre Martins, and Vivek Srikumar, editors, {\em Proceedings of the 62nd Annual Meeting of the Association for Computational Linguistics (Volume 1: Long Papers)}, pages 1932--1945, Bangkok, Thailand, August 2024. Association for Computational Linguistics.

\bibitem{drakulic2024goal}
Darko Drakulic, Sofia Michel, and Jean-Marc Andreoli.
\newblock Goal: A generalist combinatorial optimization agent learner.
\newblock In {\em The Thirteenth International Conference on Learning Representations}, 2024.

\bibitem{drakulic2024bq}
Darko Drakulic, Sofia Michel, Florian Mai, Arnaud Sors, and Jean-Marc Andreoli.
\newblock Bq-nco: Bisimulation quotienting for efficient neural combinatorial optimization.
\newblock {\em Advances in Neural Information Processing Systems}, 36, 2024.

\bibitem{fu2021generalize}
Zhang-Hua Fu, Kai-Bin Qiu, and Hongyuan Zha.
\newblock Generalize a small pre-trained model to arbitrarily large tsp instances.
\newblock In {\em Proceedings of the AAAI conference on artificial intelligence}, volume~35, pages 7474--7482, 2021.

\bibitem{gao2023towards}
Chengrui Gao, Haopu Shang, Ke~Xue, Dong Li, and Chao Qian.
\newblock Towards generalizable neural solvers for vehicle routing problems via ensemble with transferrable local policy.
\newblock In {\em The 33rd International Joint Conference on Artificial Intelligence (IJCAI-24)}, 2024.

\bibitem{gao2024higherlayersneedlora}
Chongyang Gao, Kezhen Chen, Jinmeng Rao, Baochen Sun, Ruibo Liu, Daiyi Peng, Yawen Zhang, Xiaoyuan Guo, Jie Yang, and VS~Subrahmanian.
\newblock Higher layers need more lora experts, 2024.

\bibitem{garaix2010vehicle}
Thierry Garaix, Christian Artigues, Dominique Feillet, and Didier Josselin.
\newblock Vehicle routing problems with alternative paths: An application to on-demand transportation.
\newblock {\em European Journal of Operational Research}, 204(1):62--75, 2010.

\bibitem{goh2025shield}
Yong~Liang Goh, Yining Ma, Jianan Zhou, Zhiguang Cao, Mohammed~Haroon Dupty, and Wee~Sun Lee.
\newblock {SHIELD}: Multi-task multi-distribution vehicle routing solver with sparsity and hierarchy.
\newblock In {\em Forty-second International Conference on Machine Learning}, 2025.

\bibitem{gou2023mixture}
Yunhao Gou, Zhili Liu, Kai Chen, Lanqing Hong, Hang Xu, Aoxue Li, Dit-Yan Yeung, James~T Kwok, and Yu~Zhang.
\newblock Mixture of cluster-conditional lora experts for vision-language instruction tuning.
\newblock {\em arXiv preprint arXiv:2312.12379}, 2023.

\bibitem{grinsztajn2023winner}
Nathan Grinsztajn, Daniel Furelos-Blanco, Shikha Surana, Cl{\'e}ment Bonnet, and Tom Barrett.
\newblock Winner takes it all: Training performant rl populations for combinatorial optimization.
\newblock {\em Advances in Neural Information Processing Systems}, 36:48485--48509, 2023.

\bibitem{helsgaun2017extension}
Keld Helsgaun.
\newblock An extension of the lin-kernighan-helsgaun tsp solver for constrained traveling salesman and vehicle routing problems.
\newblock {\em Roskilde: Roskilde University}, 12:966--980, 2017.

\bibitem{hottung2020neural}
Andr{\'e} Hottung and Kevin Tierney.
\newblock Neural large neighborhood search for the capacitated vehicle routing problem.
\newblock In {\em ECAI 2020}, pages 443--450. IOS Press, 2020.

\bibitem{hou2023generalize}
Qingchun Hou, Jingwei Yang, Yiqiang Su, Xiaoqing Wang, and Yuming Deng.
\newblock Generalize learned heuristics to solve large-scale vehicle routing problems in real-time.
\newblock In {\em The Eleventh International Conference on Learning Representations}, 2023.

\bibitem{hu2022lora}
Edward~J Hu, yelong shen, Phillip Wallis, Zeyuan Allen-Zhu, Yuanzhi Li, Shean Wang, Lu~Wang, and Weizhu Chen.
\newblock Lo{RA}: Low-rank adaptation of large language models.
\newblock In {\em International Conference on Learning Representations}, 2022.

\bibitem{huang2024lorahub}
Chengsong Huang, Qian Liu, Bill~Yuchen Lin, Tianyu Pang, Chao Du, and Min Lin.
\newblock Lorahub: Efficient cross-task generalization via dynamic lo{RA} composition.
\newblock In {\em First Conference on Language Modeling}, 2024.

\bibitem{ilharco2023editing}
Gabriel Ilharco, Marco~Tulio Ribeiro, Mitchell Wortsman, Ludwig Schmidt, Hannaneh Hajishirzi, and Ali Farhadi.
\newblock Editing models with task arithmetic.
\newblock In {\em The Eleventh International Conference on Learning Representations}, 2023.

\bibitem{jiang2024unco}
Xia Jiang, Yaoxin Wu, Yuan Wang, and Yingqian Zhang.
\newblock Unco: Towards unifying neural combinatorial optimization through large language model.
\newblock {\em arXiv preprint arXiv:2408.12214}, 2024.

\bibitem{jiang2024ensemble}
Yuan Jiang, Zhiguang Cao, Yaoxin Wu, Wen Song, and Jie Zhang.
\newblock Ensemble-based deep reinforcement learning for vehicle routing problems under distribution shift.
\newblock {\em Advances in Neural Information Processing Systems}, 36, 2024.

\bibitem{jin2023dataless}
Xisen Jin, Xiang Ren, Daniel Preotiuc-Pietro, and Pengxiang Cheng.
\newblock Dataless knowledge fusion by merging weights of language models.
\newblock In {\em The Eleventh International Conference on Learning Representations}, 2023.

\bibitem{kim2021learning}
Minsu Kim, Jinkyoo Park, et~al.
\newblock Learning collaborative policies to solve np-hard routing problems.
\newblock {\em Advances in Neural Information Processing Systems}, 34:10418--10430, 2021.

\bibitem{kim2022sym}
Minsu Kim, Junyoung Park, and Jinkyoo Park.
\newblock Sym-nco: Leveraging symmetricity for neural combinatorial optimization.
\newblock {\em Advances in Neural Information Processing Systems}, 35:1936--1949, 2022.

\bibitem{kool2018attention}
Wouter Kool, Herke van Hoof, and Max Welling.
\newblock Attention, learn to solve routing problems!
\newblock In {\em International Conference on Learning Representations}, 2019.

\bibitem{kwon2020pomo}
Yeong-Dae Kwon, Jinho Choo, Byoungjip Kim, Iljoo Yoon, Youngjune Gwon, and Seungjai Min.
\newblock Pomo: Policy optimization with multiple optima for reinforcement learning.
\newblock {\em Advances in Neural Information Processing Systems}, 33:21188--21198, 2020.

\bibitem{laporte1983branch}
Gilbert Laporte and Yves Nobert.
\newblock A branch and bound algorithm for the capacitated vehicle routing problem.
\newblock {\em Operations-Research-Spektrum}, 5:77--85, 1983.

\bibitem{lei2025boosting}
Haoyu Lei, Kaiwen Zhou, Yinchuan Li, Zhitang Chen, and Farzan Farnia.
\newblock Boosting generalization in diffusion-based neural combinatorial solver via energy-guided sampling.
\newblock {\em arXiv preprint arXiv:2502.12188}, 2025.

\bibitem{li2024cada}
Han Li, Fei Liu, Zhi Zheng, Yu~Zhang, and Zhenkun Wang.
\newblock Cada: Cross-problem routing solver with constraint-aware dual-attention.
\newblock {\em arXiv preprint arXiv:2412.00346}, 2024.

\bibitem{li2021learning}
Sirui Li, Zhongxia Yan, and Cathy Wu.
\newblock Learning to delegate for large-scale vehicle routing.
\newblock {\em Advances in Neural Information Processing Systems}, 34:26198--26211, 2021.

\bibitem{li2025toward}
Yang Li, Jiaxi Liu, Qitian Wu, Xiaohan Qin, and Junchi Yan.
\newblock Toward learning generalized cross-problem solving strategies for combinatorial optimization, 2025.

\bibitem{lin2024cross}
Zhuoyi Lin, Yaoxin Wu, Bangjian Zhou, Zhiguang Cao, Wen Song, Yingqian Zhang, and Senthilnath Jayavelu.
\newblock Cross-problem learning for solving vehicle routing problems.
\newblock In {\em Proceedings of the Thirty-Third International Joint Conference on Artificial Intelligence}, pages 6958--6966, 2024.

\bibitem{liu2024multi}
Fei Liu, Xi~Lin, Zhenkun Wang, Qingfu Zhang, Tong Xialiang, and Mingxuan Yuan.
\newblock Multi-task learning for routing problem with cross-problem zero-shot generalization.
\newblock In {\em Proceedings of the 30th ACM SIGKDD Conference on Knowledge Discovery and Data Mining}, pages 1898--1908, 2024.

\bibitem{liu2025mixed}
Suyu Liu, Zhiguang Cao, Shanshan Feng, and Yew-Soon Ong.
\newblock A mixed-curvature based pre-training paradigm for multi-task vehicle routing solver.
\newblock In {\em Forty-second International Conference on Machine Learning}, 2025.

\bibitem{liu2024adamole}
Zefang Liu and Jiahua Luo.
\newblock Adamo{LE}: Fine-tuning large language models with adaptive mixture of low-rank adaptation experts.
\newblock In {\em First Conference on Language Modeling}, 2024.

\bibitem{lu2019learning}
Hao Lu, Xingwen Zhang, and Shuang Yang.
\newblock A learning-based iterative method for solving vehicle routing problems.
\newblock In {\em International conference on learning representations}, 2019.

\bibitem{luo2024neural}
Fu~Luo, Xi~Lin, Fei Liu, Qingfu Zhang, and Zhenkun Wang.
\newblock Neural combinatorial optimization with heavy decoder: Toward large scale generalization.
\newblock {\em Advances in Neural Information Processing Systems}, 36, 2024.

\bibitem{luo2025boosting}
Fu~Luo, Xi~Lin, Yaoxin Wu, Zhenkun Wang, Tong Xialiang, Mingxuan Yuan, and Qingfu Zhang.
\newblock Boosting neural combinatorial optimization for large-scale vehicle routing problems.
\newblock In {\em The Thirteenth International Conference on Learning Representations}, 2025.

\bibitem{lv2023parameter}
Xingtai Lv, Ning Ding, Yujia Qin, Zhiyuan Liu, and Maosong Sun.
\newblock Parameter-efficient weight ensembling facilitates task-level knowledge transfer.
\newblock In Anna Rogers, Jordan Boyd-Graber, and Naoaki Okazaki, editors, {\em Proceedings of the 61st Annual Meeting of the Association for Computational Linguistics (Volume 2: Short Papers)}, pages 270--282, Toronto, Canada, July 2023. Association for Computational Linguistics.

\bibitem{ma2023neuopt}
Yining Ma, Zhiguang Cao, and Yeow~Meng Chee.
\newblock Learning to search feasible and infeasible regions of routing problems with flexible neural k-opt.
\newblock In {\em Advances in Neural Information Processing Systems}, volume~36, 2023.

\bibitem{ma2021learning}
Yining Ma, Jingwen Li, Zhiguang Cao, Wen Song, Le~Zhang, Zhenghua Chen, and Jing Tang.
\newblock Learning to iteratively solve routing problems with dual-aspect collaborative transformer.
\newblock In {\em Advances in Neural Information Processing Systems}, volume~34, pages 11096--11107, 2021.

\bibitem{manchanda2022generalization}
Sahil Manchanda, Sofia Michel, Darko Drakulic, and Jean-Marc Andreoli.
\newblock On the generalization of neural combinatorial optimization heuristics.
\newblock In {\em Joint European Conference on Machine Learning and Knowledge Discovery in Databases}, pages 426--442. Springer, 2022.

\bibitem{matena2022mergingmodels}
Michael Matena and Colin Raffel.
\newblock Merging models with fisher-weighted averaging.
\newblock In {\em Proceedings of the 36th International Conference on Neural Information Processing Systems}, NIPS '22, Red Hook, NY, USA, 2022. Curran Associates Inc.

\bibitem{muqeeth2024learning}
Mohammed Muqeeth, Haokun Liu, Yufan Liu, and Colin Raffel.
\newblock Learning to route among specialized experts for zero-shot generalization.
\newblock In {\em International Conference on Machine Learning}, pages 36829--36846. PMLR, 2024.

\bibitem{nazari2018reinforcement}
Mohammadreza Nazari, Afshin Oroojlooy, Lawrence Snyder, and Martin Tak{\'a}c.
\newblock Reinforcement learning for solving the vehicle routing problem.
\newblock {\em Advances in neural information processing systems}, 31, 2018.

\bibitem{ostapenko2024towards}
Oleksiy Ostapenko, Zhan Su, Edoardo Ponti, Laurent Charlin, Nicolas Le~Roux, Lucas Caccia, and Alessandro Sordoni.
\newblock Towards modular llms by building and reusing a library of loras.
\newblock In {\em International Conference on Machine Learning}, pages 38885--38904. PMLR, 2024.

\bibitem{pan2025hlgp4cvrp}
Yuxin Pan, Ruohong Liu, Yize Chen, Cao Zhiguang, and Fangzhen Lin.
\newblock Hierarchical learning-based graph partition for large-scale vehicle routing problems.
\newblock In {\em The 24th International Conference on Autonomous Agents and Multi-Agent Systems}, 2025.

\bibitem{perron2023ortools}
Laurent Perron and Vincent Furnon.
\newblock {{OR-Tools}}.
\newblock Google, 2023.

\bibitem{pfeiffer2021adapterfusion}
Jonas Pfeiffer, Aishwarya Kamath, Andreas R{\"u}ckl{\'e}, Kyunghyun Cho, and Iryna Gurevych.
\newblock {A}dapter{F}usion: Non-destructive task composition for transfer learning.
\newblock In Paola Merlo, Jorg Tiedemann, and Reut Tsarfaty, editors, {\em Proceedings of the 16th Conference of the European Chapter of the Association for Computational Linguistics: Main Volume}, pages 487--503, Online, April 2021. Association for Computational Linguistics.

\bibitem{ponti2023combining}
Edoardo~Maria Ponti, Alessandro Sordoni, Yoshua Bengio, and Siva Reddy.
\newblock Combining parameter-efficient modules for task-level generalisation.
\newblock In Andreas Vlachos and Isabelle Augenstein, editors, {\em Proceedings of the 17th Conference of the European Chapter of the Association for Computational Linguistics}, pages 687--702, Dubrovnik, Croatia, May 2023. Association for Computational Linguistics.

\bibitem{qiu2022dimes}
Ruizhong Qiu, Zhiqing Sun, and Yiming Yang.
\newblock {DIMES}: A differentiable meta solver for combinatorial optimization problems.
\newblock In {\em Advances in Neural Information Processing Systems}, volume~35, 2022.

\bibitem{rame2023ratatouille}
Alexandre Ram\'{e}, Kartik Ahuja, Jianyu Zhang, Matthieu Cord, L\'{e}on Bottou, and David Lopez-Paz.
\newblock Model ratatouille: recycling diverse models for out-of-distribution generalization.
\newblock In {\em Proceedings of the 40th International Conference on Machine Learning}, ICML'23. JMLR.org, 2023.

\bibitem{ropke2006backhauls}
Stefan Ropke and David Pisinger.
\newblock A unified heuristic for a large class of vehicle routing problems with backhauls.
\newblock {\em European Journal of Operational Research}, 171(3):750--775, 2006.
\newblock Feature Cluster: Heuristic and Stochastic Methods in Optimization Feature Cluster: New Opportunities for Operations Research.

\bibitem{stoica2024zipit}
George Stoica, Daniel Bolya, Jakob~Brandt Bjorner, Pratik Ramesh, Taylor Hearn, and Judy Hoffman.
\newblock Zipit! merging models from different tasks without training.
\newblock In {\em The Twelfth International Conference on Learning Representations}, 2024.

\bibitem{tam2024merging}
Derek Tam, Mohit Bansal, and Colin Raffel.
\newblock Merging by matching models in task parameter subspaces.
\newblock {\em Transactions on Machine Learning Research}, 2024.

\bibitem{VIDAL2022Hybrid}
Thibaut Vidal.
\newblock Hybrid genetic search for the cvrp: Open-source implementation and swap* neighborhood.
\newblock {\em Computers \& Operations Research}, 140:105643, 2022.

\bibitem{vidal2012hybrid}
Thibaut Vidal, Teodor~Gabriel Crainic, Michel Gendreau, Nadia Lahrichi, and Walter Rei.
\newblock A hybrid genetic algorithm for multidepot and periodic vehicle routing problems.
\newblock {\em Operations Research}, 60(3):611--624, 2012.

\bibitem{vinyals2015pointer}
Oriol Vinyals, Meire Fortunato, and Navdeep Jaitly.
\newblock Pointer networks.
\newblock {\em Advances in neural information processing systems}, 28, 2015.

\bibitem{wangefficient}
Chenguang Wang, Zhang-Hua Fu, Pinyan Lu, and Tianshu Yu.
\newblock Efficient training of multi-task neural solver for combinatorial optimization.
\newblock {\em Transactions on Machine Learning Research}, 2025.

\bibitem{williams1992simple}
Ronald~J Williams.
\newblock Simple statistical gradient-following algorithms for connectionist reinforcement learning.
\newblock {\em Machine learning}, 8:229--256, 1992.

\bibitem{wortsman2022modelsoups}
Mitchell Wortsman, Gabriel Ilharco, Samir~Ya Gadre, Rebecca Roelofs, Raphael Gontijo-Lopes, Ari~S Morcos, Hongseok Namkoong, Ali Farhadi, Yair Carmon, Simon Kornblith, and Ludwig Schmidt.
\newblock Model soups: averaging weights of multiple fine-tuned models improves accuracy without increasing inference time.
\newblock In Kamalika Chaudhuri, Stefanie Jegelka, Le~Song, Csaba Szepesvari, Gang Niu, and Sivan Sabato, editors, {\em Proceedings of the 39th International Conference on Machine Learning}, volume 162 of {\em Proceedings of Machine Learning Research}, pages 23965--23998. PMLR, 17--23 Jul 2022.

\bibitem{Wouda2024PyVRP}
Niels~A. Wouda, Leon Lan, and Wouter Kool.
\newblock {PyVRP}: a high-performance {VRP} solver package.
\newblock {\em INFORMS Journal on Computing}, 36(4):943--955, 2024.

\bibitem{wu2023pituning}
Chengyue Wu, Teng Wang, Yixiao Ge, Zeyu Lu, Ruisong Zhou, Ying Shan, and Ping Luo.
\newblock $\pi$-tuning: Transferring multimodal foundation models with optimal multi-task interpolation.
\newblock In Andreas Krause, Emma Brunskill, Kyunghyun Cho, Barbara Engelhardt, Sivan Sabato, and Jonathan Scarlett, editors, {\em Proceedings of the 40th International Conference on Machine Learning}, volume 202 of {\em Proceedings of Machine Learning Research}, pages 37713--37727. PMLR, 23--29 Jul 2023.

\bibitem{wu2024mixture}
Xun Wu, Shaohan Huang, and Furu Wei.
\newblock Mixture of lo{RA} experts.
\newblock In {\em The Twelfth International Conference on Learning Representations}, 2024.

\bibitem{xin2021neurolkh}
Liang Xin, Wen Song, Zhiguang Cao, and Jie Zhang.
\newblock Neurolkh: Combining deep learning model with lin-kernighan-helsgaun heuristic for solving the traveling salesman problem.
\newblock In {\em Advances in Neural Information Processing Systems}, volume~34, 2021.

\bibitem{yadav2023tiesmerging}
Prateek Yadav, Derek Tam, Leshem Choshen, Colin Raffel, and Mohit Bansal.
\newblock {TIES}-merging: Resolving interference when merging models.
\newblock In {\em Thirty-seventh Conference on Neural Information Processing Systems}, 2023.

\bibitem{yang2024AdaMerging}
Enneng Yang, Zhenyi Wang, Li~Shen, Shiwei Liu, Guibing Guo, Xingwei Wang, and Dacheng Tao.
\newblock Adamerging: Adaptive model merging for multi-task learning.
\newblock {\em The Twelfth International Conference on Learning Representations}, 2024.

\bibitem{ye2024glop}
Haoran Ye, Jiarui Wang, Helan Liang, Zhiguang Cao, Yong Li, and Fanzhang Li.
\newblock Glop: Learning global partition and local construction for solving large-scale routing problems in real-time.
\newblock In {\em Proceedings of the AAAI Conference on Artificial Intelligence}, volume~38, pages 20284--20292, 2024.

\bibitem{ye2022eliciting}
Qinyuan Ye, Juan Zha, and Xiang Ren.
\newblock Eliciting and understanding cross-task skills with task-level mixture-of-experts.
\newblock In Yoav Goldberg, Zornitsa Kozareva, and Yue Zhang, editors, {\em Findings of the Association for Computational Linguistics: EMNLP 2022}, pages 2567--2592, Abu Dhabi, United Arab Emirates, December 2022. Association for Computational Linguistics.

\bibitem{zadouri2024pushing}
Ted Zadouri, Ahmet {\"U}st{\"u}n, Arash Ahmadian, Beyza Ermis, Acyr Locatelli, and Sara Hooker.
\newblock Pushing mixture of experts to the limit: Extremely parameter efficient moe for instruction tuning.
\newblock In {\em The Twelfth International Conference on Learning Representations}, 2024.

\bibitem{zhang2023review}
Cong Zhang, Yaoxin Wu, Yining Ma, Wen Song, Zhang Le, Zhiguang Cao, and Jie Zhang.
\newblock A review on learning to solve combinatorial optimisation problems in manufacturing.
\newblock {\em IET Collaborative Intelligent Manufacturing}, 5(1):e12072, 2023.

\bibitem{zhang2023composing}
Jinghan Zhang, Shiqi Chen, Junteng Liu, and Junxian He.
\newblock Composing parameter-efficient modules with arithmetic operations.
\newblock In {\em Advances in Neural Information Processing Systems}, 2023.

\bibitem{zheng2024udc}
Zhi Zheng, Changliang Zhou, Tong Xialiang, Mingxuan Yuan, and Zhenkun Wang.
\newblock {UDC}: A unified neural divide-and-conquer framework for large-scale combinatorial optimization problems.
\newblock In {\em The Thirty-eighth Annual Conference on Neural Information Processing Systems}, 2024.

\bibitem{zhou2024mvmoe}
Jianan Zhou, Zhiguang Cao, Yaoxin Wu, Wen Song, Yining Ma, Jie Zhang, and Chi Xu.
\newblock Mvmoe: multi-task vehicle routing solver with mixture-of-experts.
\newblock In {\em Proceedings of the 41st International Conference on Machine Learning}, pages 61804--61824, 2024.

\bibitem{zhou2023towards}
Jianan Zhou, Yaoxin Wu, Wen Song, Zhiguang Cao, and Jie Zhang.
\newblock Towards omni-generalizable neural methods for vehicle routing problems.
\newblock In {\em International Conference on Machine Learning}, pages 42769--42789. PMLR, 2023.

\bibitem{zong2022rbg}
Zefang Zong, Hansen Wang, Jingwei Wang, Meng Zheng, and Yong Li.
\newblock Rbg: Hierarchically solving large-scale routing problems in logistic systems via reinforcement learning.
\newblock In {\em Proceedings of the 28th ACM SIGKDD Conference on Knowledge Discovery and Data Mining}, pages 4648--4658, 2022.

\end{thebibliography}


\newpage
\section*{NeurIPS Paper Checklist}

The checklist is designed to encourage best practices for responsible machine learning research, addressing issues of reproducibility, transparency, research ethics, and societal impact. Do not remove the checklist: {\bf The papers not including the checklist will be desk rejected.} The checklist should follow the references and follow the (optional) supplemental material.  The checklist does NOT count towards the page
limit. 

Please read the checklist guidelines carefully for information on how to answer these questions. For each question in the checklist:
\begin{itemize}
    \item You should answer \answerYes{}, \answerNo{}, or \answerNA{}.
    \item \answerNA{} means either that the question is Not Applicable for that particular paper or the relevant information is Not Available.
    \item Please provide a short (1–2 sentence) justification right after your answer (even for NA). 
\end{itemize}

{\bf The checklist answers are an integral part of your paper submission.} They are visible to the reviewers, area chairs, senior area chairs, and ethics reviewers. You will be asked to also include it (after eventual revisions) with the final version of your paper, and its final version will be published with the paper.

The reviewers of your paper will be asked to use the checklist as one of the factors in their evaluation. While "\answerYes{}" is generally preferable to "\answerNo{}", it is perfectly acceptable to answer "\answerNo{}" provided a proper justification is given (e.g., "error bars are not reported because it would be too computationally expensive" or "we were unable to find the license for the dataset we used"). In general, answering "\answerNo{}" or "\answerNA{}" is not grounds for rejection. While the questions are phrased in a binary way, we acknowledge that the true answer is often more nuanced, so please just use your best judgment and write a justification to elaborate. All supporting evidence can appear either in the main paper or the supplemental material, provided in appendix. If you answer \answerYes{} to a question, in the justification please point to the section(s) where related material for the question can be found.

IMPORTANT, please:
\begin{itemize}
    \item {\bf Delete this instruction block, but keep the section heading ``NeurIPS Paper Checklist"},
    \item  {\bf Keep the checklist subsection headings, questions/answers and guidelines below.}
    \item {\bf Do not modify the questions and only use the provided macros for your answers}.
\end{itemize}


\begin{enumerate}

\item {\bf Claims}
    \item[] Question: Do the main claims made in the abstract and introduction accurately reflect the paper's contributions and scope?
    \item[] Answer: \answerYes{} 
    \item[] Justification: We present our key contributions and scope in the abstract and introduction.

\item {\bf Limitations}
    \item[] Question: Does the paper discuss the limitations of the work performed by the authors?
    \item[] Answer: \answerYes{} 
    \item[] Justification: The limitations of this work are discussed in Section 6.

\item {\bf Theory assumptions and proofs}
    \item[] Question: For each theoretical result, does the paper provide the full set of assumptions and a complete (and correct) proof?
    \item[] Answer: \answerYes{} 
    \item[] Justification: Assumptions are stated in the main text, with full proofs of theorems provided in the Appendix.

    \item {\bf Experimental result reproducibility}
    \item[] Question: Does the paper fully disclose all the information needed to reproduce the main experimental results of the paper to the extent that it affects the main claims and/or conclusions of the paper (regardless of whether the code and data are provided or not)?
    \item[] Answer: \answerYes{} 
    \item[] Justification: 
    We present adequate information in both the main text and Appendix for the reproducibility.

\item {\bf Open access to data and code}
    \item[] Question: Does the paper provide open access to the data and code, with sufficient instructions to faithfully reproduce the main experimental results, as described in supplemental material?
    \item[] Answer: \answerYes{} 
    \item[] Justification: The implementation code is included in the supplementary materials. 

\item {\bf Experimental setting/details}
    \item[] Question: Does the paper specify all the training and test details (e.g., data splits, hyperparameters, how they were chosen, type of optimizer, etc.) necessary to understand the results?
    \item[] Answer: \answerYes{} 
    \item[] Justification: Experimental details are presented in both the main text and Appendix.

\item {\bf Experiment statistical significance}
    \item[] Question: Does the paper report error bars suitably and correctly defined or other appropriate information about the statistical significance of the experiments?
    \item[] Answer: \answerNo{} 
    \item[] Justification: Error bars derived from multiple experimental runs do not represent the primary evaluation metric in neural solver research.

\item {\bf Experiments compute resources}
    \item[] Question: For each experiment, does the paper provide sufficient information on the computer resources (type of compute workers, memory, time of execution) needed to reproduce the experiments?
    \item[] Answer: \answerYes{} 
    \item[] Justification: We report computational resources in the main text.
    
\item {\bf Code of ethics}
    \item[] Question: Does the research conducted in the paper conform, in every respect, with the NeurIPS Code of Ethics \url{https://neurips.cc/public/EthicsGuidelines}?
    \item[] Answer: \answerYes{} 
    \item[] Justification: Our research conforms with the NeurIPS Code of Ethics.

\item {\bf Broader impacts}
    \item[] Question: Does the paper discuss both potential positive societal impacts and negative societal impacts of the work performed?
    \item[] Answer: \answerYes{} 
    \item[] Justification: By focusing on practical VRP applications, our work delivers more positive real-world impact.
    
\item {\bf Safeguards}
    \item[] Question: Does the paper describe safeguards that have been put in place for responsible release of data or models that have a high risk for misuse (e.g., pretrained language models, image generators, or scraped datasets)?
    \item[] Answer: \answerNA{} 
    \item[] Justification: Our paper poses no such risks.

\item {\bf Licenses for existing assets}
    \item[] Question: Are the creators or original owners of assets (e.g., code, data, models), used in the paper, properly credited and are the license and terms of use explicitly mentioned and properly respected?
    \item[] Answer: \answerYes{} 
    \item[] Justification: We properly cite the original paper that produced the code package or dataset.

\item {\bf New assets}
    \item[] Question: Are new assets introduced in the paper well documented and is the documentation provided alongside the assets?
    \item[] Answer: \answerYes{} 
    \item[] Justification: We include the implementation code in the supplementary materials. 

\item {\bf Crowdsourcing and research with human subjects}
    \item[] Question: For crowdsourcing experiments and research with human subjects, does the paper include the full text of instructions given to participants and screenshots, if applicable, as well as details about compensation (if any)? 
    \item[] Answer: \answerNA{} 
    \item[] Justification: Our paper does not involve crowdsourcing nor research with human subjects.

\item {\bf Institutional review board (IRB) approvals or equivalent for research with human subjects}
    \item[] Question: Does the paper describe potential risks incurred by study participants, whether such risks were disclosed to the subjects, and whether Institutional Review Board (IRB) approvals (or an equivalent approval/review based on the requirements of your country or institution) were obtained?
    \item[] Answer: \answerNA{} 
    \item[] Justification: Our paper does not involve crowdsourcing nor research with human subjects.

\item {\bf Declaration of LLM usage}
    \item[] Question: Does the paper describe the usage of LLMs if it is an important, original, or non-standard component of the core methods in this research? Note that if the LLM is used only for writing, editing, or formatting purposes and does not impact the core methodology, scientific rigorousness, or originality of the research, declaration is not required.
    \item[] Answer: \answerNA{} 
    \item[] Justification: The core method development in our research does not involve LLMs as any important, original, or non-standard components.

\end{enumerate}


\newpage
\appendix
\section{Environment Details}

\subsection{VRP Variants}
\label{appendix:vrp_variants}

We adopt the same environment configurations for Vehicle Routing Problems (VRPs) as those described in RouteFinder~\cite{berto2024routefinder}. In this setting, all VRP variants extend the Capacitated VRP (CVRP) by incorporating arbitrary combinations of the following four constraints: 1) \emph{Open Route (O):} A binary variable $o$ indicates whether the vehicle needs to return to depot $(o=0)$ or not $(o=1)$; 2) \emph{Backhaul (B):} Customers $\{v_{j}\}_{j=1}^{N}$ are categorized into linehauls with $q_{j}^{\mathrm{LH}} > 0$ or backhauls with $q_{j}^{\mathrm{BH}} < 0$, where linehauls require deliveries and backhauls require pickups. These are mutually exclusive: $q_{j}^{\mathrm{LH}}q_{j}^{\mathrm{BH}} = 0$ for all $j\in\{ 1, \ldots, N \}$. VRPs with backhaul allow traversing both types in a mixed manner, but linehauls must precede backhauls in each subtour~\cite{ropke2006backhauls}. In VRPs without backhaul, only linehaul customers are present; 3) \emph{Duration Limit (L):} The cost of each subtour within a solution cannot exceed a threshold $l^{\mathrm{dur}}$; 4) \emph{Time Window (TW):} Each node $v_{i}$ ($i \in \{ 0,  \ldots, N \}$) is associated with a time window $[w_{i}^{\mathrm{beg}}, w_{i}^{\mathrm{end}}]$ and a service duration $w_{i}^{\mathrm{dur}}$, requiring that service must begin within the specified window. The vehicle arriving prior to $w_{i}^{\mathrm{beg}}$ must wait until the window opens, after which the node is occupied for exactly $w_{i}^{\mathrm{dur}}$ time units. Furthermore, all vehicles are constrained to return to the depot no later than $w_{0}^{\mathrm{end}}$. CVRP serves as a basis VRP variant, from which the remaining four basis VRP variants emerge by individually incorporating one additional constraint: CVRP with open routes (OVRP), CVRP with Backhauls (VRPB), CVRP with Duration Limits (VRPL), and CVRP with Time Windows (VRPTW). The complete combinatorial space of these four constraints results in 16 VRP variants ($2^{4}$ possible combinations): CVRP, OVRP, VRPB, VRPL, VRPTW, OVRPTW, VRPBL, VRPBLTW, VRPBTW, VRPLTW, OVRPB, OVRPBL, OVRPBLTW, OVRPBTW, OVRPL, OVRPLTW.

\subsection{Problem Instance Generation}
\label{appendix:problem_instance}

In this section, we present the procedures for creating the problem instances of various VRP variants for both the training and testing phases. As each VRP variant stems from a common set of attributes, we proceed to outline the generation process for each attribute directly.

\textbf{Locations.} In each problem instance, $N+1$ coordinates are uniformly sampled from a unit square, denoted as $\{ (x_{0}, y_{0}), (x_{1}, y_{1}), \ldots,  (x_{N}, y_{N})\}$, where $x_{i}, y_{i} \sim U(0, 1)$ for $i = 0, \ldots, N$. $(x_{0}, y_{0})$ represent the depot coordinates, while $\{ (x_{i}, y_{i})\}_{i=1}^{N}$ denote the coordinates of the customers. 

\textbf{Vehicle Capacity.} The vehicle capacity $Q$ remains constant for all vehicles within each problem instance and is calculated as follows:

\begin{equation}
Q = 
\begin{cases}
30 + \lfloor \frac{1000}{5} + \frac{N-1000}{33.3} \rfloor
& \text{if } N>1000 \\
30 + \lfloor \frac{N}{5} \rfloor  
& \text{if } 20 < N \leq 1000 \\
30 & \text{otherwise}
\end{cases}
\end{equation}

\textbf{Linehaul and Backhaul Demands.} For the depot node $v_0$, both linehaul and backhaul demands are set to zero. For each customer $v_{j}$ $(j \geq 1)$, the linehaul demand $q_{j}^{\mathrm{LH}}$ is drawn uniformly at random from the integer set $\{ 1, 2, \ldots, 9 \}$, while the backhaul demand $q_{j}^{\mathrm{BH}}$ is sampled uniformly from $\{ -1, -2, \ldots, -9 \}$. A binary decision variable $\hat{q}_{j} \in \{0, 1\}$ is then introduced, with probabilities $\mathbb{P}(\hat{q}_{j} = 0) = 0.8$ and $\mathbb{P}(\hat{q}_{j} = 1) = 0.2$. If $\hat{q}_{j} = 0$, the customer is designated as a linehaul, retaining the sampled value of $q_{j}^{\mathrm{LH}}$ while setting $q_{j}^{\mathrm{BH}} = 0$. Conversely, if $\hat{q}_{j} = 1$, the customer is treated as a backhaul, preserving $q_{j}^{\mathrm{BH}}$ and setting $q_{j}^{\mathrm{LH}} = 0$. It is important to note that both linehaul and backhaul demands are present only in VRP variants that incorporate backhauls. In all other VRP variants, only linehaul demands are considered.

\textbf{Open Routes.} A binary variable $o \in \{0, 1\}$ is introduced to indicate whether the vehicle is required to return to the depot. The proportion of VRP instances with open routes is governed by the probability $\mathbb{P}(o = 1)$, where $o = 1$ signifies that the open route constraint is active.

\textbf{Time Windows.} For the depot node $v_0$, the time window is defined as $[w_{0}^{\mathrm{beg}}, w_{0}^{\mathrm{end}}]$, where $w_{0}^{\mathrm{beg}} = 0$ and $w_{0}^{\mathrm{end}}$ represents the system's end time. The service duration at the depot, denoted by $w_{0}^{\mathrm{dur}}$, is set to zero. For each customer $v_{i}$ ($i \in \{ 1, \ldots, N \}$), a time window $[w_{i}^{\mathrm{beg}}, w_{i}^{\mathrm{end}}]$ and a service duration $w_{i}^{\mathrm{dur}}$ require to be generated. The service duration $w_{i}^{\mathrm{dur}}$ is sampled uniformly from the interval $[I_1, I_2]$. The length of the time window, defined as $w_{i}^{\mathrm{len}} = w_{i}^{\mathrm{end}} - w_{i}^{\mathrm{beg}}$, is drawn uniformly from the interval $[I_2, I_3]$. Let $\mathrm{Dist}(v_0, v_{i})$ denote the distance between the depot $v_0$ and customer $v_{i}$. The upper bound $w_{i}^{\mathrm{UB}}$ for the start time of the time window is computed as:
\begin{equation}
    w_{i}^{\mathrm{UB}} = \frac{w_{0}^{\mathrm{end}} - w_{i}^{\mathrm{dur}} - w_{i}^{\mathrm{len}}}{\mathrm{Dist}(v_{0}, v_{i})} - 1.
\end{equation}
The start time of the time window is then determined by:
\begin{equation}
    w_{i}^{\mathrm{beg}} = (1 + (w_{i}^{\mathrm{UB}} - 1) \cdot u_{i} ) \cdot \mathrm{Dist}(v_{0}, v_{i});
\end{equation}
where $u_{i} \sim U(0, 1)$ is a uniformly distributed random variable. Finally, the end time of the time window is given by:
\begin{equation}
    w_{i}^{\mathrm{end}} = w_{i}^{\mathrm{beg}} + w_{i}^{\mathrm{len}}.
\end{equation}
Additionally, if the time window constraint is not active in a given problem instance, for all nodes $v_{i}$ ($i \geq 0$), the start time $w_{i}^{\mathrm{beg}}$ is set to zero, the end time $w_{i}^{\mathrm{end}}$ is set to $\infty$, and the service duration $w_{i}^{\mathrm{dur}}$ is set to zero.

\textbf{Distance Limit.} The distance limit $l^{\mathrm{dur}}$ is sampled from the uniform distribution $U(2\cdot \max_{i}(\mathrm{Dist}(v_{0}, v_{i})), l_{\max}^{\mathrm{dur}})$,  where $l_{\max}^{\mathrm{dur}}$ is a predefined upper bound. This sampling strategy ensures that every customer node is reachable from the depot. If the distance limit constraint is not enforced, the distance limit $l^{\mathrm{dur}}$ is set to $\infty$.

\subsection{Feasible Action Space}
\label{appendix:vrp_constraints}

To derive the feasible action space in the environment, we need a action mask based on the combined constraints of basis VRP variants. In each time step, we use the following feasibility testing procedures to mask out infeasible actions from the action space. Please note that the vehicle speed is fixed at $1.0$, and the time windows are normalized accordingly. As a result, the travel time between any two nodes is numerically equivalent to the Euclidean distance between them.

1) Each customer needs to be visited exactly once. If the vehicle is currently located at the depot and there are still unserved customers, selecting the depot as the next action is not permitted.

2) The vehicle must arrive at node $v_j$ from node $v_i$ before the end of the service time window: $w_{t}^{\mathrm{cur}}+\mathrm{Dist}(v_{i}, v_{j}) \leq w_{j}^{\mathrm{end}}$, where $w_{t}^{\mathrm{cur}}$ is the current time.

3) The cost of a subtour, when traveling from node $v_i$ to node $v_j$, must not exceed the predefined distance limit: $l_{t}^{\mathrm{cur}} + \mathrm{Dist}(v_{i}, v_{j}) \leq l^{\mathrm{dur}}$, where $l_{t}^{\mathrm{cur}}$ denotes the cumulative distance traveled along the current subtour.

4) If the vehicle is required to return to the depot and intends to visit node $v_j$ from its current location $v_i$, it must satisfy both time and distance feasibility conditions: the vehicle must be able to complete the visit and return to the depot before the system's end time: $\max(w_{t}^{\mathrm{cur}} + \mathrm{Dist}(v_{i}, v_{j}), w_{j}^{\mathrm{beg}}) + w_{j}^{\mathrm{dur}} + \mathrm{Dist}(v_{0}, v_{j}) < w_{0}^{\mathrm{end}}$, and the total length of the resulting subtour, including travel from $v_i$ to $v_j$ and from $v_j$ back to the depot, must not exceed the distance limit: $l_{t}^{\mathrm{cur}} + \mathrm{Dist}(v_{i}, v_{j}) + \mathrm{Dist}(v_{0}, v_{j}) < l^{\mathrm{dur}}$.

5) If the vehicle intends to visit a linehaul customer $v_j$, the corresponding linehaul demand $q_j^{\mathrm{LH}}$ must not exceed the remaining linehaul capacity $Q_t^{\mathrm{LH}}$. Similarly, if the vehicle intends to visit a backhaul customer $v_j$, the backhaul demand $q_j^{\mathrm{BH}}$ must not exceed the remaining backhaul capacity $Q_t^{\mathrm{BH}}$. Additionally, all linehaul customers must be visited before any backhaul customers within each subtour.

\subsection{VRP Formulation}
\label{appendix:vrp_example}

To illustrate the VRP formulation within the SDMDP framework, we provide an example for clarification. Under the SDMDP framework, there are five types of basis state spaces: $\mathcal{S}^{(0)}, \ldots, \mathcal{S}^{(4)}$. Specifically, $\mathcal{S}^{(0)}$ encodes node coordinates, linehaul demands, and remaining linehaul capacity, $\mathcal{S}^{(1)}$ represents a binary variable indicating whether the vehicle needs to return to the depot; $\mathcal{S}^{(2)}$ captures backhaul demands and remaining backhaul capacity; $\mathcal{S}^{(3)}$ includes the duration limit and the current traveled length of the subtour; and $\mathcal{S}^{(4)}$ contains time windows, service times, and the current time. Before each episode begins, the initial state space is constructed by selecting $\mathcal{S}^{(0)}$ and any subset of the remaining basis states. In this example, we choose $\mathcal{S}^{(0)}$, $\mathcal{S}^{(3)}$ and $\mathcal{S}^{(4)}$, which correspond to the VRP with limited duration and time windows (VRPLTW). At each time step, the policy observes the basis states $s^{(0)}$, $s^{(3)}$, and $s^{(4)}$, and a masking mechanism is applied to filter out infeasible nodes from the set of unvisited nodes, specifically, those whose demands exceed the remaining linehaul capacity, those that would violate the duration limit, and those whose arrival time would exceed the end of their time window, from unvisited nodes. This results in a list of feasible nodes from which the next node is selected. Upon selection, the linehaul capacity, traveled length, and current time are updated accordingly, and the selected node is masked out for subsequent steps. Once a subtour is completed, the linehaul capacity is reset to its default value, the traveled length is reset to zero, and the current time is also reset to zero. This reset reflects the assumption in VRP that multiple vehicles can begin their routes concurrently.

\section{Experiments}
\label{appendix:experiment}

\subsection{Out-of-distribution Attribute Generalization}

In this section, we evaluate the out-of-distribution (OOD) generalization capabilities of our proposed methods, MoSES(RF) and MoSES(CaDA), in comparison with existing unified multi-task neural solvers. The evaluation specifically targets generalization to unseen attribute values pertaining to vehicle capacities, time windows, and distance limits. For each evaluation setting, we construct a testing dataset consisting of 1,000 problem instances, each with $N=100$ nodes. Performance is measured using two metrics: the average cost across the 1,000 problem instances, and the total computational time required to solve all instances, with lower values indicating better performance. Notably, all neural solvers evaluated were trained on problem instances with the same number of nodes (i.e., $N = 100$).

\begin{table}[t]
\centering
\caption{Performance comparison of multi-task VRP solvers on OOD CVRP instances.}
\label{appendix:tab:ood_cvrp}
\resizebox{\textwidth}{!}{\begin{tabular}{lcccccccccccccccc}
\toprule

Vehicle Capacity & \multicolumn{2}{c}{30} & \multicolumn{2}{c}{50(\textbf{ID})} & \multicolumn{2}{c}{70} & \multicolumn{2}{c}{90} & \multicolumn{2}{c}{110} & \multicolumn{2}{c}{130} & \multicolumn{2}{c}{150} & \multicolumn{2}{c}{200} \\

\cmidrule(lr){2-3} \cmidrule(lr){4-5} \cmidrule(lr){6-7} \cmidrule(lr){8-9} \cmidrule(lr){10-11} \cmidrule(lr){12-13} \cmidrule(lr){14-15} \cmidrule(lr){16-17}

& Cost & Time & Cost & Time & Cost & Time & Cost & Time & Cost & Time & Cost & Time & Cost & Time & Cost & Time \\
 
\midrule

MTPOMO      & 23.415 & 9s & 15.933 & 8s & 13.144 & 8s & 11.684 & 8s & 10.835 & 8s & 10.359 & 8s & 10.027 & 8s & 9.566 & 8s \\
MVMoE        & 23.162 & 12s & 15.888 & 11s & 13.068 & 11s & 11.534 & 11s & 10.616 & 11s & 10.089 & 11s & 9.727 & 11s & 9.203 & 11s \\
RF-MoE       & 23.251 & 12s & 15.877 & 11s & 13.080 & 11s & 11.580 & 11s & 10.685 & 11s & 10.170 & 11s & 9.802 & 11s & 9.294 & 11s \\
RF-POMO      & 23.229 & 9s & 15.908 & 8s & 13.104 & 8s & 11.607 & 8s & 10.715 & 8s & 10.200 & 8s & 9.846 & 8s & 9.362 & 8s \\

\cdashline{1-17}

RF-TE        & 23.085 & 9s & 15.857 & 8s & 13.023 & 8s & 11.466 & 8s & 10.521 & 8s & 9.956 & 8s & 9.582 & 8s & 9.003 & 8s \\

MoSES(RF)$_{\mathrm{Sftm}}$    & 23.058 & 22s & 15.826 & 21s & 12.999 & 21s & 11.443 & 20s & \textbf{10.485} & 20s & \textbf{9.922} & 20s & \textbf{9.550} & 20s & \textbf{8.968} & 20s \\

MoSES(RF)$_{\mathrm{Sftp}}$ & \textbf{23.035}& 22s & \textbf{15.808} & 21s & \textbf{12.978} & 20s & \textbf{11.431} & 20s & 10.504 & 20s & 10.107 & 20s & 10.188 & 20s & 10.271 & 20s \\

MoSES(RF)$_{\mathrm{Sigm}}$ & 23.049 & 21s & 15.839 & 20s & 13.009 & 19s & 11.447 & 19s & 10.494 & 19s & 9.922 & 19s & 9.550 & 19s & 8.968 & 19s \\

\cdashline{1-17}

CaDA  & 23.083 & 10s & \textbf{15.831} & 10s & 12.999 & 10s & 11.438 & 9s & 10.479 & 9s & 9.906 & 9s & 9.532 & 9s & 8.942 & 9s \\

MoSES(CaDA)$_{\mathrm{Sftm}}$ & 23.066 & 28s & 15.840 & 26s & 13.015 & 26s & 11.457 & 25s & 10.513 & 25s & 9.950 & 25s & 9.589 & 25s & 9.025 & 25s \\

MoSES(CaDA)$_{\mathrm{Sftp}}$ & \textbf{23.031} & 27s & 15.836 & 26s & 13.005 & 26s & 11.441 & 25s & 10.482 & 25s & 9.906 & 25s & 9.537 & 25s & 8.960 & 25s \\

MoSES(CaDA)$_{\mathrm{Sigm}}$  & 23.047 & 25s & 15.833 & 24s & \textbf{12.999} & 24s & \textbf{11.437} & 24s & \textbf{10.473} & 24s & \textbf{9.897} & 23s & \textbf{9.526} & 23s & \textbf{8.937} & 23s \\

\bottomrule
\end{tabular}}
\end{table}

In CVRP, each neural solver is trained on problem instances with a fixed vehicle capacity of $Q=50$. To investigate OOD generalization to unseen vehicle capacities, we generate a separate testing dataset for each capacity value in the set $\{ 30, 50, 70, 90, 110, 130, 150, 200 \}$. The corresponding results are presented in Table~\ref{appendix:tab:ood_cvrp}. Please note that $Q=50$ corresponds to the in-distribution (ID) evaluation setting. Both MoSES(RF) and MoSES(CaDA) adopt the dense routing strategy. We use $\mathrm{Sftm}$, $\mathrm{Sftp}$ and $\mathrm{Sigm}$ to denote the activation functions $\mathrm{softmax}(\cdot)$, $\mathrm{norm}\_\mathrm{softplus}(\cdot)$ and $\mathrm{sigmoid}(\cdot)$, respectively. For MoSES(RF), we observe MoSES(RF)$_{\mathrm{Sftm}}$ consistently outperforms its baseline, RF-TE, in terms of generalization performance on OOD vehicle capacities. However, MoSES(RF)$_{\mathrm{Sftp}}$ demonstrates superior performance on small-scale problem instances ($N \leq 90$), compared to MoSES(RF)$_{\mathrm{Sftm}}$. MoSES(CaDA)$_{\mathrm{Sigm}}$ demonstrates stronger OOD generalization than its baseline, CaDA, across all evaluation settings, with the exception of a marginal performance drop observed in the ID case at $Q=50$. Furthermore, MoSES(CaDA)$_{\mathrm{Sigm}}$ outperforms MoSES(RF)$_{\mathrm{Sftm}}$ across the majority of evaluation scenarios.

\begin{table}[t]
\centering
\caption{Performance comparison of multi-task VRP solvers on OOD VRPL instances.}
\label{appendix:tab:ood_vrpl}
\resizebox{\textwidth}{!}{\begin{tabular}{lcccccccccccccccc}
\toprule

Distance Limit & \multicolumn{2}{c}{2.7} & \multicolumn{2}{c}{2.8(\textbf{ID})} & \multicolumn{2}{c}{2.9} & \multicolumn{2}{c}{3.0} & \multicolumn{2}{c}{3.1} & \multicolumn{2}{c}{3.2} & \multicolumn{2}{c}{3.3} & \multicolumn{2}{c}{3.4} \\

\cmidrule(lr){2-3} \cmidrule(lr){4-5} \cmidrule(lr){6-7} \cmidrule(lr){8-9} \cmidrule(lr){10-11} \cmidrule(lr){12-13} \cmidrule(lr){14-15} \cmidrule(lr){16-17}

& Cost & Time & Cost & Time & Cost & Time & Cost & Time & Cost & Time & Cost & Time & Cost & Time & Cost & Time \\
 
\midrule

MTPOMO    & 16.193 & 9s & 16.151 & 9s & 16.121 & 9s & 16.105 & 9s & 16.090 & 9s & 16.078 & 9s & 16.069 & 9s & 16.060 & 9s\\
MVMoE      & 16.142 & 12s & 16.099 & 12s & 16.073 & 12s & 16.054 & 12s & 16.040 & 12s & 16.030 & 12s & 16.017 & 12s & 16.010 & 12s\\
RF-MoE     & 16.109 & 12s & 16.070 & 12s & 16.046 & 12s & 16.032 & 12s & 16.017 & 12s & 16.005 & 12s & 15.997 & 12s & 15.991 & 12s\\
RF-POMO    & 16.152 & 9s & 16.106 & 9s & 16.079 & 9s & 16.063 & 9s & 16.049 & 9s & 16.041 & 9s & 16.029 & 9s & 16.019 & 9s\\
RF-TE      & 16.091 & 9s & 16.051 & 9s & 16.029 & 9s & 16.012 & 9s & 15.997 & 9s & 15.987 & 9s & 15.975 & 9s & 15.969 & 9s\\
CaDA       & 16.067 & 10s & 16.026 & 10s & 16.002 & 10s & 15.983 & 10s & 15.971 & 10s & 15.960 & 10s & 15.952 & 10s & 15.946 & 10s\\
MoSES(RF)  & \textbf{16.045} & 21s & \textbf{16.005} & 21s & \textbf{15.982} & 21s & \textbf{15.966} & 21s & \textbf{15.952} & 21s & \textbf{15.940} & 21s & \textbf{15.928} & 21s & \textbf{15.923} & 21s\\
MoSES(CaDA)& 16.070 & 25s & 16.024 & 25s & 16.002 & 25s & 15.984 & 25s & 15.971 & 25s & 15.961 & 25s & 15.947 & 25s & 15.942 & 25s\\

\bottomrule
\end{tabular}}
\end{table}

In VRPL, the predefined upper bound on the distance limit, denoted as $l_{\max}^{\mathrm{dur}}$, is set to 2.8 (approximately $2\sqrt{2}$) during training, which serves as the ID evaluation setting. To examine OOD generalization to unseen distance limits, we evaluate performance across a range of values for $l_{\max}^{\mathrm{dur}}$ from the set $\{ 2.7, 2.8, 2.9, 3.0, 3.1, 3.2, 3.3, 3.4 \}$, where $l_{\max}^{\mathrm{dur}}=2.8$ corresponds to the ID case. The results of this evaluation are summarized in Table~\ref{appendix:tab:ood_vrpl}. MoSES(RF) employs the dense routing strategy with the $\mathrm{norm}\_\mathrm{softplus}(\cdot)$ activation function, while MoSES(CaDA) utilizes the same routing strategy with $\mathrm{sigmid}(\cdot)$ as the activation function. Experimental results indicate that MoSES(RF) consistently outperforms all other methods, including MoSES(CaDA), across all evaluation settings. This demonstrates its superior OOD generalization capability when faced with unseen distance limits.

\begin{table}[t]
\centering
\caption{Performance comparison of multi-task VRP solvers on OOD VRPTW instances.}
\label{appendix:tab:ood_vrptw}
\resizebox{0.98\textwidth}{!}{\begin{tabular}{lcccccccccc}
\toprule

Time Window & \multicolumn{2}{c}{[0.05, 0.08, 0.10]} & \multicolumn{2}{c}{[0.15, 0.18, 0.20](\textbf{ID})} & \multicolumn{2}{c}{[0.25, 0.28, 0.30]} & \multicolumn{2}{c}{[0.35, 0.38, 0.40]} & \multicolumn{2}{c}{[0.45, 0.48, 0.50]} \\

\cmidrule(lr){2-3} \cmidrule(lr){4-5} \cmidrule(lr){6-7} \cmidrule(lr){8-9} \cmidrule(lr){10-11}

& Cost & Time & Cost & Time & Cost & Time & Cost & Time & Cost & Time \\
 
\midrule

MTPOMO 
    & 25.549 & 9s & 26.410 & 9s & 28.294 & 9s & 31.339 & 10s & 35.217 & 10s \\
MVMoE 
    & 25.490 & 12s & 26.391 & 12s & 28.258 & 12s & 31.237 & 13s & 35.041 & 13s \\
RF-MoE 
    & 25.459 & 12s & 26.319 & 12s & 28.227 & 12s & 31.338 & 13s & 35.251 & 13s \\
RF-POMO 
    & 25.492 & 9s & 26.335 & 9s & 28.242 & 9s & 31.408 & 10s & 35.315 & 10s \\
RF-TE 
    & 25.377 & 9s & 26.234 & 9s & 28.156 & 9s & 31.268 & 9s & 35.135 & 10s \\
CaDA
    & 25.309 & 11s & 26.128 & 10s & 28.095 & 11s & 31.334 & 11s & 35.329 & 11s \\
MoSES(RF) 
    & 25.271 & 22s & 26.143 & 22s & \textbf{28.027} & 23s & \textbf{30.943} & 24s & \textbf{34.500} & 24s \\
MoSES(CaDA) 
    & \textbf{25.255} & 26s & \textbf{26.032} & 26s & 28.054 & 26s & 31.354 & 27s & 35.223 & 28s \\

\toprule

Time Window & \multicolumn{2}{c}{[0.55, 0.58, 0.60]} & \multicolumn{2}{c}{[0.65, 0.68, 0.70]} & \multicolumn{2}{c}{[0.75, 0.78, 0.80]} & \multicolumn{2}{c}{[0.85, 0.88, 0.90]} & \multicolumn{2}{c}{[0.95, 0.98, 1.00]} \\

\cmidrule(lr){2-3} \cmidrule(lr){4-5} \cmidrule(lr){6-7} \cmidrule(lr){8-9} \cmidrule(lr){10-11}

& Cost & Time & Cost & Time & Cost & Time & Cost & Time & Cost & Time \\

\midrule

MTPOMO 
    & 39.549 & 10s & 44.035 & 11s & 48.215 & 11s & 52.295 & 11s & 55.933 & 11s \\
MVMoE 
    & 39.349 & 14s & 43.781 & 14s & 47.995 & 15s & 52.160 & 16s & 55.903 & 16s \\
RF-MoE 
    & 39.606 & 14s & 44.026 & 14s & 48.082 & 15s & 52.029 & 15s & 55.465 & 15s \\
RF-POMO 
    & 39.668 & 10s & 44.102 & 11s & 48.144 & 11s & 52.090 & 11s & 55.539 & 11s \\
RF-TE 
    & 39.402 & 10s & 43.783 & 10s & 47.812 & 11s & 51.869 & 11s & 55.424 & 11s \\
CaDA 
    & 39.697 & 12s & 44.152 & 12s & 48.231 & 12s & 52.122 & 13s & 55.650 & 13s \\
MoSES(RF)
    & \textbf{38.490} & 25s & \textbf{42.600} & 26s & \textbf{46.770} & 27s & \textbf{51.475} & 28s & 55.973 & 29s \\
MoSES(CaDA) 
    & 39.488 & 29s & 43.864 & 29s & 47.854 & 30s & 51.784 & 30s & \textbf{55.172} & 31s \\

\bottomrule
\end{tabular}}
\end{table}

In VRPTW, each problem instance is primarily defined by both the time window and the service duration. These are generated based on two intervals $[I_{1}, I_{2}]$ and ${I_{2}, I_{3}}$ as described in Section~\ref{appendix:problem_instance}. Let the triplet $[I_{1}, I_{2}, I_{3}]$ represent the full configurations. The setting $[0.15, 0.18, 0.20]$ serves as the in-distribution ID evaluation. To assess OOD generalization to unseen time window configurations, we consider a range of settings from the set $\{ [0.05, 0.08, 0.10], [0.15, 0.18, 0.20], [0.25, 0.28, 0.30], \ldots, [0.95, 0.98, 1.00] \}$, as reported in Table~\ref{appendix:tab:ood_vrptw}. Both MoSES(RF) and MoSES(CaDA) adopt the dense routing strategy. MoSES(RF) uses $\mathrm{norm\_softplus}(\cdot)$ as the activation function, while MoSES(CaDA) employs $\mathrm{sigmoid}(\cdot)$. Experimental results show that MoSES(RF) generally outperforms its baseline, RF-TE, across most evaluation settings, with the exception of a performance drop observed at $[0.95, 0.98, 1.00]$. Similarly, MoSES(CaDA) surpasses its baseline, CaDA, except for a marginal decline in performance at $[0.35, 0.38, 0.40]$. Overall, MoSES(RF) demonstrates superior performance compared to MoSES(CaDA) in the majority of task settings.

\subsection{CVRPLIB Evaluation}

We report the performance comparison of our proposed methods, MoSES(RF) and MoSES(CaDA), against baseline approaches on CVRPLIB instances from the X set, which includes problem sizes ranging from 101 to at most 1,001 nodes, as done in~\cite{zhou2024mvmoe, berto2024routefinder}. Both MoSES(RF) and MoSES(CaDA) use the Variant-Aware Routing-\uppercase\expandafter{\romannumeral1} strategy. MoSES(RF) uses $\mathrm{norm\_softplus}(\cdot)$ as the activation function, while MoSES(CaDA) employs $\mathrm{sigmoid}(\cdot)$. The detailed results are presented in Table~\ref{appendix:tab:cvrplib_x}. To facilitate analysis, we partition the evaluation set based on problem instance size into three subsets: instances with $N < 251$, $251 \leq N < 501$ and $501 < N \leq 1001$. Across all three subsets, both MoSES(RF) and MoSES(CaDA) consistently outperform their respective baselines, RF-TE and CaDA. On the larger instance sets ($251 \leq N < 501$ and $501 < N \leq 1001$), MoSES(RF) demonstrates superior performance compared to MoSES(CaDA). Conversely, on the smaller instance set ($N < 251$), MoSES(CaDA) achieves better results than MoSES(RF). Overall, MoSES(RF) exhibits the best performance among all evaluated methods.

\begin{table}[t]
\centering
\caption{Performance comparison of multi-task VRP solvers on CVRPLIB instances from the X set.}
\label{appendix:tab:cvrplib_x}
\resizebox{\textwidth}{!}{\begin{tabular}{ll|cccccccccccccccccccccccc}
\toprule

\multicolumn{2}{c|}{Set-X} & \multicolumn{3}{c}{MTPOMO} & \multicolumn{3}{c}{MVMoE} & \multicolumn{3}{c}{RF-MoE} & \multicolumn{3}{c}{RF-POMO} & \multicolumn{3}{c}{RF-TE} & \multicolumn{3}{c}{CaDA} & \multicolumn{3}{c}{MoSES(RF)} & \multicolumn{3}{c}{MoSES(CaDA)} \\

\midrule

 Instance & Opt. & Cost & Gap & Time & Cost & Gap & Time & Cost & Gap & Time & Cost & Gap & Time & Cost & Gap & Time & Cost & Gap & Time & Cost & Gap & Time & Cost & Gap & Time \\

\midrule

X-n101-k25 & 27591 & 29470 & 6.810\% & 0.4s & 29076 & 5.382\% & 0.5s & 28934 & 4.868\% & 0.5s & 29090 & 5.433\% & 0.3s & 29048 & 5.281\% & 0.4s & 28944 & 4.904\% & 0.5s & 28895 & 4.726\% & 0.8s & 29110 & 5.505\% & 0.8s \\

X-n106-k14 & 26362 & 28029 & 6.323\% & 0.3s & 27443 & 4.101\% & 0.5s & 27292 & 3.528\% & 0.6s & 27378 & 3.854\% & 0.3s & 27159 & 3.023\% & 0.4s & 27042 & 2.579\% & 0.3s & 27205 & 3.198\% & 0.7s & 27051 & 2.614\% & 0.8s \\

X-n110-k13 & 14971 & 15100 & 0.862\% & 0.3s & 15327 & 2.378\% & 0.5s & 15260 & 1.930\% & 0.5s & 15519 & 3.660\% & 0.3s & 15314 & 2.291\% & 0.4s & 15229 & 1.723\% & 0.3s & 15242 & 1.810\% & 0.8s & 15332 & 2.411\% & 0.8s \\

X-n115-k10 & 12747 & 13433 & 5.382\% & 0.4s & 13475 & 5.711\% & 0.6s & 13638 & 6.990\% & 0.5s & 13263 & 4.048\% & 0.3s & 13338 & 4.636\% & 0.4s & 13060 & 2.455\% & 0.4s & 13313 & 4.440\% & 0.8s & 13085 & 2.652\% & 0.8s \\

X-n120-k6 & 13332 & 14051 & 5.393\% & 0.3s & 13782 & 3.375\% & 0.6s & 13908 & 4.320\% & 0.6s & 14061 & 5.468\% & 0.4s & 13765 & 3.248\% & 0.4s & 13678 & 2.595\% & 0.3s & 13781 & 3.368\% & 0.8s & 13619 & 2.153\% & 0.8s \\

X-n125-k30 & 55539 & 59015 & 6.259\% & 0.4s & 58200 & 4.791\% & 0.7s & 58587 & 5.488\% & 0.6s & 58770 & 5.818\% & 0.4s & 58570 & 5.457\% & 0.4s & 57748 & 3.977\% & 0.4s & 58220 & 4.827\% & 0.9s & 57620 & 3.747\% & 1.0s \\

X-n129-k18 & 28940 & 30176 & 4.271\% & 0.4s & 29334 & 1.361\% & 0.6s & 30039 & 3.798\% & 0.6s & 29645 & 2.436\% & 0.4s & 29457 & 1.786\% & 0.5s & 29500 & 1.935\% & 0.4s & 29558 & 2.135\% & 0.9s & 29620 & 2.350\% & 0.9s \\

X-n134-k13 & 10916 & 11707 & 7.246\% & 0.4s & 11462 & 5.002\% & 0.6s & 11439 & 4.791\% & 0.6s & 11463 & 5.011\% & 0.4s & 11624 & 6.486\% & 0.4s & 11652 & 6.742\% & 0.4s & 11584 & 6.119\% & 0.8s & 11573 & 6.019\% & 1.0s \\

X-n139-k10 & 13590 & 14058 & 3.444\% & 0.4s & 14099 & 3.745\% & 0.6s & 13917 & 2.406\% & 0.6s & 13945 & 2.612\% & 0.4s & 13812 & 1.634\% & 0.4s & 13940 & 2.575\% & 0.4s & 13908 & 2.340\% & 0.8s & 13877 & 2.112\% & 0.9s \\

X-n143-k7 & 15700 & 16626 & 5.898\% & 0.4s & 16349 & 4.134\% & 0.6s & 16655 & 6.083\% & 0.6s & 16603 & 5.752\% & 0.5s & 16257 & 3.548\% & 0.4s & 16189 & 3.115\% & 0.4s & 16024 & 2.064\% & 0.9s & 15980 & 1.783\% & 0.9s \\

X-n148-k46 & 43448 & 46648 & 7.365\% & 0.5s & 45893 & 5.627\% & 0.8s & 46542 & 7.121\% & 0.8s & 46082 & 6.062\% & 0.5s & 45026 & 3.632\% & 0.6s & 45606 & 4.967\% & 0.6s & 45408 & 4.511\% & 1.0s & 45600 & 4.953\% & 1.0s \\

X-n153-k22 & 21220 & 23514 & 10.811\% & 0.5s & 23661 & 11.503\% & 0.7s & 23906 & 12.658\% & 0.7s & 22991 & 8.346\% & 0.5s & 23478 & 10.641\% & 0.6s & 23142 & 9.057\% & 0.5s & 23347 & 10.024\% & 1.0s & 23310 & 9.849\% & 1.0s \\

X-n157-k13 & 16876 & 17922 & 6.198\% & 0.5s & 17439 & 3.336\% & 0.7s & 17801 & 5.481\% & 0.8s & 17536 & 3.911\% & 0.5s & 17315 & 2.601\% & 0.5s & 17295 & 2.483\% & 0.5s & 17227 & 2.080\% & 1.0s & 17317 & 2.613\% & 1.0s \\

X-n162-k11 & 14138 & 14616 & 3.381\% & 0.5s & 14705 & 4.010\% & 0.7s & 14524 & 2.730\% & 0.7s & 14663 & 3.713\% & 0.5s & 14664 & 3.720\% & 0.5s & 14704 & 4.003\% & 0.5s & 14683 & 3.855\% & 1.0s & 14677 & 3.812\% & 1.0s \\

X-n167-k10 & 20557 & 21662 & 5.375\% & 0.5s & 21504 & 4.607\% & 0.7s & 21481 & 4.495\% & 0.7s & 21410 & 4.149\% & 0.5s & 21425 & 4.222\% & 0.5s & 21078 & 2.534\% & 0.5s & 21368 & 3.945\% & 1.0s & 21384 & 4.023\% & 1.0s \\

X-n172-k51 & 45607 & 48960 & 7.352\% & 0.6s & 47883 & 4.990\% & 0.9s & 49726 & 9.032\% & 1.0s & 48412 & 6.150\% & 0.6s & 48162 & 5.602\% & 0.7s & 48198 & 5.681\% & 0.6s & 48136 & 5.545\% & 1.0s & 48145 & 5.565\% & 1.0s \\

X-n176-k26 & 47812 & 51989 & 8.736\% & 0.5s & 52117 & 9.004\% & 0.8s & 53626 & 12.160\% & 0.9s & 52347 & 9.485\% & 0.6s & 51501 & 7.716\% & 0.6s & 51120 & 6.919\% & 0.6s & 52001 & 8.761\% & 1.0s & 51612 & 7.948\% & 1.0s \\

X-n181-k23 & 25569 & 26572 & 3.923\% & 0.6s & 26456 & 3.469\% & 0.8s & 29154 & 14.021\% & 0.9s & 26544 & 3.813\% & 0.6s & 26097 & 2.065\% & 0.6s & 26262 & 2.710\% & 0.6s & 26181 & 2.394\% & 1.0s & 26143 & 2.245\% & 1.0s \\

X-n186-k15 & 24145 & 25236 & 4.519\% & 0.5s & 25151 & 4.166\% & 0.8s & 25140 & 4.121\% & 0.9s & 25238 & 4.527\% & 0.5s & 25153 & 4.175\% & 0.6s & 25345 & 4.970\% & 0.6s & 25115 & 4.017\% & 1.0s & 25246 & 4.560\% & 1.0s \\

X-n190-k8 & 16980 & 18369 & 8.180\% & 0.5s & 19078 & 12.356\% & 0.9s & 18217 & 7.285\% & 0.9s & 18696 & 10.106\% & 0.6s & 17871 & 5.247\% & 0.6s & 17882 & 5.312\% & 0.6s & 17929 & 5.589\% & 1.0s & 17569 & 3.469\% & 1.0s \\

X-n195-k51 & 44225 & 48310 & 9.237\% & 0.7s & 46974 & 6.216\% & 1.0s & 48965 & 10.718\% & 1.0s & 47479 & 7.358\% & 0.7s & 47396 & 7.170\% & 0.7s & 46723 & 5.648\% & 0.7s & 46541 & 5.237\% & 1.0s & 47479 & 7.358\% & 1.0s \\

X-n200-k36 & 58578 & 62041 & 5.912\% & 0.6s & 61627 & 5.205\% & 0.9s & 61696 & 5.323\% & 0.9s & 61662 & 5.265\% & 0.6s & 61139 & 4.372\% & 0.7s & 61010 & 4.152\% & 0.8s & 61088 & 4.285\% & 1.0s & 61089 & 4.287\% & 2.0s \\

X-n204-k19 & 19565 & 20652 & 5.556\% & 0.6s & 20584 & 5.208\% & 0.9s & 20466 & 4.605\% & 1.0s & 20730 & 5.955\% & 0.6s & 20531 & 4.937\% & 0.6s & 20735 & 5.980\% & 0.6s & 20620 & 5.392\% & 1.0s & 20420 & 4.370\% & 1.0s \\

X-n209-k16 & 30656 & 32333 & 5.470\% & 0.6s & 32358 & 5.552\% & 0.9s & 32145 & 4.857\% & 0.9s & 32585 & 6.292\% & 0.6s & 31876 & 3.980\% & 0.6s & 32184 & 4.984\% & 0.6s & 31775 & 3.650\% & 1.0s & 32053 & 4.557\% & 1.0s \\

X-n214-k11 & 10856 & 11699 & 7.765\% & 0.6s & 11597 & 6.826\% & 0.9s & 11534 & 6.245\% & 0.9s & 11638 & 7.203\% & 0.6s & 11668 & 7.480\% & 0.6s & 11748 & 8.217\% & 0.6s & 11635 & 7.176\% & 1.0s & 11716 & 7.922\% & 1.0s \\

X-n219-k73 & 117595 & 121980 & 3.729\% & 0.8s & 124434 & 5.816\% & 1.0s & 121627 & 3.429\% & 1.0s & 123500 & 5.021\% & 0.8s & 120344 & 2.338\% & 0.8s & 120011 & 2.055\% & 0.8s & 119497 & 1.617\% & 2.0s & 119710 & 1.799\% & 2.0s \\

X-n223-k34 & 40437 & 43381 & 7.280\% & 0.7s & 42694 & 5.582\% & 1.0s & 43097 & 6.578\% & 1.0s & 42601 & 5.352\% & 0.7s & 42251 & 4.486\% & 0.7s & 42273 & 4.540\% & 0.7s & 42312 & 4.637\% & 1.0s & 42128 & 4.182\% & 2.0s \\

X-n228-k23 & 25742 & 28523 & 10.803\% & 0.7s & 28033 & 8.900\% & 1.0s & 29590 & 14.948\% & 1.0s & 28212 & 9.595\% & 0.8s & 28699 & 11.487\% & 0.8s & 27821 & 8.076\% & 0.7s & 27701 & 7.610\% & 1.0s & 27724 & 7.699\% & 1.0s \\

X-n233-k16 & 19230 & 20644 & 7.353\% & 0.7s & 20656 & 7.415\% & 1.0s & 20507 & 6.641\% & 1.0s & 20427 & 6.225\% & 0.7s & 20761 & 7.962\% & 0.7s & 20285 & 5.486\% & 0.9s & 20552 & 6.875\% & 1.0s & 20623 & 7.244\% & 1.0s \\

X-n237-k14 & 27042 & 30047 & 11.112\% & 0.7s & 29772 & 10.095\% & 1.0s & 29514 & 9.141\% & 1.0s & 30084 & 11.249\% & 0.7s & 29595 & 9.441\% & 0.7s & 30282 & 11.981\% & 0.7s & 29720 & 9.903\% & 1.0s & 29518 & 9.156\% & 1.0s \\

X-n242-k48 & 82751 & 88179 & 6.559\% & 0.8s & 87497 & 5.735\% & 1.0s & 87832 & 6.140\% & 1.0s & 87029 & 5.170\% & 0.8s & 85704 & 3.569\% & 0.9s & 85813 & 3.700\% & 0.8s & 85420 & 3.225\% & 2.0s & 85643 & 3.495\% & 2.0s \\

X-n247-k50 & 37274 & 41610 & 11.633\% & 0.8s & 40973 & 9.924\% & 1.0s & 43153 & 15.772\% & 1.0s & 41120 & 10.318\% & 0.8s & 40642 & 9.036\% & 0.9s & 39918 & 7.093\% & 0.8s & 40131 & 7.665\% & 2.0s & 40736 & 9.288\% & 2.0s \\

\midrule

\multicolumn{2}{c|}{Avg. Gap ($N<251$)} & \multicolumn{3}{c}{6.566\%} & \multicolumn{3}{c}{5.829\%} & \multicolumn{3}{c}{6.754\%} & \multicolumn{3}{c}{5.905\%} & \multicolumn{3}{c}{5.061\%} & \multicolumn{3}{c}{4.772\%} & \multicolumn{3}{c}{4.789\%} & \multicolumn{3}{c}{\textbf{4.724\%}} \\

\midrule

X-n251-k28 & 38684 & 41211 & 6.532\% & 0.7s & 41330 & 6.840\% & 1.0s & 40691 & 5.188\% & 1.0s & 40811 & 5.498\% & 0.8s & 40127 & 3.730\% & 0.8s & 40359 & 4.330\% & 0.8s & 40630 & 5.031\% & 2.0s & 40290 & 4.152\% & 2.0s \\

X-n256-k16 & 18839 & 20400 & 8.286\% & 0.7s & 20559 & 9.130\% & 1.0s & 20015 & 6.242\% & 1.0s & 20238 & 7.426\% & 0.7s & 19994 & 6.131\% & 0.8s & 20372 & 8.137\% & 0.7s & 20034 & 6.343\% & 1.0s & 20068 & 6.524\% & 2.0s \\

X-n261-k13 & 26558 & 28741 & 8.220\% & 0.7s & 28524 & 7.403\% & 1.0s & 28203 & 6.194\% & 1.0s & 28525 & 7.406\% & 0.8s & 28510 & 7.350\% & 0.8s & 28833 & 8.566\% & 1.0s & 28447 & 7.113\% & 2.0s & 28577 & 7.602\% & 2.0s \\

X-n266-k58 & 75478 & 84617 & 12.108\% & 0.9s & 82048 & 8.705\% & 1.0s & 81135 & 7.495\% & 1.0s & 81053 & 7.386\% & 0.9s & 79832 & 5.769\% & 0.9s & 80115 & 6.144\% & 0.9s & 79820 & 5.753\% & 2.0s & 80036 & 6.039\% & 2.0s \\

X-n270-k35 & 35291 & 38146 & 8.090\% & 0.9s & 38333 & 8.620\% & 1.0s & 37401 & 5.979\% & 1.0s & 38051 & 7.821\% & 0.8s & 37382 & 5.925\% & 0.9s & 37674 & 6.752\% & 0.8s & 37420 & 6.033\% & 2.0s & 36923 & 4.624\% & 2.0s \\

X-n275-k28 & 21245 & 24688 & 16.206\% & 0.8s & 25021 & 17.774\% & 1.0s & 25241 & 18.809\% & 1.0s & 24321 & 14.479\% & 0.8s & 24187 & 13.848\% & 0.9s & 24482 & 15.237\% & 0.8s & 24292 & 14.342\% & 2.0s & 24312 & 14.436\% & 2.0s \\

X-n280-k17 & 33503 & 36677 & 9.474\% & 0.8s & 36636 & 9.351\% & 1.0s & 36538 & 9.059\% & 1.0s & 35558 & 6.134\% & 0.9s & 36653 & 9.402\% & 0.9s & 36081 & 7.695\% & 0.8s & 35988 & 7.417\% & 2.0s & 35494 & 5.943\% & 2.0s \\

X-n284-k15 & 20226 & 22474 & 11.114\% & 0.8s & 22583 & 11.653\% & 1.0s & 21857 & 8.064\% & 1.0s & 21976 & 8.652\% & 0.8s & 22154 & 9.532\% & 0.8s & 22295 & 10.229\% & 0.8s & 22035 & 8.944\% & 2.0s & 22071 & 9.122\% & 2.0s \\

X-n289-k60 & 95151 & 104159 & 9.467\% & 0.9s & 102202 & 7.410\% & 2.0s & 102267 & 7.479\% & 2.0s & 101494 & 6.666\% & 1.0s & 100418 & 5.535\% & 1.0s & 99739 & 4.822\% & 1.0s & 100733 & 5.866\% & 2.0s & 100080 & 5.180\% & 2.0s \\

X-n294-k50 & 47161 & 52769 & 11.891\% & 0.9s & 50886 & 7.898\% & 2.0s & 51924 & 10.099\% & 1.0s & 51033 & 8.210\% & 0.9s & 50637 & 7.370\% & 1.0s & 49929 & 5.869\% & 1.0s & 50538 & 7.161\% & 2.0s & 49877 & 5.759\% & 2.0s \\

X-n298-k31 & 34231 & 37652 & 9.994\% & 0.9s & 37344 & 9.094\% & 1.0s & 36808 & 7.528\% & 1.0s & 36785 & 7.461\% & 0.9s & 37163 & 8.565\% & 0.9s & 36993 & 8.069\% & 1.0s & 36876 & 7.727\% & 2.0s & 37068 & 8.288\% & 2.0s \\

X-n303-k21 & 21736 & 23556 & 8.373\% & 0.9s & 23263 & 7.025\% & 1.0s & 23027 & 5.939\% & 1.0s & 23097 & 6.262\% & 0.9s & 23442 & 7.849\% & 0.9s & 23748 & 9.257\% & 0.9s & 23453 & 7.899\% & 2.0s & 23548 & 8.336\% & 2.0s \\

X-n308-k13 & 25859 & 28736 & 11.126\% & 0.9s & 28518 & 10.283\% & 1.0s & 29079 & 12.452\% & 1.0s & 28030 & 8.396\% & 0.9s & 28326 & 9.540\% & 0.9s & 28913 & 11.810\% & 1.0s & 28138 & 8.813\% & 2.0s & 28440 & 9.981\% & 2.0s \\

X-n313-k71 & 94043 & 102253 & 8.730\% & 1.0s & 100620 & 6.994\% & 2.0s & 100714 & 7.094\% & 2.0s & 100083 & 6.423\% & 1.0s & 99564 & 5.871\% & 1.0s & 98899 & 5.164\% & 1.0s & 98738 & 4.992\% & 2.0s & 98931 & 5.198\% & 2.0s \\

X-n317-k53 & 78355 & 82587 & 5.401\% & 1.0s & 83632 & 6.735\% & 2.0s & 87360 & 11.493\% & 2.0s & 81981 & 4.628\% & 1.0s & 80690 & 2.980\% & 1.0s & 80542 & 2.791\% & 1.0s & 80709 & 3.004\% & 2.0s & 80472 & 2.702\% & 2.0s \\

X-n322-k28 & 29834 & 32593 & 9.248\% & 1.0s & 33497 & 12.278\% & 1.0s & 32143 & 7.739\% & 1.0s & 32403 & 8.611\% & 0.9s & 32658 & 9.466\% & 1.0s & 33206 & 11.303\% & 1.0s & 32648 & 9.432\% & 2.0s & 32541 & 9.074\% & 2.0s \\

X-n327-k20 & 27532 & 30646 & 11.310\% & 1.0s & 30603 & 11.154\% & 1.0s & 29649 & 7.689\% & 1.0s & 29638 & 7.649\% & 0.9s & 29784 & 8.180\% & 1.0s & 30953 & 12.426\% & 1.0s & 29793 & 8.212\% & 2.0s & 30089 & 9.287\% & 2.0s \\

X-n331-k15 & 31102 & 34734 & 11.678\% & 0.9s & 33636 & 8.147\% & 1.0s & 34431 & 10.703\% & 2.0s & 33597 & 8.022\% & 1.0s & 34048 & 9.472\% & 1.0s & 34578 & 11.176\% & 1.0s & 33526 & 7.794\% & 2.0s & 34014 & 9.363\% & 2.0s \\

X-n336-k84 & 139111 & 152846 & 9.873\% & 1.0s & 149229 & 7.273\% & 2.0s & 150468 & 8.164\% & 2.0s & 147371 & 5.938\% & 1.0s & 146620 & 5.398\% & 1.0s & 146707 & 5.460\% & 1.0s & 147177 & 5.798\% & 3.0s & 146465 & 5.286\% & 2.0s \\

X-n344-k43 & 42050 & 46619 & 10.866\% & 1.0s & 46947 & 11.646\% & 2.0s & 45143 & 7.356\% & 2.0s & 46098 & 9.627\% & 1.0s & 44914 & 6.811\% & 1.0s & 45571 & 8.373\% & 1.0s & 45232 & 7.567\% & 2.0s & 44746 & 6.411\% & 2.0s \\

X-n351-k40 & 25896 & 29243 & 12.925\% & 1.0s & 28373 & 9.565\% & 2.0s & 28728 & 10.936\% & 2.0s & 28628 & 10.550\% & 1.0s & 28236 & 9.036\% & 1.0s & 28059 & 8.353\% & 1.0s & 28124 & 8.604\% & 2.0s & 28130 & 8.627\% & 2.0s \\

X-n359-k29 & 51505 & 55778 & 8.296\% & 1.0s & 56165 & 9.048\% & 2.0s & 54690 & 6.184\% & 2.0s & 55013 & 6.811\% & 1.0s & 55122 & 7.023\% & 1.0s & 55183 & 7.141\% & 1.0s & 55231 & 7.234\% & 2.0s & 55158 & 7.093\% & 2.0s \\

X-n367-k17 & 22814 & 26132 & 14.544\% & 1.0s & 25588 & 12.159\% & 2.0s & 26470 & 16.025\% & 2.0s & 25150 & 10.239\% & 1.0s & 25522 & 11.870\% & 1.0s & 25534 & 11.923\% & 1.0s & 24728 & 8.390\% & 2.0s & 25489 & 11.725\% & 2.0s \\

X-n376-k94 & 147713 & 156857 & 6.190\% & 1.0s & 156546 & 5.980\% & 2.0s & 156077 & 5.662\% & 2.0s & 158456 & 7.273\% & 1.0s & 151975 & 2.885\% & 1.0s & 151390 & 2.489\% & 1.0s & 151521 & 2.578\% & 3.0s & 151614 & 2.641\% & 3.0s \\

X-n384-k52 & 65940 & 73705 & 11.776\% & 1.0s & 73570 & 11.571\% & 2.0s & 70853 & 7.451\% & 2.0s & 71089 & 7.809\% & 1.0s & 70471 & 6.871\% & 1.0s & 70611 & 7.084\% & 1.0s & 70775 & 7.332\% & 2.0s & 70479 & 6.884\% & 2.0s \\

X-n393-k38 & 38260 & 43533 & 13.782\% & 1.0s & 44638 & 16.670\% & 2.0s & 41843 & 9.365\% & 2.0s & 42161 & 10.196\% & 1.0s & 41552 & 8.604\% & 1.0s & 42934 & 12.216\% & 1.0s & 41924 & 9.577\% & 2.0s & 42192 & 10.277\% & 3.0s \\

X-n401-k29 & 66154 & 71565 & 8.179\% & 1.0s & 71787 & 8.515\% & 2.0s & 69492 & 5.046\% & 2.0s & 70480 & 6.539\% & 1.0s & 69430 & 4.952\% & 1.0s & 69875 & 5.625\% & 1.0s & 69241 & 4.666\% & 3.0s & 69991 & 5.800\% & 2.0s \\

X-n411-k19 & 19712 & 23869 & 21.089\% & 1.0s & 23139 & 17.385\% & 2.0s & 24162 & 22.575\% & 2.0s & 22203 & 12.637\% & 1.0s & 22849 & 15.914\% & 1.0s & 23521 & 19.323\% & 1.0s & 22489 & 14.088\% & 2.0s & 22768 & 15.503\% & 2.0s \\

X-n420-k130 & 107798 & 122761 & 13.881\% & 2.0s & 116362 & 7.944\% & 2.0s & 120841 & 12.099\% & 2.0s & 118046 & 9.507\% & 2.0s & 117418 & 8.924\% & 2.0s & 115012 & 6.692\% & 2.0s & 115838 & 7.458\% & 3.0s & 116853 & 8.400\% & 3.0s \\

X-n429-k61 & 65449 & 74261 & 13.464\% & 1.0s & 74158 & 13.307\% & 2.0s & 71017 & 8.507\% & 2.0s & 71070 & 8.588\% & 1.0s & 70164 & 7.204\% & 2.0s & 70969 & 8.434\% & 1.0s & 70639 & 7.930\% & 3.0s & 70617 & 7.896\% & 3.0s \\

X-n439-k37 & 36391 & 41165 & 13.119\% & 1.0s & 42161 & 15.856\% & 2.0s & 38998 & 7.164\% & 2.0s & 39947 & 9.772\% & 1.0s & 39752 & 9.236\% & 1.0s & 41149 & 13.075\% & 1.0s & 39799 & 9.365\% & 3.0s & 39697 & 9.085\% & 3.0s \\

X-n449-k29 & 55233 & 60162 & 8.924\% & 1.0s & 60015 & 8.658\% & 2.0s & 59919 & 8.484\% & 2.0s & 59925 & 8.495\% & 1.0s & 60634 & 9.779\% & 1.0s & 61144 & 10.702\% & 1.0s & 60340 & 9.246\% & 3.0s & 60723 & 9.940\% & 3.0s \\

X-n459-k26 & 24139 & 29543 & 22.387\% & 1.0s & 29100 & 20.552\% & 2.0s & 26995 & 11.831\% & 2.0s & 27224 & 12.780\% & 1.0s & 27347 & 13.290\% & 2.0s & 28267 & 17.101\% & 1.0s & 27107 & 12.295\% & 3.0s & 27510 & 13.965\% & 3.0s \\

X-n469-k138 & 221824 & 252031 & 13.618\% & 2.0s & 245581 & 10.710\% & 3.0s & 242533 & 9.336\% & 3.0s & 242197 & 9.184\% & 2.0s & 238904 & 7.700\% & 2.0s & 237548 & 7.089\% & 2.0s & 236859 & 6.778\% & 4.0s & 237001 & 6.842\% & 3.0s \\

X-n480-k70 & 89449 & 101314 & 13.265\% & 2.0s & 100121 & 11.931\% & 2.0s & 96042 & 7.371\% & 3.0s & 96484 & 7.865\% & 2.0s & 95032 & 6.242\% & 2.0s & 95466 & 6.727\% & 2.0s & 95101 & 6.319\% & 4.0s & 95211 & 6.442\% & 3.0s \\

X-n491-k59 & 66483 & 77536 & 16.625\% & 2.0s & 75226 & 13.151\% & 2.0s & 72443 & 8.965\% & 3.0s & 72142 & 8.512\% & 2.0s & 72618 & 9.228\% & 2.0s & 71702 & 7.850\% & 2.0s & 72383 & 8.874\% & 3.0s & 71730 & 7.892\% & 3.0s \\

\midrule

\multicolumn{2}{c|}{Avg. Gap ($251 \leq N<501$)} & \multicolumn{3}{c}{11.529\%} & \multicolumn{3}{c}{10.616\%} & \multicolumn{3}{c}{9.217\%} & \multicolumn{3}{c}{8.399\%} & \multicolumn{3}{c}{8.107\%} & \multicolumn{3}{c}{8.889\%} & \multicolumn{3}{c}{\textbf{7.741\%}} & \multicolumn{3}{c}{7.948\%} \\

\midrule

X-n502-k39 & 69226 & 75711 & 9.368\% & 2.0s & 77033 & 11.278\% & 3.0s & 73557 & 6.256\% & 3.0s & 74317 & 7.354\% & 2.0s & 71908 & 3.874\% & 2.0s & 72655 & 4.953\% & 2.0s & 72023 & 4.040\% & 3.0s & 71682 & 3.548\% & 3.0s \\

X-n513-k21 & 24201 & 34910 & 44.250\% & 2.0s & 32858 & 35.771\% & 2.0s & 27867 & 15.148\% & 2.0s & 27871 & 15.165\% & 2.0s & 28542 & 17.937\% & 2.0s & 29422 & 21.573\% & 2.0s & 27907 & 15.313\% & 3.0s & 29139 & 20.404\% & 3.0s \\

X-n524-k153 & 154593 & 176491 & 14.165\% & 2.0s & 171734 & 11.088\% & 3.0s & 178794 & 15.655\% & 3.0s & 172181 & 11.377\% & 2.0s & 174150 & 12.651\% & 2.0s & 168181 & 8.790\% & 2.0s & 171777 & 11.116\% & 4.0s & 172580 & 11.635\% & 4.0s \\

X-n536-k96 & 94846 & 109897 & 15.869\% & 2.0s & 106031 & 11.793\% & 3.0s & 103862 & 9.506\% & 3.0s & 103854 & 9.498\% & 2.0s & 103242 & 8.852\% & 2.0s & 102355 & 7.917\% & 2.0s & 102432 & 7.998\% & 4.0s & 101712 & 7.239\% & 4.0s \\

X-n548-k50 & 86700 & 110984 & 28.009\% & 2.0s & 104240 & 20.231\% & 3.0s & 101294 & 16.833\% & 3.0s & 101549 & 17.127\% & 2.0s & 100850 & 16.321\% & 2.0s & 102318 & 18.014\% & 2.0s & 100550 & 15.975\% & 4.0s & 101918 & 17.552\% & 4.0s \\

X-n561-k42 & 42717 & 55936 & 30.946\% & 2.0s & 53110 & 24.330\% & 3.0s & 47544 & 11.300\% & 3.0s & 47835 & 11.981\% & 2.0s & 49133 & 15.020\% & 2.0s & 50287 & 17.721\% & 2.0s & 48805 & 14.252\% & 4.0s & 49363 & 15.558\% & 4.0s \\

X-n573-k30 & 50673 & 60884 & 20.151\% & 2.0s & 62033 & 22.418\% & 3.0s & 59670 & 17.755\% & 3.0s & 57388 & 13.252\% & 2.0s & 56048 & 10.607\% & 2.0s & 55353 & 9.236\% & 2.0s & 54322 & 7.201\% & 4.0s & 55058 & 8.654\% & 4.0s \\

X-n586-k159 & 190316 & 226245 & 18.879\% & 3.0s & 212545 & 11.680\% & 4.0s & 209373 & 10.013\% & 4.0s & 210049 & 10.369\% & 3.0s & 205654 & 8.059\% & 3.0s & 204649 & 7.531\% & 3.0s & 204194 & 7.292\% & 5.0s & 204848 & 7.636\% & 5.0s \\

X-n599-k92 & 108451 & 131035 & 20.824\% & 3.0s & 126654 & 16.785\% & 4.0s & 118761 & 9.507\% & 3.0s & 120022 & 10.669\% & 3.0s & 116840 & 7.735\% & 3.0s & 117784 & 8.606\% & 3.0s & 117681 & 8.511\% & 4.0s & 116938 & 7.826\% & 4.0s \\

X-n613-k62 & 59535 & 77555 & 30.268\% & 3.0s & 73633 & 23.680\% & 3.0s & 67477 & 13.340\% & 3.0s & 66818 & 12.233\% & 2.0s & 67545 & 13.454\% & 3.0s & 69069 & 16.014\% & 3.0s & 67832 & 13.936\% & 4.0s & 67730 & 13.765\% & 4.0s \\

X-n627-k43 & 62164 & 76776 & 23.506\% & 3.0s & 70744 & 13.802\% & 3.0s & 68747 & 10.590\% & 4.0s & 69716 & 12.149\% & 3.0s & 67523 & 8.621\% & 3.0s & 69361 & 11.577\% & 3.0s & 68036 & 9.446\% & 4.0s & 67896 & 9.221\% & 4.0s \\

X-n641-k35 & 63684 & 83138 & 30.548\% & 3.0s & 71986 & 13.036\% & 4.0s & 70691 & 11.003\% & 4.0s & 71120 & 11.676\% & 3.0s & 70631 & 10.909\% & 3.0s & 73624 & 15.608\% & 3.0s & 70687 & 10.996\% & 4.0s & 71974 & 13.017\% & 4.0s \\

X-n655-k131 & 106780 & 120771 & 13.103\% & 3.0s & 118758 & 11.217\% & 4.0s & 119665 & 12.067\% & 4.0s & 117339 & 9.889\% & 3.0s & 112289 & 5.159\% & 3.0s & 110657 & 3.631\% & 4.0s & 111563 & 4.479\% & 5.0s & 110267 & 3.266\% & 5.0s \\

X-n670-k130 & 146332 & 183183 & 25.183\% & 3.0s & 168210 & 14.951\% & 4.0s & 180539 & 23.376\% & 4.0s & 166596 & 13.848\% & 3.0s & 168829 & 15.374\% & 3.0s & 161571 & 10.414\% & 3.0s & 167248 & 14.294\% & 5.0s & 163051 & 11.425\% & 5.0s \\

X-n685-k75 & 68205 & 92701 & 35.915\% & 3.0s & 82607 & 21.116\% & 4.0s & 78039 & 14.418\% & 4.0s & 77265 & 13.283\% & 3.0s & 77890 & 14.200\% & 3.0s & 78473 & 15.055\% & 4.0s & 77618 & 13.801\% & 5.0s & 77132 & 13.088\% & 5.0s \\

X-n701-k44 & 81923 & 92723 & 13.183\% & 3.0s & 89704 & 9.498\% & 4.0s & 89743 & 9.546\% & 4.0s & 90006 & 9.867\% & 3.0s & 90580 & 10.567\% & 3.0s & 92198 & 12.542\% & 3.0s & 90359 & 10.297\% & 5.0s & 90703 & 10.717\% & 5.0s \\

X-n716-k35 & 43373 & 59383 & 36.912\% & 3.0s & 52170 & 20.282\% & 4.0s & 49166 & 13.356\% & 4.0s & 49524 & 14.182\% & 3.0s & 49480 & 14.080\% & 3.0s & 50605 & 16.674\% & 3.0s & 49005 & 12.985\% & 5.0s & 49405 & 13.907\% & 5.0s \\

X-n733-k159 & 136187 & 175848 & 29.122\% & 4.0s & 156268 & 14.745\% & 5.0s & 158156 & 16.131\% & 5.0s & 154339 & 13.329\% & 4.0s & 148581 & 9.101\% & 4.0s & 146080 & 7.264\% & 4.0s & 149852 & 10.034\% & 6.0s & 147334 & 8.185\% & 6.0s \\

X-n749-k98 & 77269 & 102208 & 32.276\% & 4.0s & 92403 & 19.586\% & 5.0s & 88483 & 14.513\% & 5.0s & 87621 & 13.397\% & 4.0s & 85046 & 10.065\% & 4.0s & 85325 & 10.426\% & 4.0s & 85594 & 10.774\% & 6.0s & 84712 & 9.633\% & 6.0s \\

X-n766-k71 & 114417 & 132968 & 16.213\% & 4.0s & 130101 & 13.708\% & 5.0s & 133549 & 16.721\% & 6.0s & 126445 & 10.512\% & 4.0s & 129866 & 13.502\% & 4.0s & 127752 & 11.655\% & 4.0s & 126865 & 10.880\% & 6.0s & 126387 & 10.462\% & 7.0s \\

X-n783-k48 & 72386 & 108577 & 49.997\% & 4.0s & 96432 & 33.219\% & 5.0s & 82299 & 13.695\% & 5.0s & 82041 & 13.338\% & 4.0s & 82839 & 14.441\% & 4.0s & 87562 & 20.965\% & 5.0s & 82324 & 13.729\% & 6.0s & 83864 & 15.857\% & 6.0s \\

X-n801-k40 & 73311 & 92125 & 25.663\% & 4.0s & 87187 & 18.928\% & 6.0s & 89100 & 21.537\% & 6.0s & 88259 & 20.390\% & 5.0s & 86121 & 17.474\% & 4.0s & 94076 & 28.325\% & 4.0s & 85696 & 16.894\% & 6.0s & 89478 & 22.053\% & 6.0s \\

X-n819-k171 & 158121 & 192102 & 21.491\% & 5.0s & 178856 & 13.113\% & 7.0s & 175286 & 10.856\% & 6.0s & 177119 & 12.015\% & 5.0s & 174446 & 10.324\% & 5.0s & 172387 & 9.022\% & 5.0s & 171520 & 8.474\% & 8.0s & 171676 & 8.573\% & 7.0s \\

X-n837-k142 & 193737 & 231002 & 19.235\% & 5.0s & 230226 & 18.834\% & 7.0s & 213765 & 10.338\% & 7.0s & 215009 & 10.980\% & 5.0s & 208669 & 7.707\% & 5.0s & 209540 & 8.157\% & 5.0s & 208667 & 7.706\% & 8.0s & 209031 & 7.894\% & 7.0s \\

X-n856-k95 & 88965 & 117243 & 31.786\% & 5.0s & 105763 & 18.882\% & 6.0s & 109164 & 22.704\% & 7.0s & 99273 & 11.587\% & 5.0s & 98164 & 10.340\% & 5.0s & 102312 & 15.003\% & 5.0s & 99233 & 11.542\% & 7.0s & 98914 & 11.183\% & 7.0s \\

X-n876-k59 & 99299 & 114212 & 15.018\% & 5.0s & 114175 & 14.981\% & 7.0s & 110476 & 11.256\% & 7.0s & 112919 & 13.716\% & 5.0s & 107477 & 8.236\% & 5.0s & 109693 & 10.467\% & 5.0s & 107589 & 8.349\% & 7.0s & 110843 & 11.625\% & 7.0s \\

X-n895-k37 & 53860 & 106062 & 96.922\% & 6.0s & 70363 & 30.641\% & 6.0s & 64648 & 20.030\% & 6.0s & 64343 & 19.463\% & 5.0s & 64225 & 19.244\% & 5.0s & 73280 & 36.056\% & 6.0s & 62460 & 15.967\% & 7.0s & 67830 & 25.938\% & 8.0s \\

X-n916-k207 & 329179 & 387367 & 17.677\% & 7.0s & 374899 & 13.889\% & 8.0s & 361709 & 9.882\% & 9.0s & 360505 & 9.516\% & 7.0s & 353039 & 7.248\% & 7.0s & 351887 & 6.898\% & 7.0s & 353222 & 7.304\% & 9.0s & 352488 & 7.081\% & 10.0s \\

X-n936-k151 & 132715 & 200816 & 51.314\% & 7.0s & 161700 & 21.840\% & 8.0s & 182393 & 37.432\% & 8.0s & 158680 & 19.564\% & 7.0s & 162903 & 22.746\% & 7.0s & 154847 & 16.676\% & 7.0s & 157310 & 18.532\% & 9.0s & 155618 & 17.257\% & 9.0s \\

X-n957-k87 & 85465 & 126220 & 47.686\% & 7.0s & 124190 & 45.311\% & 8.0s & 106292 & 24.369\% & 8.0s & 104024 & 21.715\% & 7.0s & 103089 & 20.621\% & 7.0s & 108664 & 27.144\% & 7.0s & 103134 & 20.674\% & 9.0s & 106903 & 25.084\% & 9.0s \\

X-n979-k58 & 118976 & 138987 & 16.819\% & 7.0s & 132651 & 11.494\% & 9.0s & 133186 & 11.944\% & 8.0s & 133188 & 11.945\% & 8.0s & 129633 & 8.957\% & 7.0s & 133201 & 11.956\% & 7.0s & 129535 & 8.875\% & 9.0s & 132728 & 11.559\% & 9.0s \\

X-n1001-k43 & 72355 & 132976 & 83.783\% & 7.0s & 89175 & 23.246\% & 9.0s & 85919 & 18.746\% & 8.0s & 84377 & 16.615\% & 7.0s & 85852 & 18.654\% & 7.0s & 92974 & 28.497\% & 7.0s & 84390 & 16.633\% & 9.0s & 93476 & 29.191\% & 10.0s \\

\midrule

\multicolumn{2}{c|}{Avg. Gap ($501 < N \leq 1001$)} & \multicolumn{3}{c}{30.190\%} & \multicolumn{3}{c}{18.918\%} & \multicolumn{3}{c}{14.994\%} & \multicolumn{3}{c}{13.188\%} & \multicolumn{3}{c}{12.253\%} & \multicolumn{3}{c}{14.199\%} & \multicolumn{3}{c}{\textbf{11.509\%}} & \multicolumn{3}{c}{12.814\%} \\

\midrule

\multicolumn{2}{c|}{Avg. Gap} & \multicolumn{3}{c}{15.863\%} & \multicolumn{3}{c}{11.693\%} & \multicolumn{3}{c}{10.253\%} & \multicolumn{3}{c}{9.108\%} & \multicolumn{3}{c}{8.428\%} & \multicolumn{3}{c}{9.230\%} & \multicolumn{3}{c}{\textbf{7.973\%}} & \multicolumn{3}{c}{8.441\%} \\

\bottomrule
\end{tabular}}
\end{table}

\subsection{Hyperparameter Studies}

Figure~\ref{appendix:fig:lora_rank} presents an analysis of how varying LoRA ranks influence the performance of MoSES(RF) and MoSES(CaDA) under the settings of $N=50$ and $N=100$. We allow the ranks for the frozen modules $\{ B_{i}A_{i} \}_{i=1}^{4}$ and the trainable module $\hat{B}\hat{A}$ to differ, denoted as $r_{\mathrm{frozen}}$ and $r_{\mathrm{free}}$ respectively. In Figure~\ref{appendix:fig:lora_rank}, the blue curves represent experiments where $r_{\mathrm{frozen}}=32$ is fixed while $r_{\mathrm{free}}$ varies from 4 to 32. Conversely, the green curves fix $r_{\mathrm{free}}=32$ while varying $r_{\mathrm{frozen}}$ over the same range. The y-axis indicates the average optimality gap across 16 VRP variants. For MoSES(RF) under the setting of $N=50$, it is observed that reducing the LoRA rank of the trainable module $\hat{B}\hat{A}$ while keeping $r_{\mathrm{frozen}}=32$ incurs less performance degradation compared to reducing the LoRA rank of the frozen modules $\{ B_{i}A_{i} \}_{i=1}^{4}$ with $r_{\mathrm{free}}=32$, as shown in Figure~\ref{appendix:fig:lora_rank_rf_50}. However, under the setting of $N=100$, the performance decline caused by reducing $r_{\mathrm{frozen}}$ with $r_{\mathrm{free}}=32$ is generally smaller than that caused by reducing $r_{\mathrm{free}}$ with $r_{\mathrm{frozen}}=32$ (see Figure~\ref{appendix:fig:lora_rank_rf_100}). This suggests that for smaller-scale problem instances ($N=50$), MoSES(RF) relies more heavily on the frozen modules $\{ B_{i}A_{i} \}_{i=1}^{4}$ inherited from the basis solvers. In contrast, for larger-scale problem instances ($N=100$), the trainable module $\hat{B}\hat{A}$ becomes at least as critical as the frozen modules in contributing to overall performance. Likewise, similar trends are observed for MoSES(CaDA) under the setting of $N=50$,, as illustrated in Figure~\ref{appendix:fig:lora_rank_cada_50}. Under the setting of $N=100$, the performance gap between reducing the LoRA rank $r_{\mathrm{frozen}}$ of the frozen modules (with fixed trainable module rank $r_{\mathrm{free}}$) and reducing the LoRA rank $r_{\mathrm{free}}$ of the trainable module (with fixed frozen module rank $r_{\mathrm{frozen}}$) becomes more pronounced (see Figure~\ref{appendix:fig:lora_rank_cada_100}). This suggests that at the smaller scale ($N=50$), the frozen LoRA experts contribute more significantly to the performance of MoSES(CaDA) than the trainable LoRA expert. In contrast, at the larger scale ($N=100$), the trainable LoRA expert plays a more critical role in driving performance than the frozen LoRA experts. We observe similar performance trends across each individual VRP variant, as shown in Figures~\ref{appendix:fig:lora_rank_rf_taskwise} and~\ref{appendix:fig:lora_rank_cada_taskwise}.

\begin{figure*}[t]
\centering
\subfigure[MoSES(RF), $N=50$]
{\label{appendix:fig:lora_rank_rf_50}\includegraphics[width=0.33\textwidth]{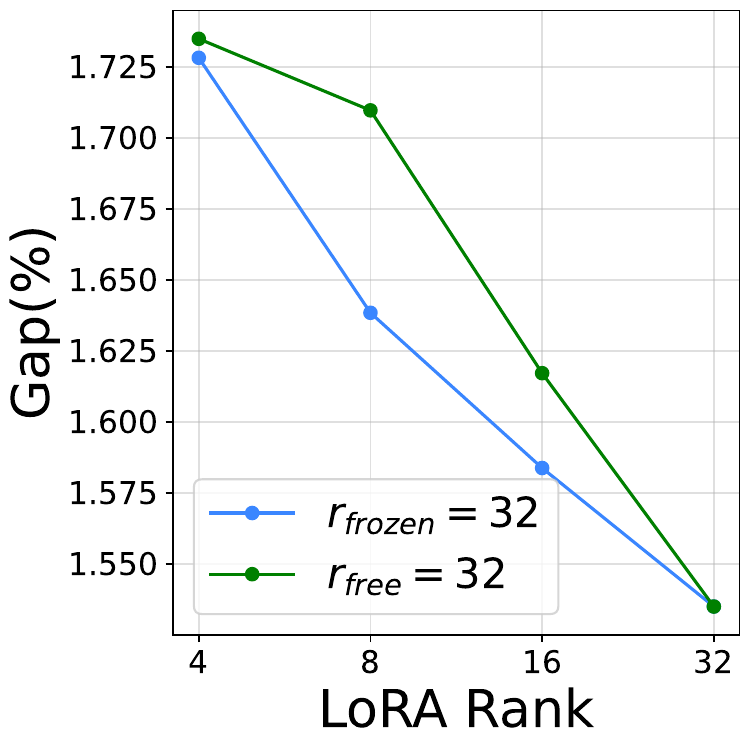}}
\hspace{0.05\textwidth}
\subfigure[MoSES(RF), $N=100$]
{\label{appendix:fig:lora_rank_rf_100}\includegraphics[width=0.33\textwidth]{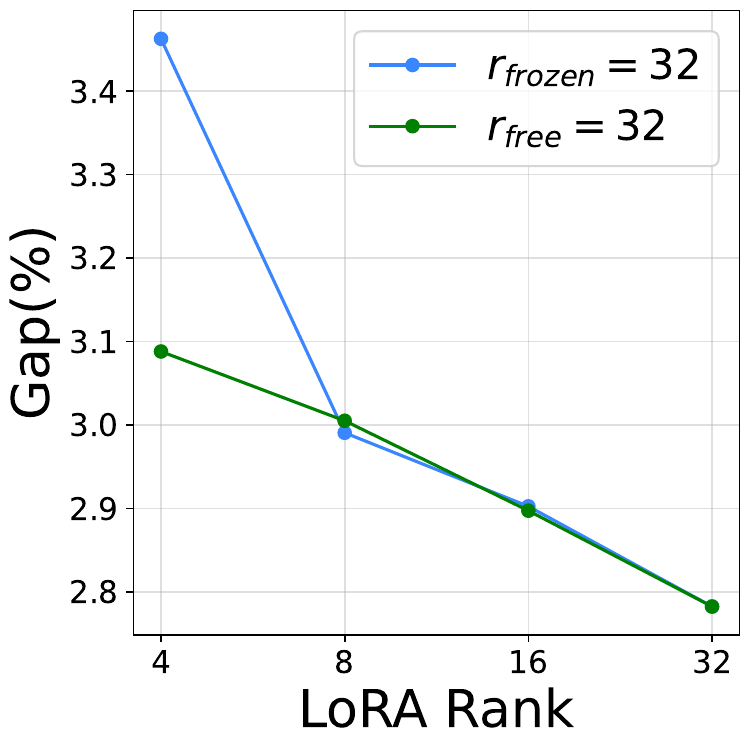}}
\subfigure[MoSES(CaDA), $N=50$]
{\label{appendix:fig:lora_rank_cada_50}\includegraphics[width=0.33\textwidth]{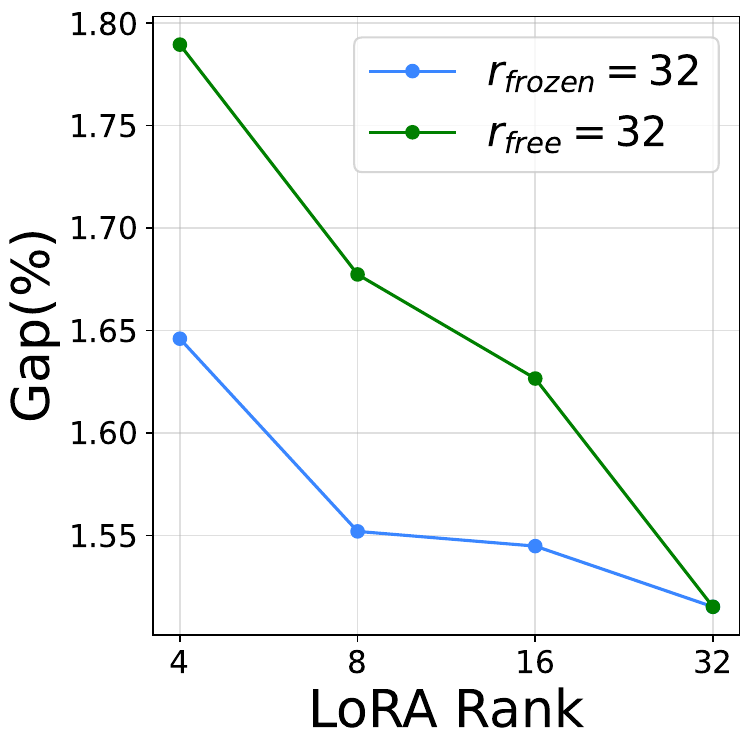}}
\hspace{0.05\textwidth}
\subfigure[MoSES(CaDA), $N=100$]
{\label{appendix:fig:lora_rank_cada_100}\includegraphics[width=0.33\textwidth]{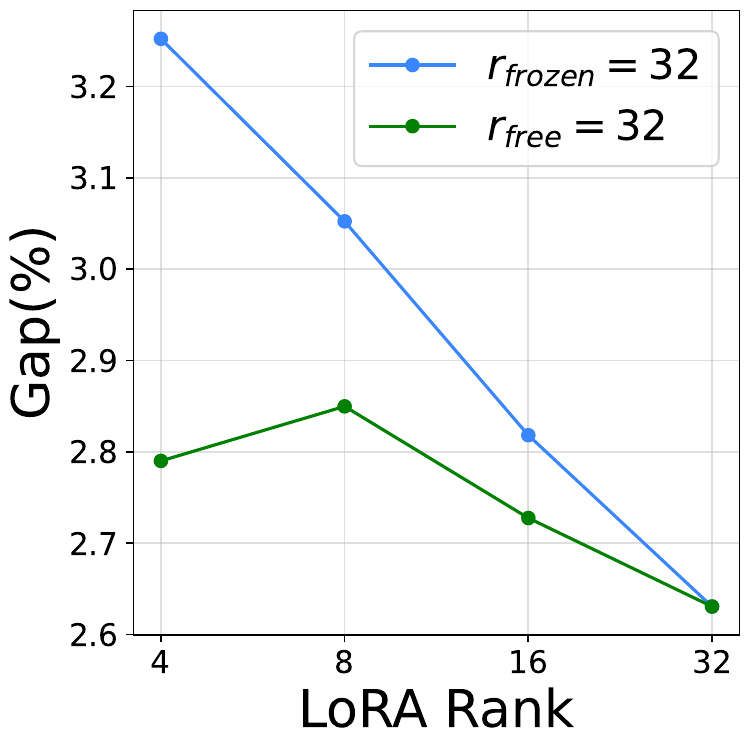}}
\caption{This figure investigates the effect of LoRA ranks on MoSES(RF) and MoSES(CaDA) under both $N=50$ and $N=100$ settings.}
\label{appendix:fig:lora_rank}
\end{figure*}

\subsection{Ablation studies}

We evaluate our proposed Gated-LoRA against standard LoRA that varies $\beta$ from 0.1 to 1.0, as shown in Figure~\ref{appendix:fig:gated_lora}. Since both methods are applied during the fine-tuning phase to produce specialized basis solvers for the basis VRP variants (OVRP, VRPB, VRPL, VRPTW), Figure~\ref{appendix:fig:gated_lora} reports the average optimality gap and average cost across these four variants. When built on the RF-based backbone, the standard LoRA expert achieves its best performance at $\beta=0.8$ in terms of both optimality gap and cost (blue curves in Figures~\ref{appendix:fig:gated_lora_rf_gap} and~\ref{appendix:fig:gated_lora_rf_cost}). However, the noticeable gap between the lowest point of the blue curves and the green dashed line indicates that our proposed Gated-LoRA module outperforms standard LoRA. Similarly, when built on the CaDA-based backbone, the standard LoRA expert performs best at $\beta=0.9$, yet the Gated-LoRA module still achieves superior performance. These experimental results support our hypothesis that problem instances from any basis VRP variant (excluding CVRP) are OOD inputs to the frozen backbone model, resulting in task-misaligned features in the embeddings generated by the backbone.

\begin{figure*}[t]
\centering
\subfigure[RF-based Backbone]
{\label{appendix:fig:gated_lora_rf_gap}\includegraphics[width=0.33\textwidth]{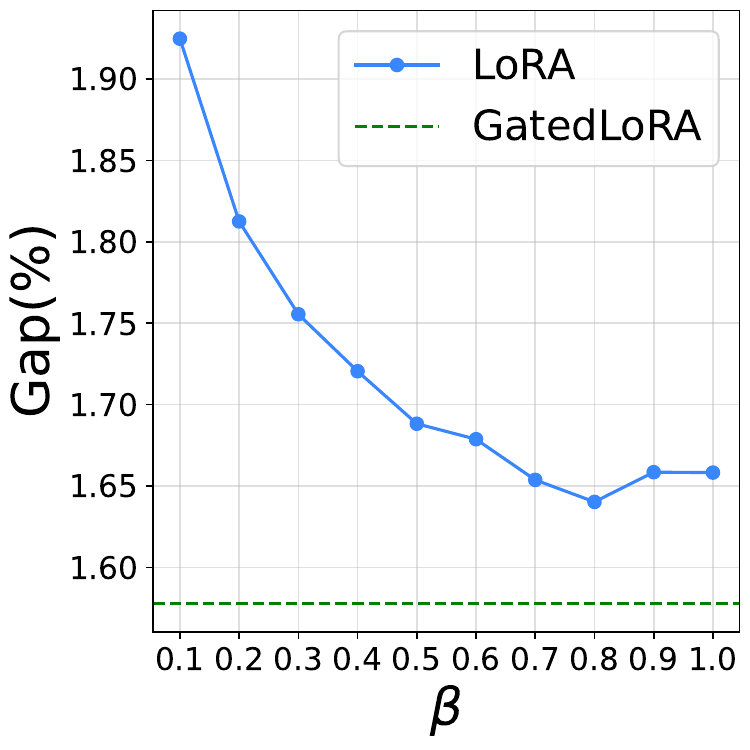}}
\hspace{0.05\textwidth}
\subfigure[RF-based Backbone]
{\label{appendix:fig:gated_lora_rf_cost}\includegraphics[width=0.33\textwidth]{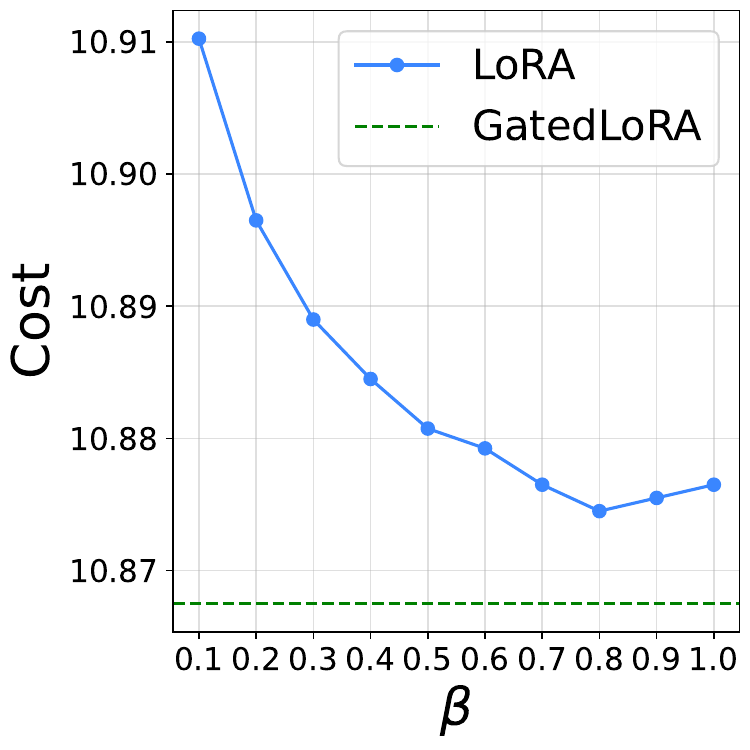}}
\subfigure[CaDA-based Backbone]
{\label{appendix:fig:gated_lora_cada_gap}\includegraphics[width=0.33\textwidth]{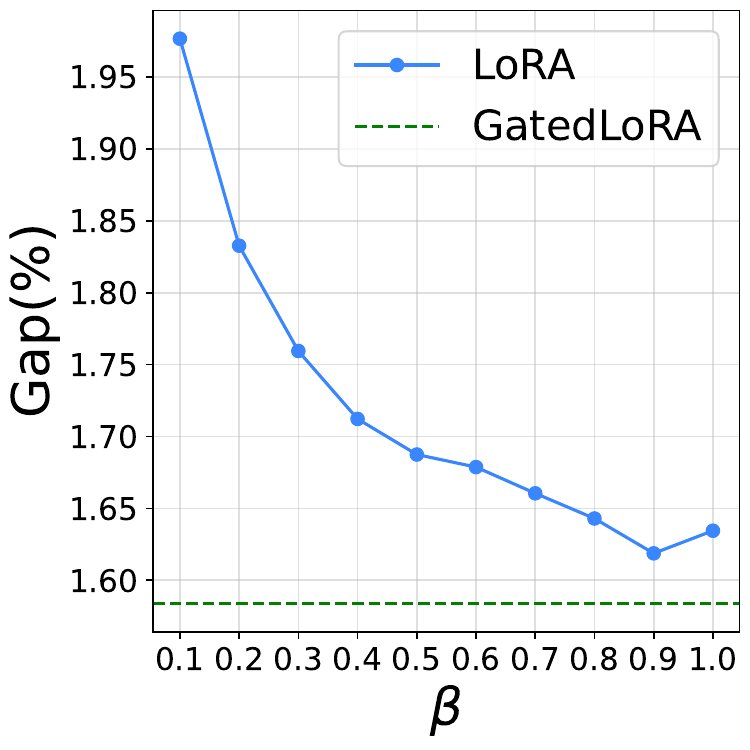}}
\hspace{0.05\textwidth}
\subfigure[CaDA-based Backbone]
{\label{appendix:fig:gated_lora_cada_cost}\includegraphics[width=0.33\textwidth]{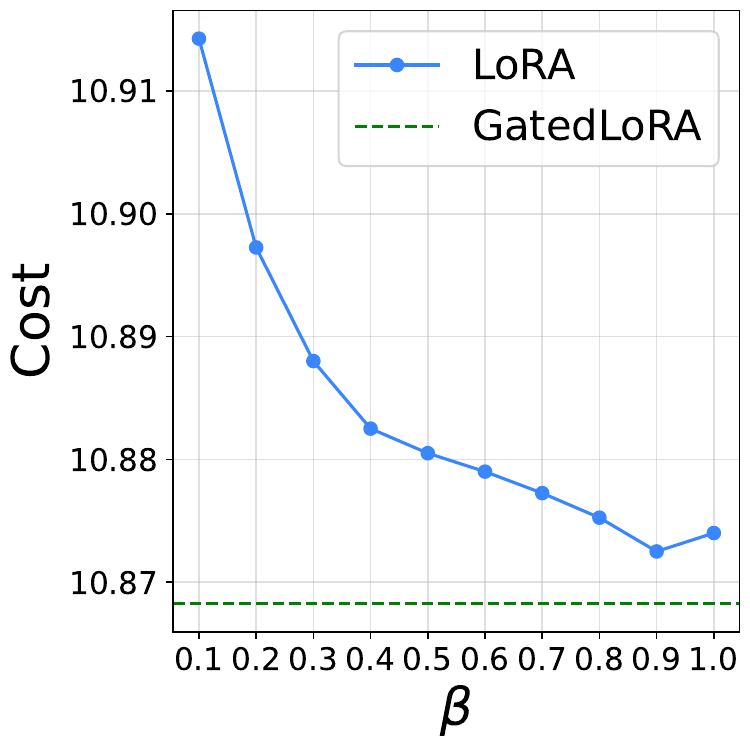}}
\caption{This figure compares the performance of Gated-LoRA with standard LoRA across varying values of $\beta\in(0,1]$, under both RF-based and CaDA-based backbones.}
\label{appendix:fig:gated_lora}
\end{figure*}

Figure~\ref{appendix:fig:act_func} presents a performance comparison of different activation functions used in the adaptive gating mechanism of the mixture of specialized experts, including $\mathrm{softmax}(\cdot)$, $\mathrm{norm\_softplus}(\cdot)$, and $\mathrm{sigmoid}(\cdot)$. In both the $N=50$ and $N=100$ settings, MoSES(RF) demonstrates superior performance when employing the $\mathrm{norm_softplus}(\cdot)$ function, while performance degrades when using $\mathrm{sigmoid}(\cdot)$. This suggests that MoSES(RF) benefits from a convex combination of the pretrained, frozen LoRA experts. In contrast, MoSES(CaDA) achieves its best performance with the $\mathrm{sigmoid}(\cdot)$ function, indicating a preference for selectively reusing or suppressing specific LoRA experts without being constrained by a unit-sum requirement. Figure~\ref{appendix:fig:act_func_taskwise} presents a performance comparison of different activation functions across each VRP variant. As shown in Figure~\ref{appendix:fig:route_method}, the dense routing method yields the best performance for both MoSES(RF) and MoSES(CaDA), while a significant performance degradation is observed when the trainable LoRA module $\hat{B}\hat{A}$ is ablated. This suggests that, to achieve better in-distribution performance, the unified neural solver should leverage the knowledge encoded in the pretrained specialized LoRA experts, which capture the ability to solve basis VRP variants. At the same time, the new knowledge learned through the trainable LoRA expert is also crucial for enhancing performance. Figure~\ref{appendix:fig:route_method_taskwise} presents a comparative result of different routing strategies across each VRP variant.

To investigate the impact of incorporating explicit task descriptors into the gating mechanism on performance, we modified the input of the gating mechanism to explicitly include constraint flags. This modification is implemented on MoSES(RF) under the $N=50$ setting. As shown in Table~\ref{appendix:tab:task_flag}, we observe that including task descriptors in the gating mechanism reduces the optimality gap for some tasks, while increasing it for others. Overall, the performance impact is minimal. We speculate that the gating mechanism is capable of implicitly identifying different VRP variants based solely on the problem instance. As a result, the additional task descriptors do not provide a performance gain.

\begin{table}[t]
\centering
\caption{Impact of incorporating task descriptors into the gating mechanism on performance.}
\label{appendix:tab:task_flag}
\begin{tabular}{lcc}
\toprule

Task & w/o Task Descriptors & w/ Task Descriptors \\
\midrule
CVRP        & 0.900\% & 0.880\% \\
VRPTW       & 1.445\% & 1.452\% \\
OVRP        & 1.892\% & 1.940\% \\
VRPL        & 1.089\% & 1.058\% \\
VRPB        & 2.342\% & 2.339\% \\
OVRPTW      & 0.959\% & 0.950\% \\
VRPBL       & 3.185\% & 3.187\% \\
VRPBLTW     & 1.370\% & 1.382\% \\
VRPBTW      & 1.121\% & 1.109\% \\
VRPLTW      & 1.811\% & 1.820\% \\
OVRPB       & 1.979\% & 1.995\% \\
OVRPBL      & 2.014\% & 2.013\% \\
OVRPBLTW    & 0.791\% & 0.786\% \\
OVRPBTW     & 0.783\% & 0.789\% \\
OVRPL       & 1.917\% & 1.941\% \\
OVRPLTW     & 0.962\% & 0.952\% \\
\midrule
\textbf{Average}     & \textbf{1.535\%} & 1.537\% \\

\bottomrule
\end{tabular}
\end{table}

\begin{table}[t]
\centering
\caption{Performance Comparison of LoRA Expert vs. Linear Expert.}
\label{appendix:tab:linear_expert}
\begin{tabular}{lcc}
\toprule

Task & LoRA Expert & Linear Expert \\
\midrule
CVRP & 0.900\% & 0.878\% \\
VRPTW & 1.445\% & 1.427\% \\
OVRP & 1.892\% & 1.808\% \\
VRPL & 1.089\% & 1.061\% \\
VRPB & 2.342\% & 2.217\% \\
OVRPTW & 0.959\% & 0.938\% \\
VRPBL & 3.185\% & 2.960\% \\
VRPBLTW & 1.370\% & 1.431\% \\
VRPBTW & 1.121\% & 1.145\% \\
VRPLTW & 1.811\% & 1.845\% \\
OVRPB & 1.979\% & 1.764\% \\
OVRPBL & 2.014\% & 1.775\% \\
OVRPBLTW & 0.791\% & 0.783\% \\
OVRPBTW & 0.783\% & 0.778\% \\
OVRPL & 1.917\% & 1.819\% \\
OVRPLTW & 0.962\% & 0.943\% \\
\midrule
\textbf{Average} & 1.535\% & \textbf{1.473\%} \\

\bottomrule
\end{tabular}
\end{table}

\begin{table}[t]
\centering
\caption{Impact of data augmentation on performance.}
\label{appendix:tab:data_aug}
\begin{tabular}{lcc}
\toprule

Method & Avg. Gap & Avg. Time \\
\midrule
RF w/ 8$\times$ dihedral & 2.063\% & 2.0s \\
RF w/ 32$\times$ symmetric & 3.261\% & 6.1s \\
MoSES(RF) w/ 8x dihedral & 1.535\% & 5.8s \\

\bottomrule
\end{tabular}
\end{table}

\begin{table}[t]
\centering
\caption{Impact of poorly trained basis VRP solver on performance.}
\label{appendix:tab:robust}
\begin{tabular}{lcc}
\toprule

 & Tasks w/ OVRP & Tasks w/o OVRP \\
\midrule
Fully Trained OVRP Solver & 1.412\% & 1.658\% \\
Poorly Trained OVRP Solver & 1.524\% & 1.654\% \\

\midrule
\midrule

 & Tasks w/ VRPL & Tasks w/o VRPL \\
\midrule
Fully Trained VRPL Solver & 1.642\% & 1.428\% \\
Poorly Trained VRPL Solver & 1.660\% & 1.453\% \\

\midrule
\midrule

 & Tasks w/ VRPB & Tasks w/o VRPB \\
\midrule
Fully Trained VRPB Solver & 1.698\% & 1.372\% \\
Poorly Trained VRPB Solver & 1.722\% & 1.393\% \\

\midrule
\midrule

 & Tasks w/ VRPTW & Tasks w/o VRPTW \\
\midrule
Fully Trained VRPTW Solver & 1.155\% & 1.915\% \\
Poorly Trained VRPTW Solver & 1.378\% & 1.969\% \\

\bottomrule
\end{tabular}
\end{table}

To investigate linear task-specific experts as an alternative, we designed an experiment where the basis VRP solvers for OVRP, VRPB, VRPL, and VRPTW are trained using linear adapters. These linear adapters are then integrated into the unified solver to enable reuse. We implemented this experiment on the MoSES(RF) model under the setting of $N=50$. The experimental results are presented in Table~\ref{appendix:tab:linear_expert}. We observe that the proposed linear expert contributes to performance improvement. 

Since our method introduces additional time overhead due to the dynamic gating mechanism compared to its direct baseline, we apply the 32$\times$ data augmentation technique proposed in~\cite{bdeir2022attention} to the baseline method to examine whether increasing the inference time of the baseline method to match the runtime of our method can also yield a performance improvement. We conduct this experiment based on the RF model in the $N=50$ setting. As shown in Table~\ref{appendix:tab:data_aug}, we observe that RF with 32$\times$ data augmentation does not result in a performance improvement, despite the increased inference time. Therefore, we conclude that the additional inference time introduced by our method is justified by the performance gains it delivers.

To evaluate the robustness of our method against poorly trained basis solvers, we conducted experiments using MoSES(RF) under the $N=50$ setting. In our method, each basis solver typically employs a LoRA rank of 32. We observed that reducing the LoRA rank to 4 significantly degrades performance, so we used basis solvers with a LoRA rank of 4 to simulate poorly trained solvers. Specifically, we trained four such poorly trained solvers, each corresponding to one of the basis VRP variants. We then replaced one of the original high-performing solvers in MoSES(RF) with a corresponding poor solver and retrained the gating mechanism accordingly. In Table~\ref{appendix:tab:robust}, `Fully Trained OVRP Solver' and `Poorly Trained OVRP Solver' refer to cases where MoSES(RF) uses a fully trained basis solver and a poorly trained basis solver for OVRP, respectively, while the remaining basis VRP solvers remain unchanged. We compare their performance separately on VRP variants with and without the open route constraint. Thus, `Tasks w/ OVRP' and `Tasks w/o OVRP' represent the average optimality gap across 8 VRP variants with and without the open route constraint, respectively. The same notation is also applied to the remaining constraints in Table~\ref{appendix:tab:robust}. From the results, we observe that the optimality gap is more significantly affected for tasks that include the corresponding constraints when the related basis solver is poorly trained.

\subsection{Visualizing Adaptive Gating Mechanisms}

In Figure~\ref{appendix:fig:gates_visual}, we present a visualization of the weights assigned to five basis solvers, each corresponding to a specific constraint: Capacity (C), Open Route (O), Distance Limit (L), Backhaul (B), and Time Window (TW), for both MoSES(RF) and MoSES(CaDA) under the setting $N=100$. Each weight is calculated by averaging across layers, time steps, and problem instances, based on evaluations on 1,000 problem instances per model.

We observe that, since each VRP variant is derived from CVRP, the CVRP solver consistently receives a dominant weight allocation in both MoSES(RF) and MoSES(CaDA). Furthermore, when a VRP variant incorporates the Time Window constraint, the corresponding VRPTW solver receives significantly higher weights compared to other solvers, with the exception of the CVRP solver. Overall, basis solvers associated with the constraints present in the current VRP variant tend to be preferred. Additionally, solvers corresponding to irrelevant basis variants may still receive non-negligible weights, likely due to partial similarities among the basis VRP variants. Due to the use of the $\mathrm{sigmoid}(\cdot)$ activation function in MoSES(CaDA), the model tends to assign weights with greater magnitudes, thereby capturing more informative signals from the basis solvers.

\begin{figure*}[h]
\centering
\subfigure[$N=50$]
{\label{appendix:fig:act_50}\includegraphics[width=0.45\textwidth]{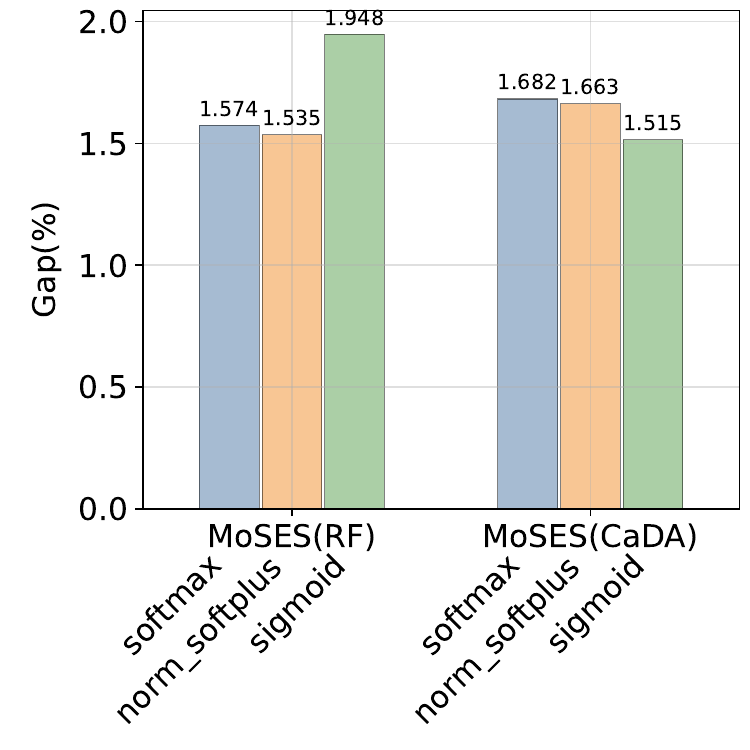}}
\hspace{0.05\textwidth}
\subfigure[$N=100$]
{\label{appendix:fig:act_100}\includegraphics[width=0.45\textwidth]{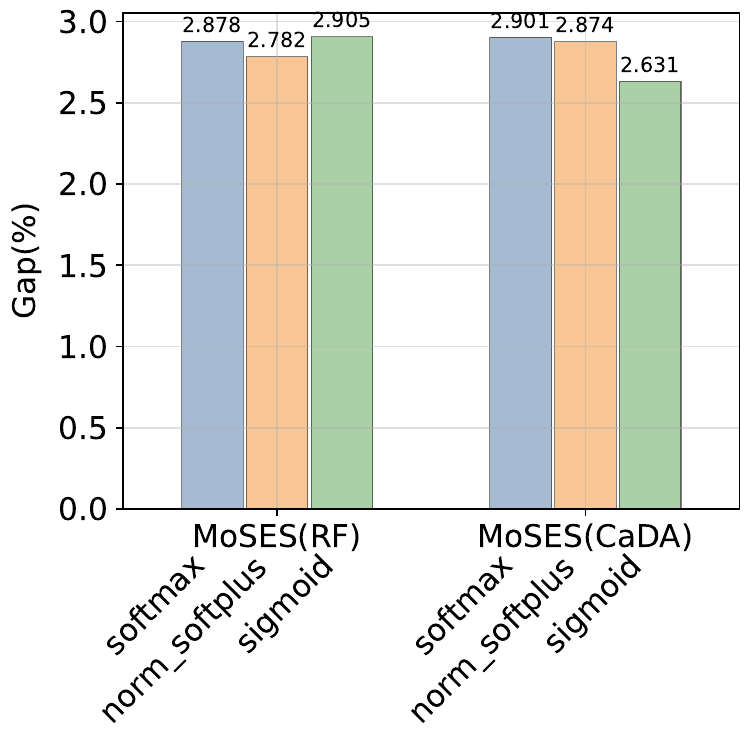}}
\caption{This figure compares the performance of activation functions used in the adaptive gating mechanism of the mixture of specialized experts for both MoSES(RF) and MoSES(CaDA) under the settings of $N=50$ and $N=100$.}
\label{appendix:fig:act_func}
\end{figure*}

\begin{figure*}[h]
\centering
\subfigure[$N=50$]
{\label{appendix:fig:route_50}\includegraphics[width=0.45\textwidth]{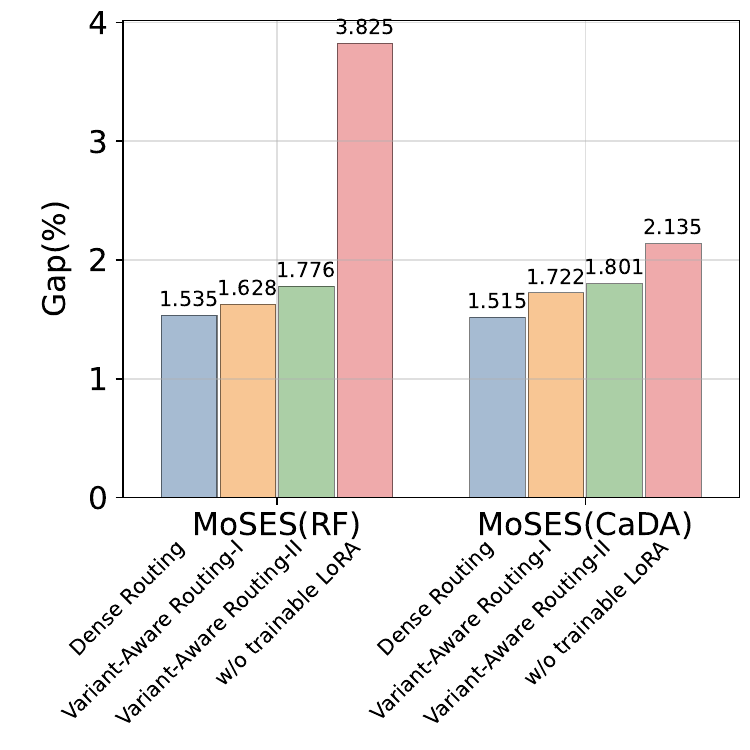}}
\hspace{0.05\textwidth}
\subfigure[$N=100$]
{\label{appendix:fig:route_100}\includegraphics[width=0.45\textwidth]{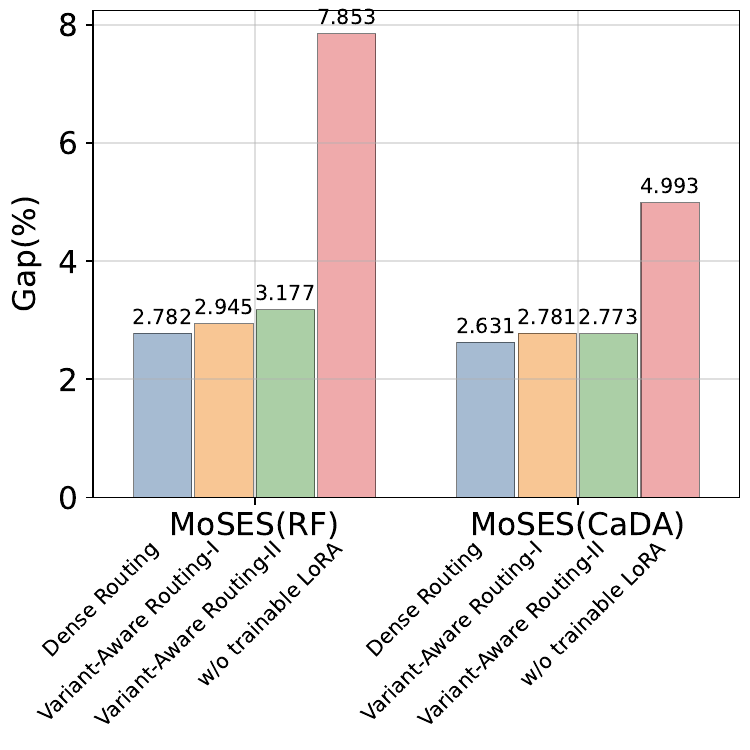}}
\caption{This figure compares the performance of routing strategies used in the adaptive gating mechanism of the mixture of specialized experts for both MoSES(RF) and MoSES(CaDA) under the settings of $N=50$ and $N=100$.}
\label{appendix:fig:route_method}
\end{figure*}

\begin{figure*}[t]
\centering
\subfigure[MoSES(RF), $N=50$]
{\label{appendix:fig:lora_rank_rf_50_taskwise}\includegraphics[width=0.9\textwidth]{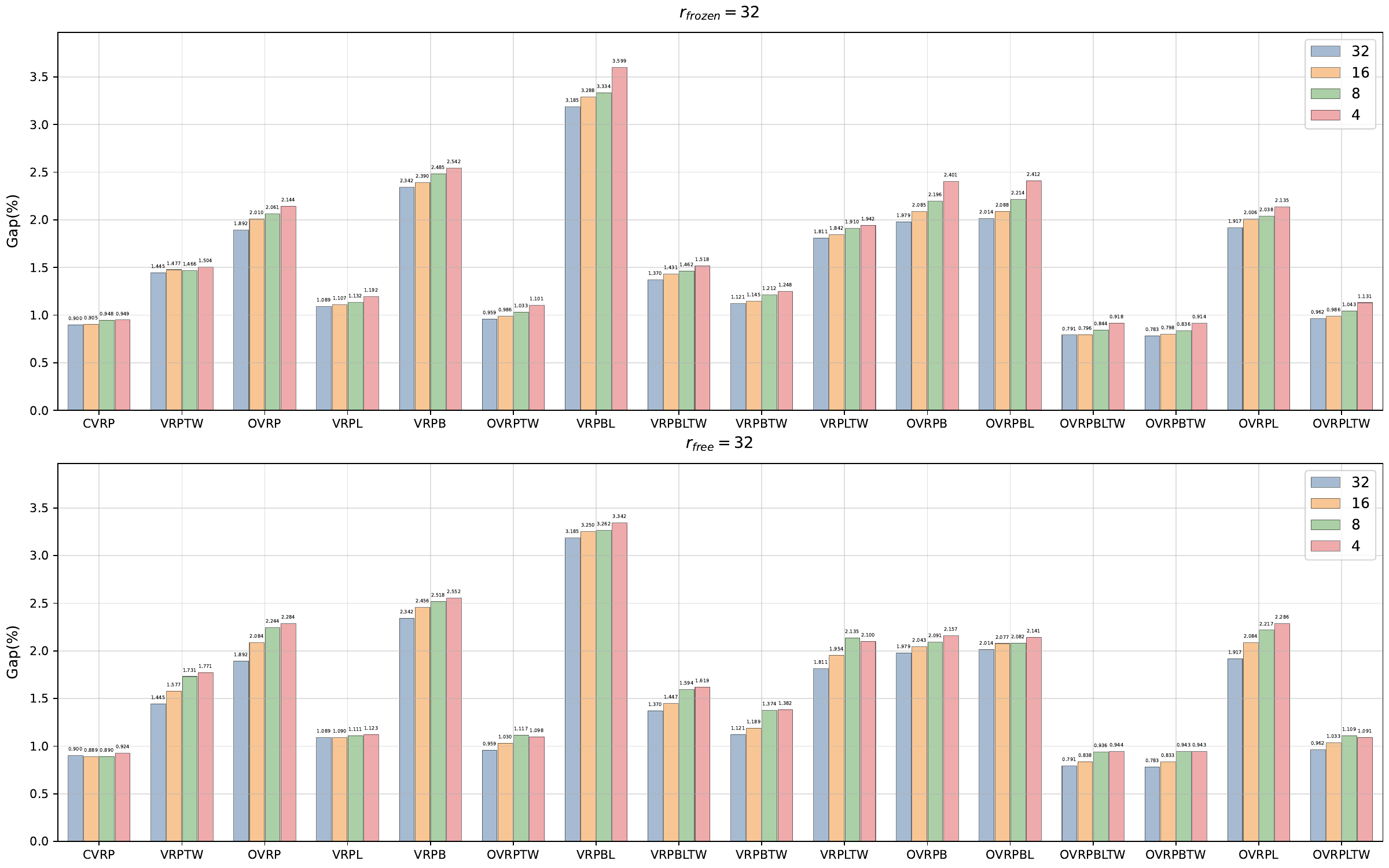}}
\subfigure[MoSES(RF), $N=100$]
{\label{appendix:fig:lora_rank_rf_100_taskwise}\includegraphics[width=0.9\textwidth]{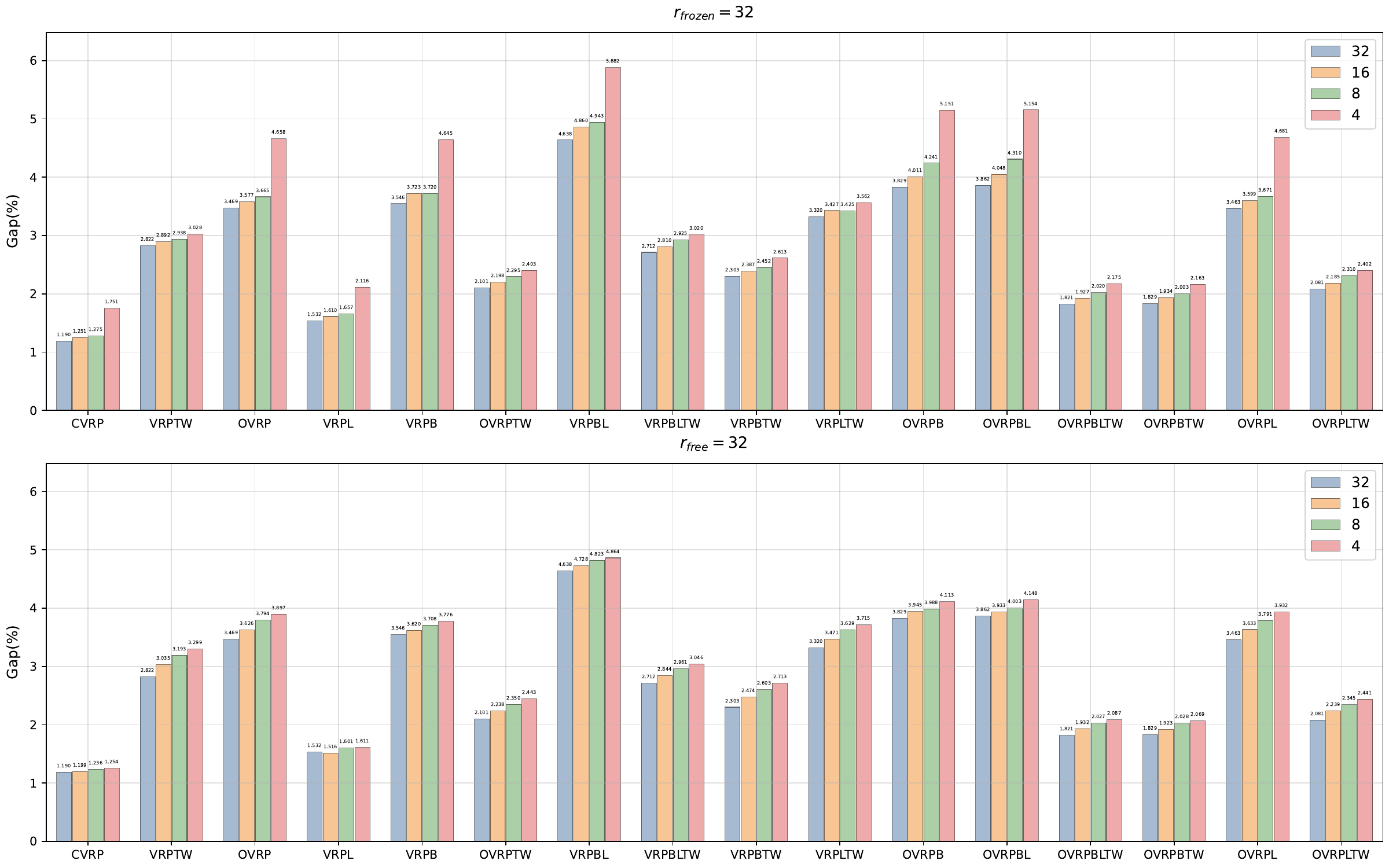}}
\caption{This figure illustrates the performance trends of MoSES(RF) across each VRP variant under both $N=50$ and $N=100$ settings, as either $r_{\mathrm{frozen}}$ or $r_{\mathrm{free}}$ is varied while keeping the other fixed.}
\label{appendix:fig:lora_rank_rf_taskwise}
\end{figure*}

\begin{figure*}[t]
\centering
\subfigure[MoSES(CaDA), $N=50$]
{\label{appendix:fig:lora_rank_cada_50_taskwise}\includegraphics[width=0.9\textwidth]{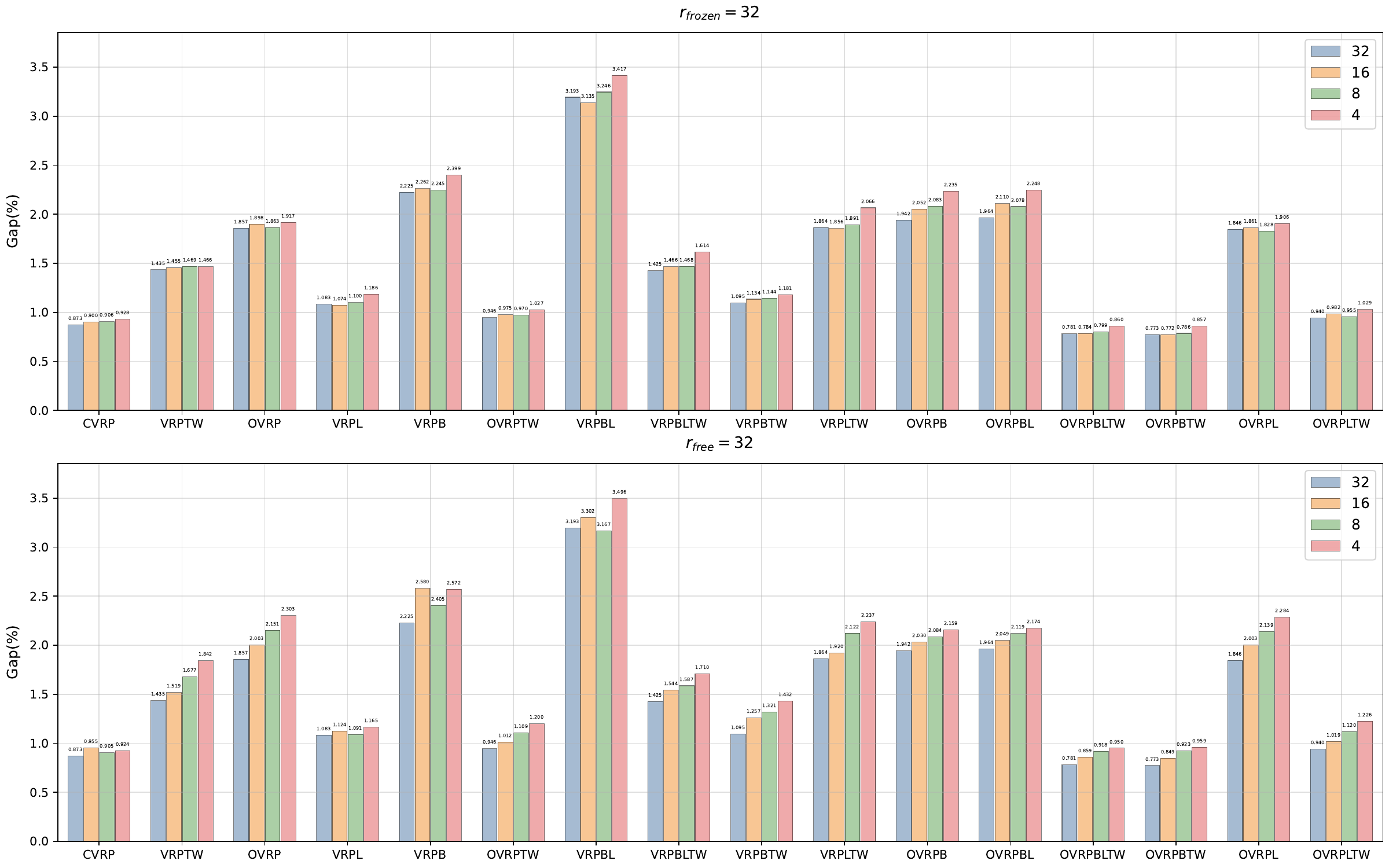}}
\subfigure[MoSES(CaDA), $N=100$]
{\label{appendix:fig:lora_rank_cada_100_taskwise}\includegraphics[width=0.9\textwidth]{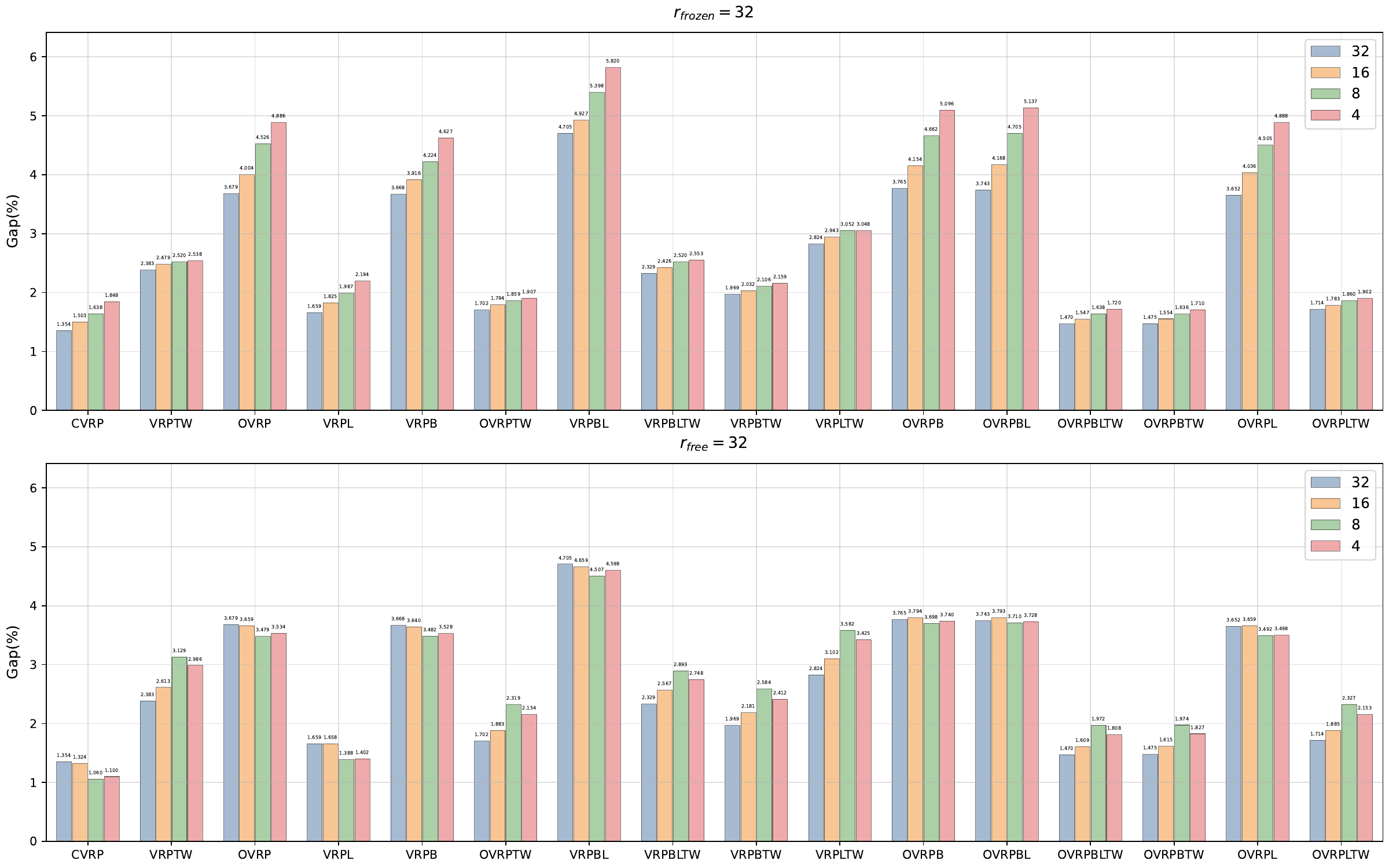}}
\caption{This figure illustrates the performance trends of MoSES(CaDA) across each VRP variant under both $N=50$ and $N=100$ settings, as either $r_{\mathrm{frozen}}$ or $r_{\mathrm{free}}$ is varied while the other is held constant.}
\label{appendix:fig:lora_rank_cada_taskwise}
\end{figure*}

\begin{figure*}[t]
\centering
\subfigure[$N=50$]
{\label{appendix:fig:act_50_taskwise}\includegraphics[width=0.9\textwidth]{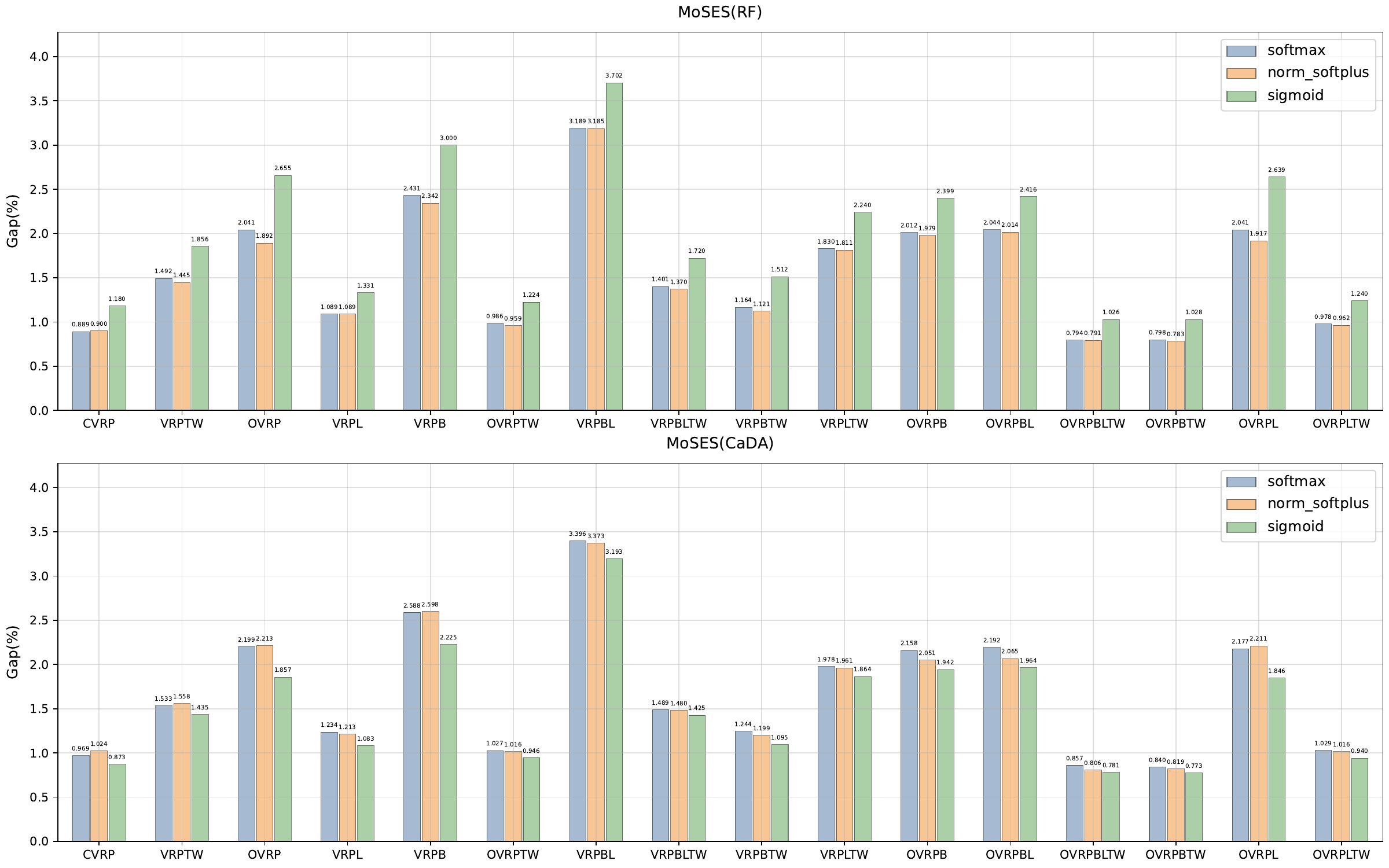}}
\subfigure[$N=100$]
{\label{appendix:fig:act_100_taskwise}\includegraphics[width=0.9\textwidth]{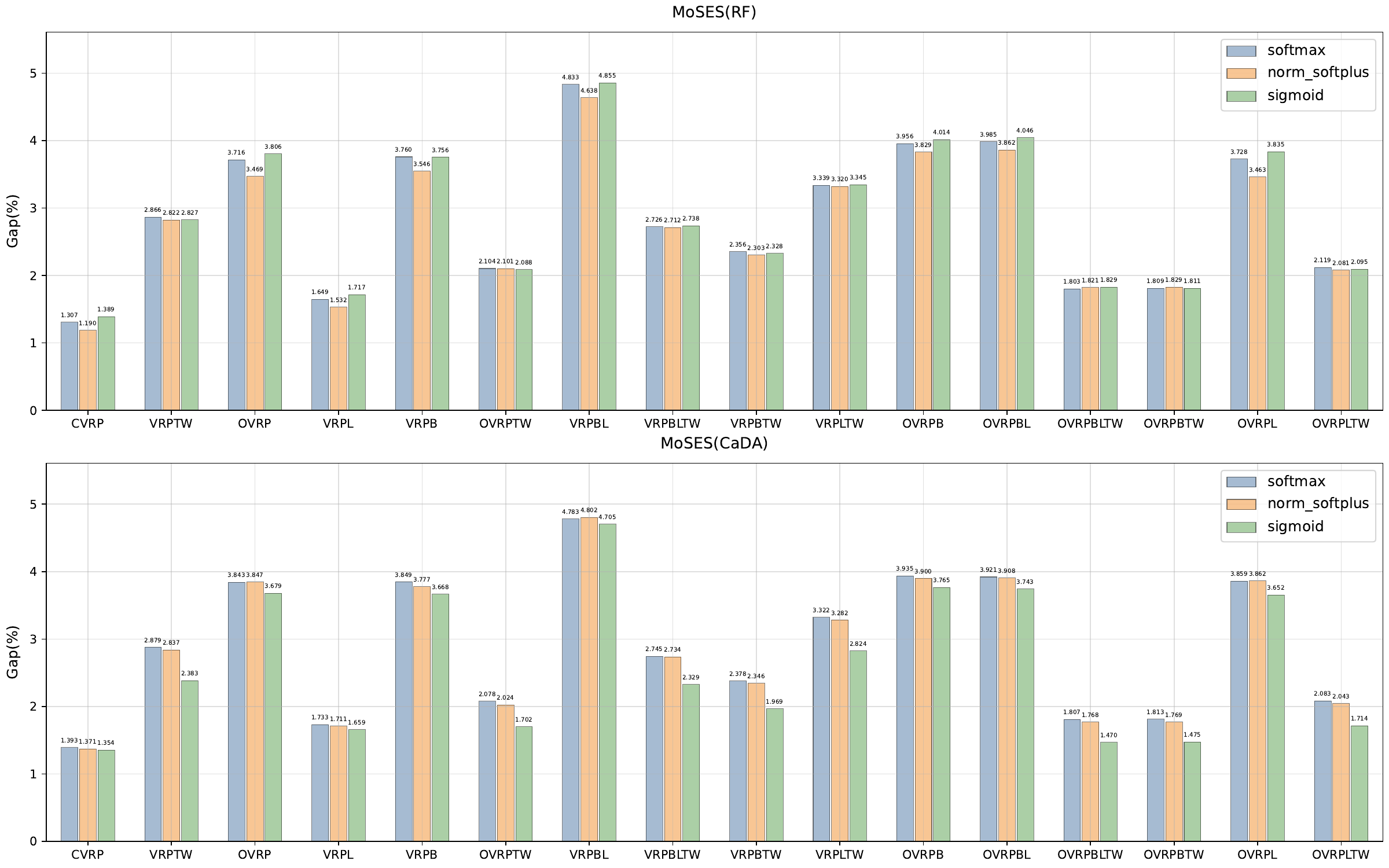}}
\caption{This figure illustrates the performance comparisons among different activation functions used in the adaptive gating mechanism of both MoSES(RF) and MoSES(CaDA) across various VRP variants under the $N=50$ and $N=100$ settings.}
\label{appendix:fig:act_func_taskwise}
\end{figure*}

\begin{figure*}[t]
\centering
\subfigure[$N=50$]
{\label{appendix:fig:route_50_taskwise}\includegraphics[width=0.9\textwidth]{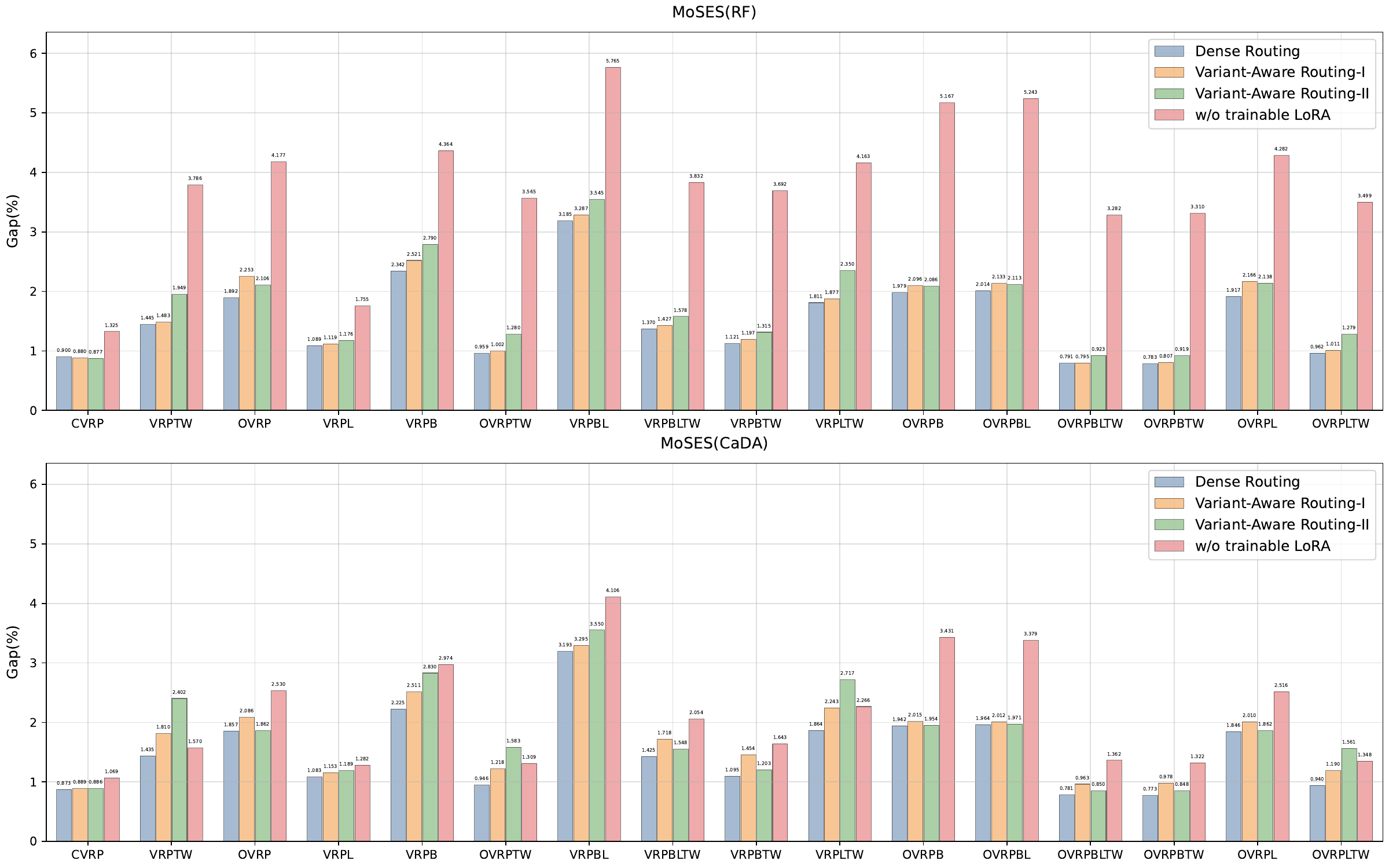}}
\subfigure[$N=100$]
{\label{appendix:fig:route_100_taskwise}\includegraphics[width=0.9\textwidth]{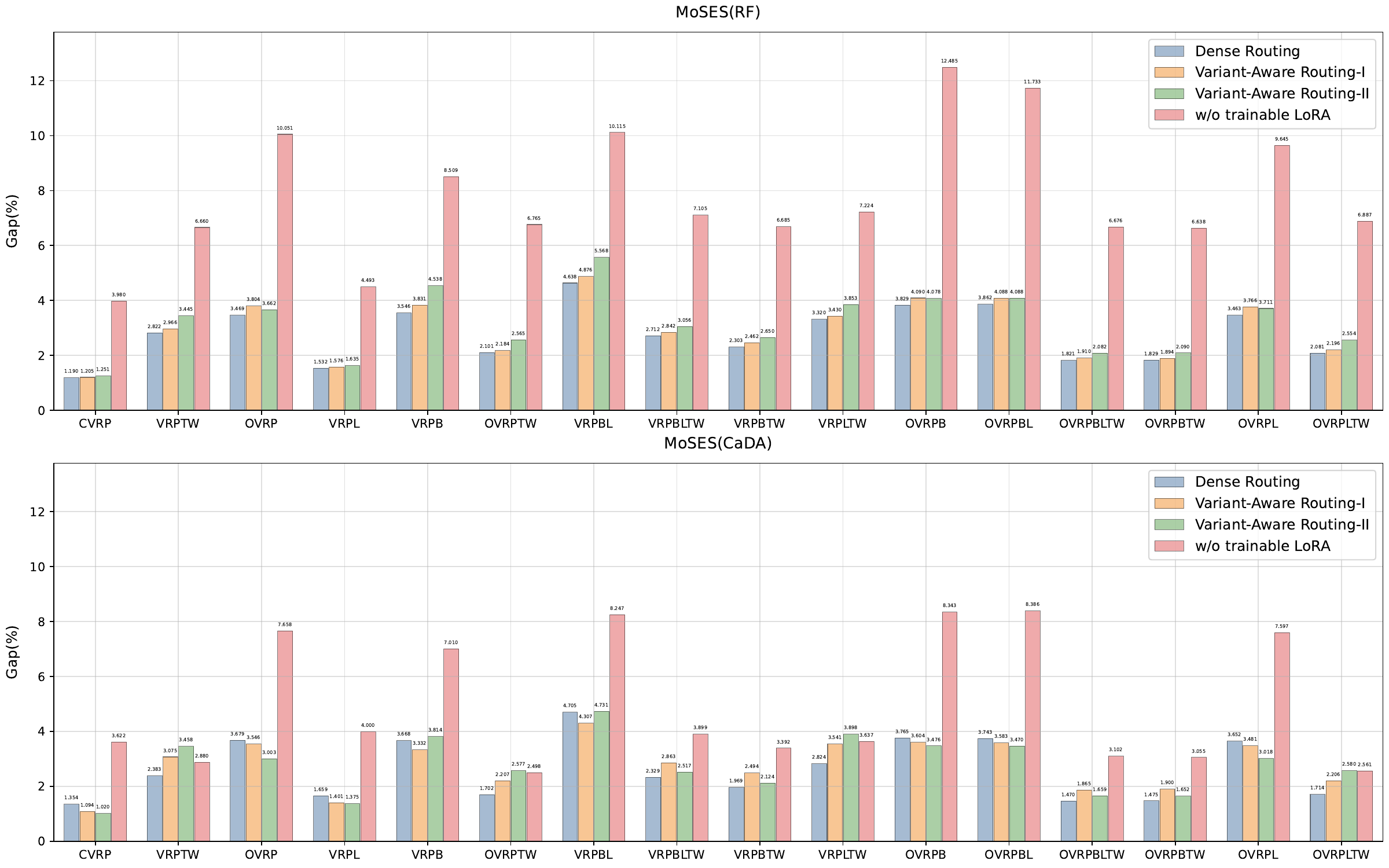}}
\caption{This figure illustrates the performance comparisons among different routing strategies used in the adaptive gating mechanism of both MoSES(RF) and MoSES(CaDA) across various VRP variants under the $N=50$ and $N=100$ settings.}
\label{appendix:fig:route_method_taskwise}
\end{figure*}

\begin{figure}[t]
    \centering
    \includegraphics[width=0.9\textwidth]{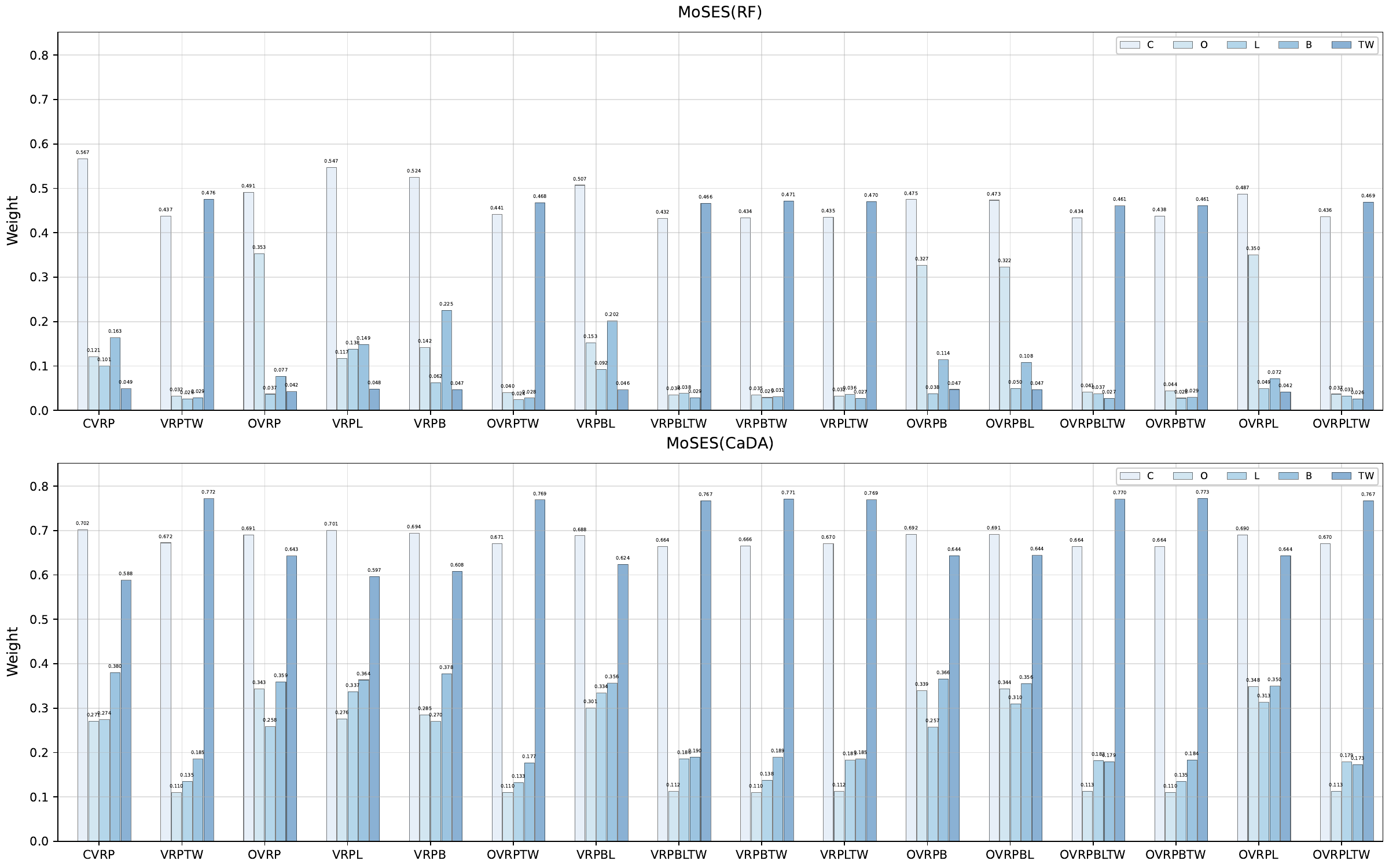}
    \caption{This figure illustrates the behavior of the adaptive gating mechanisms in both MoSES(RF) and MoSES(CaDA) under the setting $N=100$.}
    \label{appendix:fig:gates_visual}
\end{figure}

\clearpage

\section{Proofs}

\begin{theorem}
\label{appendix:thm:1}
The optimal unified policy $\pi^{\ast}$ and the $i$-th optimal basis policy $\pi^{(i)\ast}$ coincide in their value functions for each state $s_{t}$ associated with the $i$-th basis task: $V^{\pi^{\ast}, \mathcal{P}}(s_{t}) = V^{\pi^{(i)\ast},\mathcal{P}}(s_{t})$. Furthermore, if both the optimal unified policy and the optimal basis policy are unique, then for each state-action pair $(s_{t}, a_{t})$ corresponding to the $i$-th basis task, it holds that $\pi^{\ast}(a_{t}|s_{t}) = \pi^{(i)\ast}(a_{t}|s_{t})$. 
\end{theorem}

\begin{proof}
\label{appendix:proof:1}

By definition, the value function $V^{\pi^{\ast}, \mathcal{P}}$ induced by the optimal policy $\pi^{\ast}$ is greater than or equal to the value function $V^{\pi, \mathcal{P}}$ of any unified policy $\pi$ for all feasible states $s_{t}$ at each time step $t$. Formally, $\forall \pi, s_{t} \in \mathcal{S}_{t}, t\in\{0, \ldots, T\}$, it holds that $V^{\pi^{\ast}, \mathcal{P}}(s_{t}) \geq V^{\pi, \mathcal{P}}(s_{t})$. Likewise, the $i$-the optimal basis policy $\pi^{(i)\ast}$ $(i \geq 0)$ maximizes the value function $V^{\pi^{(i)},\mathcal{P}}$, evaluated on states induced by the initial state distribution defined over $\mathcal{S}_{0}^{(0)}$ for $i=0$, or over $\mathcal{S}_{0}^{(0)}\times\mathcal{S}_{0}^{(i)}$ for $i\geq1$, corresponding to the $i$-th basis task at each time step. That is, for all $\pi^{(0)}$, $s_{t}^{(0)} \in \mathcal{S}_{t}^{(0)}$, and $t\in\{0, \ldots, T\}$, we have $V^{\pi^{(0)\ast}, \mathcal{P}}(s_{t}^{(0)}) \geq V^{\pi^{(0)}, \mathcal{P}}(s_{t}^{(0)})$. Similarly, for all $\pi^{(i)}$, $(s_{t}^{(0)}, s_{t}^{(i)}) \in \mathcal{S}_{t}^{(0)}\times\mathcal{S}_{t}^{(i)}$, $t\in\{0, \ldots, T\}$, and $i\geq1$, it holds that $V^{\pi^{(i)\ast}, \mathcal{P}}((s_{t}^{(0)}, s_{t}^{(i)})) \geq V^{\pi^{(i)}, \mathcal{P}}((s_{t}^{(0)}, s_{t}^{(i)}))$.

According to the definitions of the SDMDP framework and the basis task, it is evident that, at each time step, the state space associated with a basis task is a subset of the state space defined for the unified policy. Specifically, at each time step $t$, for all $s_{t}^{(0)} \in \mathcal{S}_{t}^{(0)}$, it holds that $s_{t}^{(0)} \in \mathcal{S}_{t}$. Likewise, for all $(s_{t}^{(0)}, s_{t}^{(i)}) \in \mathcal{S}_{t}^{(0)} \times \mathcal{S}_{t}^{(i)}$, and $i\geq1$ we have $(s_{t}^{(0)}, s_{t}^{(i)}) \in \mathcal{S}_{t}$. As a result, at each time step $t$, for all $s_{t}^{(0)} \in \mathcal{S}_{t}^{(0)}$, it follows that $V^{\pi^{\ast}, \mathcal{P}}(s_{t}^{(0)}) = V^{\pi^{(0)\ast}, \mathcal{P}}(s_{t}^{(0)})$. Similarly, for all $(s_{t}^{(0)}, s_{t}^{(i)}) \in \mathcal{S}_{t}^{(0)}\times\mathcal{S}_{t}^{(i)}$, and $i\geq1$, we have $V^{\pi^{\ast}, \mathcal{P}}((s_{t}^{(0)}, s_{t}^{(i)})) = V^{\pi^{(i)\ast}, \mathcal{P}}((s_{t}^{(0)}, s_{t}^{(i)}))$. Moreover, if the $i$-the basis task ($i>0$) admits a unique optimal policy $\pi^{(i)\ast}$, and the optimal unified policy $\pi^{\ast}$ is also unique, then for each state $s_{t}$ associated with the $i$-th basis task, it holds that $\pi^{\ast}(a_{t}|s_{t}) = \pi^{(i)\ast}(a_{t}|s_{t})$, where $a_{t}\in \mathcal{A} (s_{t})$. 
\end{proof}

\begin{assumption}
\label{appendix:assup:1}
In the SDMDP framework, any state $s$ is composed of $m+1$ conditionally independent basis states, denoted as $s = \{s^{(b_{i})} \}_{i=0}^{m}$. 
Accordingly, we assume that any policy $\pi$ is capable of extracting the basis state embedding $z^{(b_{i})} \in \mathbb{R}^{d}$ for each $s^{(b_{i})}$, where $i=0, \ldots, m$. Under this assumption, we further posit that there exists a deterministic bijective mixture function $f_{\phi}:\mathcal{S}\times\prod_{i=0}^{m}\mathbb{R}^{d}\rightarrow\mathbb{R}^{d}$, parameterized by $\phi$, which maps the basis state embeddings 
$\{z^{(b_{i})} \}_{i=0}^{m}$ 
to the state embedding $z \in \mathbb{R}^{d}$ for the given state $s$, represented as 
$z=f_{\phi}(z^{(b_{0})}, \ldots, z^{(b_{m})};s)$. 
Thus, the policy defined over the action space can be rewritten as
\begin{equation}
\small
\label{appendix:equ:pi_decomp}
    \pi(a|s) = \sum_{z}\pi(a|z)\pi(z|s)=\sum_{z^{(b_{0})}, \dots, z^{(b_{m})}} \pi(a|f_{\phi}(z^{(b_{0})}, \dots, z^{(b_{m})}; s))\prod_{i=0}^{m}\pi(z^{(b_{i})}|s^{(b_{i})})
\end{equation}
\end{assumption}
where $\prod_{i=0}^{m}\pi(z^{(b_{i})}|s^{(b_{i})})=\pi(z^{(b_{0})}, \dots, z^{(b_{m})}|s)$. The second equivalence in Equation~\ref{appendix:equ:pi_decomp} holds because $f_{\phi}$ is assumed to be deterministic.

\begin{assumption}
\label{appendix:assup:2}
For any state $s$, and for any two policies $\pi$ and $\pi^{\prime}$, we assume that if $\forall a\in \mathcal{A}(s), \pi(a|s) = \pi^{\prime}(a|s)$, then $\forall z\in\mathbb{R}^{d}, \pi(z|s) = \pi^{\prime}(z|s)$, and conversely.
\end{assumption}

\begin{theorem}
\label{appendix:thm:2}
Let $J(\pi, \mathcal{P}, \mu)$ and $J(\pi_{f_{\phi}}, \mathcal{P}_{z}, \mu_{z})$ denote the objective functions (expected returns) in SDMDP and LS-SDMDP, respectively. By Theorem~\ref{appendix:thm:1} and Assumptions~\ref{appendix:assup:1}~\ref{appendix:assup:2}, it follows that the values of the objective functions are equal at their respective optimal policies, formally written as $J(\pi^{\ast}, \mathcal{P}, \mu) = J(\pi_{f_{\phi}}^{\ast}, \mathcal{P}_{z}, \mu_{z})$. Moreover, the value functions of SDMDP and LS-SDMDP at their respective optimal policies satisfy the following relationship $V^{\pi^{\ast}, \mathcal{P}}(s) = \mathbb{E}_{z^{(b_{0})} \sim \pi^{(b_{0})\ast}}\cdots \mathbb{E}_{z^{(b_{m})} \sim \pi^{(b_{m})\ast}}V^{\pi_{f_{\phi}}^{\ast}, \mathcal{P}_{z}}(z^{(b_{0})},\ldots,z^{(b_{m})};s)$.
\end{theorem}

\begin{proof}
\label{appendix:proof:2}

By definition, for any given policy $\pi$, the objective function $J(\pi, \mathcal{P}, \mu)$ within the SDMDP framework can be reformulated as shown in Equation~\ref{appendix:equ:sdmdp_obj}. In this reformulation, Equation \textcircled{1} is derived by incorporating Equation~\ref{appendix:equ:pi_decomp}. Subsequently, Equation \textcircled{1} is rearranged to yield Equation \textcircled{2}.

\begin{equation}
\label{appendix:equ:sdmdp_obj}
\begin{split}
J(\pi, \mathcal{P}, \mu) &= \mathbb{E}_{s\sim\mu} \mathbb{E}_{\tau \sim (\pi, \mathcal{P})}[\sum_{t=0}^{T-1} \gamma^{t} r_{t}|s_{0}=s] \\
&= \sum_{s_{0}, a_{0}, \ldots, s_{T}} \mu(s_{0})\prod_{t=0}^{T-1}\pi(a_{t}|s_{t})\mathcal{P}(s_{t+1}|s_{t},a_{t}) \sum_{t=0}^{T-1}\gamma^{t}r_{t} \\
& \stackrel{\text{\textcircled{1}}}{=} \sum_{s_{0}, a_{0}, \ldots, s_{T}} \mu(s_{0}) \prod_{t=0}^{T-1}\sum_{z_{t}^{(b_{0})},\ldots,z_{t}^{(b_{m})}}\pi(a_{t}|f_{\phi}(z_{t}^{(b_{0})},\ldots,z_{t}^{(b_{m})};s_{t})) \\
&\quad\quad\pi(z_{t}^{(b_{0})},\ldots,z_{t}^{(b_{m})}|s_{t})\mathcal{P}(s_{t+1}|s_{t},a_{t}) \sum_{t=0}^{T-1}\gamma^{t}r_{t} \\
& \stackrel{\text{\textcircled{2}}}{=} \sum_{s_{0}, a_{0}, \ldots, s_{T}} \sum_{z_{0}^{(b_{0})},\ldots,z_{0}^{(b_{m})}}\cdots\sum_{z_{T}^{(b_{0})},\ldots,z_{T}^{(b_{m})}} \mu(s_{0})\prod_{t=0}^{T-1}\pi(a_{t}|f_{\phi}(z_{t}^{(b_{0})},\ldots,z_{t}^{(b_{m})};s_{t})) \\
&\quad\quad\pi(z_{t}^{(b_{0})},\ldots,z_{t}^{(b_{m})}|s_{t})\mathcal{P}(s_{t+1}|s_{t},a_{t}) \sum_{t=0}^{T-1}\gamma^{t}r_{t}
\end{split}
\end{equation}

\begin{equation}
\label{appendix:equ:sdmdp_opt_obj}
\begin{split}
J(\pi^{\ast}, \mathcal{P}, \mu) &\stackrel{\text{\textcircled{1}}}{=} \sum_{s_{0}, a_{0}, \ldots, s_{T}} \sum_{z_{0}^{(b_{0})},\ldots,z_{0}^{(b_{m})}}\cdots\sum_{z_{T}^{(b_{0})},\ldots,z_{T}^{(b_{m})}} \mu(s_{0}) \pi^{\ast}(z_{0}^{(b_{0})},\ldots,z_{0}^{(b_{m})}|s_{0}) \\
&\quad\quad\prod_{t=0}^{T-1}\pi^{\ast}(a_{t}|f_{\phi}(z_{t}^{(b_{0})},\ldots,z_{t}^{(b_{m})};s_{t})) \pi^{\ast}(z_{t+1}^{(b_{0})},\ldots,z_{t+1}^{(b_{m})}|s_{t+1})\mathcal{P}(s_{t+1}|s_{t}, a_{t}) \sum_{t=0}^{T-1}\gamma^{t}r_{t} \\
&\stackrel{\text{\textcircled{2}}}{=} \sum_{s_{0}, a_{0}, \ldots, s_{T}} \sum_{z_{0}^{(b_{0})},\ldots,z_{0}^{(b_{m})}}\cdots\sum_{z_{T}^{(b_{0})},\ldots,z_{T}^{(b_{m})}} \mu(s_{0}) \prod_{i=0}^{m}\pi^{\ast}(z_{0}^{(b_{i})}|s_{0}^{(b_{i})}) \\
&\quad\quad\prod_{t=0}^{T-1}\pi^{\ast}(a_{t}|f_{\phi}(z_{t}^{(b_{0})},\ldots,z_{t}^{(b_{m})};s_{t})) \prod_{i=0}^{m} \pi^{\ast}(z_{t+1}^{(b_{i})}|s_{t+1}^{(b_{i})}) \mathcal{P}(s_{t+1}|s_{t}, a_{t}) \sum_{t=0}^{T-1}\gamma^{t}r_{t} \\
&\stackrel{\text{\textcircled{3}}}{=} \sum_{s_{0}, a_{0}, \ldots, s_{T}} \sum_{z_{0}^{(b_{0})},\ldots,z_{0}^{(b_{m})}}\cdots\sum_{z_{T}^{(b_{0})},\ldots,z_{T}^{(b_{m})}} \mu(s_{0}) \prod_{i=0}^{m}\pi^{(b_{i})\ast}(z_{0}^{(b_{i})}|s_{0}^{(b_{i})}) \\
&\quad\quad\prod_{t=0}^{T-1}\pi^{\ast}(a_{t}|f_{\phi}(z_{t}^{(b_{0})},\ldots,z_{t}^{(b_{m})};s_{t})) \prod_{i=0}^{m} \pi^{(b_{i})\ast}(z_{t+1}^{(b_{i})}|s_{t+1}^{(b_{i})}) \mathcal{P}(s_{t+1}|s_{t}, a_{t}) \sum_{t=0}^{T-1}\gamma^{t}r_{t} \\
&\stackrel{\text{\textcircled{4}}}{=} \sum_{s_{0}, a_{0}, \ldots, s_{T}} \sum_{z_{0}^{(b_{0})},\ldots,z_{0}^{(b_{m})}}\cdots\sum_{z_{T}^{(b_{0})},\ldots,z_{T}^{(b_{m})}} \mu_{z}(z_{0}) \\
&\quad\quad\prod_{t=0}^{T-1}\pi^{\ast}(a_{t}|f_{\phi}(z_{t}^{(b_{0})},\ldots,z_{t}^{(b_{m})};s_{t})) \mathcal{P}_{z}(z_{t+1}|s_{t}, a_{t}) \sum_{t=0}^{T-1}\gamma^{t}r_{t}
\end{split}
\end{equation}

As shown in Equation~\ref{appendix:equ:sdmdp_opt_obj}, the optimal policy $\pi^{\ast}$ is substituted into the expression. Equation \textcircled{1} holds because the action taken at the terminal state $s_{T}$, denoted as $\pi(z_{T}^{(b_{0})},\ldots,z_{T}^{(b_{m})}|s_{T})$, does not influence the value of the objective function $J(\pi, \mathcal{P}, \mu)$. Equation \textcircled{2} is derived based on the conditional independence of the basis states $(s_{t}^{(b_{0})}, \ldots, s_{t}^{(b_{m})})$. Based on the conclusion of Theorem~\ref{appendix:thm:1}, we derive the following results: 1) In the $b_{0}$-th basis task, for all $s_{t}^{(b_{0})} \in \mathcal{S}_{t}^{(b_{0})}$ and $a_{t} \in \mathcal{A}(s_{t}^{(b_{0})})$, where $0 \leq t \leq T$, it holds that $\pi^{\ast}(a_{t}|s_{t}^{(b_{0})}) = \pi^{(b_{0})\ast}(a_{t}|s_{t}^{(b_{0})})$. Please note that $b_{0}$ is fixed to $0$ within the SDMDP framework. 2) Similarly, in the $b_{i}$-th basis task ($i \geq 1$), for all $(s_{t}^{(b_{0})}, s_{t}^{(b_{i})})\in \mathcal{S}_{t}^{(b_{0})} \times \mathcal{S}_{t}^{(b_{i})}$ and $a_{t} \in \mathcal{A}((s_{t}^{(b_{0})}, s_{t}^{(b_{i})}))$, where $0 \leq t \leq T$, it holds that $\pi^{\ast}(a_{t}|(s_{t}^{(b_{0})}, s_{t}^{(b_{i})})) = \pi^{(b_{i})\ast}(a_{t}|(s_{t}^{(b_{0})}, s_{t}^{(b_{i})}))$. Furthermore, in accordance with Assumption~\ref{appendix:assup:2}, we have $\pi^{\ast}(z_{t}|s_{t}^{(b_{0})}) = \pi^{(b_{0})\ast}(z_{t}|s_{t}^{(b_{0})})$ and $\pi^{\ast}(z_{t}|(s_{t}^{(b_{0})}, s_{t}^{(b_{i})})) = \pi^{(b_{i})\ast}(z_{t}|(s_{t}^{(b_{0})}, s_{t}^{(b_{i})}))$ for all $i\geq1$. According to Assumption~\ref{appendix:assup:1}, for each policy, including the basis policy, there exists a mixture function $f_{\phi}$ that provides a bijective and deterministic mapping from the basis state embeddings to the corresponding state embedding. Thus, the state embedding $z_{t}$ is given by $z_{t} = f_{\phi}(z_{t}^{(b_{0})}, z_{t}^{(b_{i})}; (s_{t}^{(b_{0})}, s_{t}^{(b_{i})}))$ under both $\pi^{\ast}$ and $\pi^{(b_{i})\ast}$. Consequently, we conclude that $\pi^{\ast}(z_{t}^{(b_{0})}|s_{t}^{(b_{0})})\pi^{\ast}(z_{t}^{(b_{i})}|s_{t}^{(b_{i})}) = \pi^{(b_{i})\ast}(z_{t}^{(b_{0})}|s_{t}^{(b_{0})})\pi^{(b_{i})\ast}(z_{t}^{(b_{i})}|s_{t}^{(b_{i})})$. Given the preceding results, we additionally assume that $\pi^{\ast}(z_{t}^{(b_{i})}|s_{t}^{(b_{i})})=\pi^{(b_{i})\ast}(z_{t}^{(b_{i})}|s_{t}^{(b_{i})})$, which directly supports the Equation \textcircled{3}. Equation \textcircled{4} is obtained by incorporating the definitions of the initial state distribution and the transition probability function within the LS-SDMDP framework.

It is evident that the policy $\pi^{\ast}(a_{t}|f_{\phi}(z_{t}^{(b_{0})},\ldots,z_{t}^{(b_{m})};s_{t}))$ is the optimal policy for the objective function $J(\pi_{f_{\phi}}, \mathcal{P}_{z}, \mu_{z})$. This is because, if this were not the case, $\pi^{\ast}$ would not qualify as the optimal policy for the objective function $J(\pi, \mathcal{P}, \mu)$. Thus, we conclude that $\pi^{\ast}(a_{t}|f_{\phi}(z_{t}^{(b_{0})},\ldots,z_{t}^{(b_{m})};s_{t})) = \pi_{f_{\phi}}^{\ast}(a_{t}|z_{t}^{(b_{0})},\ldots,z_{t}^{(b_{m})};s_{t})$ and $J(\pi^{\ast}, \mathcal{P}, \mu) = J(\pi_{f_{\phi}}^{\ast}, \mathcal{P}_{z}, \mu_{z})$.

By the Bellman equation and the policy decomposition in Equation~\ref{appendix:equ:pi_decomp}, the value function of SDMDP at the optimal policy, for each time step $0 \leq t \leq T-1$, can be expressed as follows:
\begin{equation}
\label{appendix:equ:sdmdp_value}
\begin{split}
V^{\pi^{\ast}, \mathcal{P}}(s_{t}) &= \mathbb{E}_{a_{t}\sim\pi^{\ast}(a_{t}|s_{t})}[r_{t} + \gamma\mathbb{E}_{s_{t+1}\sim\mathcal{P}(s_{t+1}|s_{t},a_{t})}V^{\pi^{\ast},\mathcal{P}}(s_{t+1})] \\
&=\mathbb{E}_{a_{t}\sim\pi^{\ast}(a_{t}|f_{\phi}(z_{t}^{(b_{0})}, \dots, z_{t}^{(b_{m})};s_{t}))}\mathbb{E}_{z_{t}^{(b_{0})}, \dots, z_{t}^{(b_{m})}\sim\pi^{\ast}(z_{t}^{(b_{0})}, \dots, z_{t}^{(b_{m})}|s_{t})} \\
&\quad\quad[r_{t} + \gamma\mathbb{E}_{s_{t+1}\sim\mathcal{P}(s_{t+1}|s_{t},a_{t})}V^{\pi^{\ast},\mathcal{P}}(s_{t+1})] \\
&=\mathbb{E}_{z_{t}^{(b_{0})}\sim\pi^{\ast}(z_{t}^{(b_{0})}|s_{t})}\cdots\mathbb{E}_{z_{t}^{(b_{m})}\sim\pi^{\ast}(z_{t}^{(b_{m})}|s_{t})} \mathbb{E}_{a_{t}\sim\pi^{\ast}(a_{t}|f_{\phi}(z_{t}^{(b_{0})}, \dots, z_{t}^{(b_{m})};s_{t}))} \\
&\quad\quad[r_{t} + \gamma\mathbb{E}_{s_{t+1}\sim\mathcal{P}(s_{t+1}|s_{t},a_{t})}V^{\pi^{\ast},\mathcal{P}}(s_{t+1})]
\end{split}
\end{equation}

Likewise, the value function of LS-SDMDP at its optimal policy can be written as:
\begin{equation}
\label{appendix:equ:ls_sdmdp_value}
\begin{split}
V^{\pi_{f_{\phi}}^{\ast}, \mathcal{P}_{z}}(z_{t}^{(b_{0})}, \ldots, z_{t}^{(b_{m})};s_{t}) &= \mathbb{E}_{a_{t} \sim \pi_{f_{\phi}}^{\ast}(a_{t}|z_{t}^{(b_{0})}, \ldots, z_{t}^{(b_{m})};s_{t})}[r_{t} + \gamma \mathbb{E}_{s_{t+1}\sim\mathcal{P}(s_{t+1}|s_{t},a_{t})} \\
&\quad\quad\mathbb{E}_{z_{t+1}^{(b_{0})}\sim \pi^{(b_{0})\ast}}\cdots\mathbb{E}_{z_{t+1}^{(b_{m})}\sim \pi^{(b_{m})\ast}} V^{\pi_{f_{\phi}}^{\ast}, \mathcal{P}_{z}}(z_{t+1}^{(b_{0})}, \ldots, z_{t+1}^{(b_{m})};s_{t+1})]
\end{split}
\end{equation}

We use the inductive method to prove the relationship between the value functions at their respective optimal policies. At the time step $T-1$, since the value function at time step $T$ is equal to 0, the value functions defined in Equation~\ref{appendix:equ:sdmdp_value} and Equation~\ref{appendix:equ:ls_sdmdp_value} can be expressed as follows:
\begin{equation}
\begin{split}
&V^{\pi^{\ast}, \mathcal{P}}(s_{T-1}) =\mathbb{E}_{z_{T-1}^{(b_{0})}\sim\pi^{\ast}(z_{T-1}^{(b_{0})}|s_{T-1})}\cdots\mathbb{E}_{z_{T-1}^{(b_{m})}\sim\pi^{\ast}(z_{T-1}^{(b_{m})}|s_{T-1})} \\
&\quad\quad\quad\quad\quad\quad\quad\quad\mathbb{E}_{a_{T-1}\sim\pi^{\ast}(a_{T-1}|f_{\phi}(z_{T-1}^{(b_{0})}, \dots, z_{T-1}^{(b_{m})};s_{T-1}))}[r_{T-1}] \\
&V^{\pi_{f_{\phi}}^{\ast}, \mathcal{P}_{z}}(z_{T-1}^{(b_{0})}, \ldots, z_{T-1}^{(b_{m})};s_{T-1}) = \mathbb{E}_{a_{T-1} \sim \pi_{f_{\phi}}^{\ast}(a_{T-1}|z_{T-1}^{(b_{0})}, \ldots, z_{T-1}^{(b_{m})};s_{T-1})} [r_{T-1}]
\end{split}
\end{equation}

It is evident that the two value functions satisfy the relationship at time step $T-1$. We now assume that at any time step $t+1$, the relationship between these value functions holds. Substituting this assumption into the value function of SDMDP defined in Equation~\ref{appendix:equ:sdmdp_value}, yields the following:
\begin{equation}
\begin{split}
V^{\pi^{\ast}, \mathcal{P}}(s_{t}) &=\mathbb{E}_{z_{t}^{(b_{0})}\sim\pi^{\ast}(z_{t}^{(b_{0})}|s_{t})}\cdots\mathbb{E}_{z_{t}^{(b_{m})}\sim\pi^{\ast}(z_{t}^{(b_{m})}|s_{t})} \mathbb{E}_{a_{t}\sim\pi^{\ast}(a_{t}|f_{\phi}(z_{t}^{(b_{0})}, \dots, z_{t}^{(b_{m})};s_{t}))} \\
&\quad\quad[r_{t} + \gamma\mathbb{E}_{s_{t+1}\sim\mathcal{P}(s_{t+1}|s_{t},a_{t})}\mathbb{E}_{z_{t+1}^{(b_{0})}\sim \pi^{(b_{0})\ast}}\cdots\mathbb{E}_{z_{t+1}^{(b_{m})}\sim \pi^{(b_{m})\ast}} \\
&\quad\quad V^{\pi_{f_{\phi}}^{\ast}, \mathcal{P}_{z}}(z_{t+1}^{(b_{0})}, \ldots, z_{t+1}^{(b_{m})};s_{t+1})]
\end{split}
\end{equation}
We thus conclude that the relationship between the value functions at their optimal policies is given by $V^{\pi^{\ast}, \mathcal{P}}(s) = \mathbb{E}_{z^{(b_{0})} \sim \pi^{(b_{0})\ast}}\cdots \mathbb{E}_{z^{(b_{m})} \sim \pi^{(b_{m})\ast}}V^{\pi_{f_{\phi}}^{\ast}, \mathcal{P}_{z}}(z^{(b_{0})},\ldots,z^{(b_{m})};s)$.
\end{proof}

\section{Related Works}
\label{appendix:related_work}

In this section, we first survey advances in task-specific neural solvers, which form the foundation for multi-task approaches. Next, we review multi-task learning techniques for VRPs. Finally, we examine recent progress in mixture-of-specialized-experts (MoSE) methods, which inspire key implementation aspects of our method.

\textbf{Task-Specific Neural VRP Solvers.} Learning-based neural solvers, individually developed for each specific VRP, fall into three main categories: constructive methods, iterative methods, and divide-and-conquer methods. \emph{Constructive methods} craft end-to-end neural solvers to progressively infer solutions through autoregressive mechanisms. Pointer Network~\cite{vinyals2015pointer} pioneers this paradigm by effectively solving the small-scale traveling salesman problem (TSP), while Attention Model (AM)~\cite{kool2018attention} emerges as the dominant architecture for subsequent task-specific neural solvers trained via reinforcement learning (RL). Some approaches consider properties inherent in TSP and capacitated VRP (CVRP), including the multiple optima~\cite{kwon2020pomo} and the symmetry~\cite{kim2022sym}, to improve the solution quality. To further improve the cross-scale or cross-distribution generalization, advanced techniques, such as meta-learning~\cite{manchanda2022generalization,qiu2022dimes,zhou2023towards}, knowledge distillation~\cite{bi2022learning}, or ensemble learning~\cite{gao2023towards,grinsztajn2023winner,jiang2024ensemble}, have been successfully transferred from other domains. Additionally, neural solvers trained via supervised learning (SL) also exhibit strong generalization capabilities~\cite{drakulic2024bq,luo2024neural,luo2025boosting}. \emph{Iterative methods} leverage local search operators to consistently refine solutions until convergence. Specifically, L2I~\cite{lu2019learning} and NeuRewriter~\cite{chen2019learning} train RL policies to select among handcrafted operators for solution improvement. NLNS~\cite{hottung2020neural} alternates between heuristic destroy operators and learned repair policies to generate a new solution. DACT~\cite{ma2021learning} advances beyond prior works by learning expressive representations for RL policies, while Neural-LKH~\cite{xin2021neurolkh} and Neural k-opt~\cite{ma2023neuopt} specialize in k-opt algorithms by using RL policies to guide edge exchanges. However, these methods universally trade inference efficiency for solution quality with the aid of manually-designed operators. \emph{Divide-and-Conquer methods} decompose problem instances into smaller sub-instances that are solved independently. Prior works have attempted to solve larger instances using pretrained neural solvers on heuristically sampled sub-instances~\cite{fu2021generalize,kim2021learning,cheng2023select}. By comparison, L2D~\cite{li2021learning} and RGB~\cite{zong2022rbg} learn RL policies to select subgraphs from heuristic-generated candidates. Unlike these heuristic-based methods, TAM~\cite{hou2023generalize}, GLOP~\cite{ye2024glop}, UDC~\cite{zheng2024udc}, and HLGP~\cite{pan2025hlgp4cvrp} opt to learnable RL policies to globally partition entire instances into subproblems, which are then solved by pretrained local construction policies. However, these methods critically depend on the partition policy, where even minor performance degradation may significantly impair the overall VRP solution quality.

\textbf{Multi-Task Learning for VRPs.} To cope with practical scenarios involving multiple VRP variants, the efficient transfer learning has been leveraged to obtain specialized neural solvers by considering inherent similarities among these variants. Lin et al.~\cite{lin2024cross} propose a modular architecture consisting of a backbone pretrained on a canonical VRP and lightweight adapters inserted into the frozen backbone for the problem-specific fine-tuning. Likewise, GOAL~\cite{drakulic2024goal} extend this paradigm to general COPs, along with fewer adapters. Corr{\^e}a et al.~\cite{correa2025tunensearch} adopt full-parameter fine-tuning strategy for each downstream VRP variant, yielding multiple problem-specific but parameter-inefficient neural solvers. However, these methods struggle with the combinatorial explosion of VRP variants. Notably, although the LoRA adapter is used for problem-specific adaptations in~\cite{lin2024cross}, its potential as part of a unified solver to handle the exponential growth of variants remains unexplored.

As an alternative, unified solvers are designed to handle multiple VRP variants simultaneously while eliminating the need for adapters or backbone duplications. Liu et al.~\cite{liu2024multi} pioneer a single-solver framework that unifies variants via attribute compositions. MVMoE~\cite{zhou2024mvmoe} enhances capacity with mixture-of-expert (MoE) layers that implicitly specialize for different variants. RouteFinder~\cite{berto2024routefinder} employs modified Transformer model with mixed batch training for stable convergence. UNCO~\cite{jiang2024unco} resorts to the large language model (LLM) for the expressive instance and task embeddings. CaDA~\cite{li2024cada} utilizes a dual attention mechanism for superior cross-problem capabilities. Goh et al.~\cite{goh2025shield} considers the more practical multi-task, multi-distribution setting by using mixture-of-depths (MoD) and context-based clustering. Liu et al.~\cite{liu2025mixed} leverage mixed-curvature spaces in the feature fusion stage such that the model’s encoder can capture the geometric structures inherent in VRP instances. Lei et al.~\cite{lei2025boosting} strengthen cross-scale and cross-problem generalization capabilities of diffusion-based solvers during inference, which is orthogonal to our work. In contrast, Wang \& Yu~\cite{wangefficient} decompose the model into a shared encoder, and problem-specific headers and decoders, applying multi-armed bandits for dynamic task sampling during training. Li et al.~\cite{li2025toward} adopt the same architecture with aligned optimization directions across tasks. However, these methods fail to fully leverage the compositional structure inherent in VRP variants, each derived from a common set of basis VRP variants. Thus, the potential benefits of incorporating explicitly specialized basis solvers remain unexplored. Notably, while MoE layers are employed in~\cite{zhou2024mvmoe}, these learn implicit and less interpretable specialization rather than incorporating off-the-shelf basis solvers as experts.

\textbf{Mixture of Specialized Experts.} Prevailing methods focusing on mixture of specialized experts (MoSE) can be broadly categorized into two paradigms: merging entire models and module composition. Approaches based on merging entire models seek to combine independently trained models to efficiently achieve the performance comparable to model ensembling or multi-task learning. Most approaches are developed based on the shared model architecture. Wortsman et al.~\cite{wortsman2022modelsoups} exhibit averaging the parameters of models trained with different hyperparameter configurations improves accuracy. Matena \& Raffel~\cite{matena2022mergingmodels} develop a model aggregation framework guided by Fisher information. Yadav et al.~\cite{yadav2023tiesmerging} address interference effects in model merging to preserve critical knowledge. Tam et al.~\cite{tam2024merging} formulate model fusion as a linear equation system solved through conjugate gradient optimization. Daheim et al.~\cite{daheim2024model} introduce a uncertainty-based merging scheme to reduce the gradient mismatch. In contrast, several approaches enable merging across architecturally distinct models. Ainsworth et al.~\cite{ainsworth2023git} align weights of different models to enable the model merging. Stoica et al.~\cite{stoica2024zipit} introduce cross-task model fusion through feature-space integration. Additionally, weighted merging strategies have also been extensively explored. Ram\'{e} et al.~\cite{rame2023ratatouille} implement weighted model fusion to enhance out-of-distribution generalization, while Jin et al.~\cite{jin2023dataless} develop parameter-space merging optimized for multi-task performance through learned weighting schemes. Yang et al.~\cite{yang2024AdaMerging} further advance this paradigm by deriving merging coefficients via unsupervised entropy minimization. However, these approaches struggle with achieving precise layer-wise and token-wise aggregation within expert models, limiting their multi-task OOD generalization.

Our implementation aligns more closely with the module composition paradigm which supports the finer-grained aggregation. The most straightforward approach involves parameter averaging of task adapters. Chronopoulou et al.~\cite{chronopoulou2023adaptersoup} demonstrate test-time averaging of relevant adapters for task adaptation, while Ponti et al.~\cite{ponti2023combining} propose averaging adapter parameters selected by a routing function for new tasks. Beyond simple averaging, arithmetic operations have proven effective for the adapter composition. Chronopoulou et al.~\cite{chronopoulou2024language} demonstrate that basic arithmetic combinations enhance zero-shot cross-lingual transfer, while Zhang et al.~\cite{zhang2023composing} develop specialized arithmetic operations for lightweight adapters to boost generalization. Similarly, Ilharco et al.~\cite{ilharco2023editing} employ arithmetic combinations of task vectors to precisely steer model behavior for novel tasks. Task similarity further facilitates the effective module composition. Lv et al.~\cite{lv2023parameter} develop weighted aggregation of parameter-efficient adapters based on inter-task similarity measures for novel tasks. The MoCLE~\cite{gou2023mixture} framework addresses task conflicts and improves generalization by activating task-customized LoRA adapters based on clustered instructions and using a trainable universal adapter. Similarly, Wu et al.~\cite{wu2023pituning} propose modality-agnostic task similarity measures to combine lightweight adapters for enhanced performance on multimodal downstream tasks. Adaptive gating mechanisms have emerged as the predominant approach for dynamic model composition. LoRAHub~\cite{huang2024lorahub} employs few-shot examples to compute weighting coefficients for pre-trained LoRA modules, enabling competitive performance on novel tasks. Building on this, Ye et al.~\cite{ye2022eliciting} develop a hierarchical routing mechanism that selects optimal Transformer layers for enhanced task generalization. The MoLE~\cite{wu2024mixture} framework advances this paradigm through layer-wise gating functions that learn to optimally combine LoRA experts, with these operations being implemented via dedicated Transformer blocks. AdaMoLE~\cite{liu2024adamole} further incorporates dynamic threshold adaptation to handle varying task complexities. Caccia et al.~\cite{caccia2023multihead} introduce multi-head adapter routing, enabling fine-grained routing for the cross-task generalization. Module composition approaches also effectively address catastrophic forgetting. LoRAMoE~\cite{dou2024loramoe} mitigates world knowledge degradation in frozen backbone LLMs through strategic integration of multiple LoRA modules. Similarly, AdapterFusion~\cite{pfeiffer2021adapterfusion} combats catastrophic forgetting in multi-task learning scenarios by dynamically combining pretrained adapter modules. Furthermore, module composition techniques have also demonstrated significant potential for zero-shot generalization. PHATGOOSE~\cite{muqeeth2024learning} achieves this through inference-time aggregation of LoRA adapters via their pretrained gating functions. Similarly, Ostapenko et al.~\cite{ostapenko2024towards} develop a zero-shot routing mechanism that dynamically selects the most task-relevant adapters based on task similarity, eliminating the need for retraining. In addition, Zadouri et al.~\cite{zadouri2024pushing} advance efficient adaptation through fine-tuning of mixture of lightweight LoRA experts for limited computational cost scenarios. Gao et al.~\cite{gao2024higherlayersneedlora} uncover important architectural insights: different layers benefit from varying numbers of LoRA modules, and higher network layers particularly require more experts to maintain performance.


\end{document}